\documentclass[twoside]{article}

\usepackage[accepted]{aistats2024}

\usepackage{amsmath}
\usepackage{amssymb}
\usepackage{mathtools}
\usepackage{amsthm}

\usepackage{algorithm}
\usepackage{algorithmic}
\usepackage{amsthm}
\usepackage{subfigure}
\usepackage{epsfig}
\usepackage{graphicx}
\usepackage{url}
\usepackage{multirow}
\usepackage{amssymb}
\usepackage{amsthm}
\usepackage{mathrsfs}
\usepackage{epstopdf}
\usepackage{multicol}
\usepackage{amsmath}
\usepackage{bm}
\usepackage{multibib}
\usepackage{xcolor}

\newtheorem{theorem}{\textbf{Theorem}}
\newtheorem{assumption}{\textbf{Assumption}}
\newtheorem{lemma}{\textbf{Lemma}}
\newtheorem{corollary}{\textbf{Corollary}}
\newtheorem{remark}{\textbf{Remark}}

\usepackage[round]{natbib}

\begin{document}

%
\runningtitle{Decentralized Multi-Level Compositional Optimization}

%
\runningauthor{Hongchang Gao}

\twocolumn[

\aistatstitle{Decentralized Multi-Level Compositional Optimization Algorithms with Level-Independent Convergence Rate}

\aistatsauthor{Hongchang Gao}

\aistatsaddress{Temple University} ]

\begin{abstract}
	Stochastic multi-level compositional optimization problems cover many new machine learning paradigms, e.g., multi-step model-agnostic meta-learning, which require efficient optimization algorithms for large-scale data. This paper studies the decentralized stochastic multi-level optimization algorithm, which is challenging because the multi-level structure and decentralized communication scheme may make the number of levels significantly affect the order of the convergence rate. To this end, we develop two novel decentralized optimization algorithms to optimize the multi-level compositional optimization problem. Our theoretical results show that both algorithms  can achieve the level-independent convergence rate for nonconvex problems under much milder conditions compared with existing single-machine algorithms. To the best of our knowledge, this is the first work that achieves the level-independent convergence rate under the decentralized setting. Moreover, extensive experiments confirm the efficacy of our proposed algorithms. 
\end{abstract}

\section{Introduction}
In recent years, some new learning paradigms, such as model-agnostic meta-learning \citep{finn2017model},  have been proposed to handle realistic machine learning applications, which are typically beyond the class of  traditional stochastic optimization. Some examples include bilevel optimization, minimax optimization, compositional optimization, and so on.  Of particular interest in this paper is the learning paradigm that can be formulated as the \textit{stochastic multi-level compositional optimization problem}.  More particularly, we are interested in the decentralized setting where data are distributed on different devices and the device performs  peer-to-peer communication to exchange information with its  neighboring devices. Mathematically, the loss function is defined as follows:
\vspace{-5pt}
\begin{equation} \label{loss}
	\min _{x \in \mathbb{R}^d} F(x)= \frac{1}{N} \sum_{n=1}^{N}F_n(x)  \ , 
\end{equation}
where $F_n(x) = f_n^{(K)} \circ f_n^{(K-1)} \circ \cdots \circ f_n^{(2)}\circ f_n^{(1)} (x)$, $x\in \mathbb{R}^d$ is the model parameter of a machine learning model,   $N$  devices compose a communication network, and  $F_n(x)$ is the loss function on the $n$-th device, 
for any $k\in \{1, 2, \cdots, K-1, K\}$, $f_n^{(k)}(\cdot)=\mathbb{E}_{\xi_{n}^{(k)}}[f_n^{(k)}(\cdot; \xi_{n}^{(k)})] : \mathbb{R}^{d_{k-1}} \rightarrow  \mathbb{R}^{d_{k}}$ is the $k$-th level  function on the $n$-th device, where $\xi_{n}^{(k)}$ denotes the data distribution for the $k$-th level function on the $n$-th device. 
It can observed that the input of $f_n^{(k)}(\cdot)$ is the output of $f_n^{(k-1)}(\cdot)$.

The stochastic multi-level compositional optimization (multi-level SCO) problem covers a wide range of machine learning models. For instance, the multi-step model-agnostic meta-learning \citep{finn2017model} can be formulated as a multi-level SCO problem. The stochastic training of graph neural networks  also belongs to the class of multi-level SCO problem \citep{yu2022graphfm,cong2021importance}.  The neural network with batch-normalization is actually a multi-level SCO problem \citep{lian2018revisit}.   The challenge of optimizing the multi-level SCO problem lies in that the stochastic gradient is not an unbiased estimator of the full gradient when the inner-level functions are nonlinear. To address this challenge, a couple of stochastic multi-level compositional gradient descent (multi-level SCGD) algorithms have been proposed recently. For instance, \cite{yang2019multilevel} proposed the first stochastic multi-level  compositional gradient descent algorithm. However, due to the nested structure of the loss function, {the order of its convergence rate depends on the number of levels $K$ exponentially} \footnote{Following \citep{balasubramanian2022stochastic}, throughout this paper, the level-dependent convergence rate means that the number of levels $K$  affects the order of the convergence rate, e.g., $O(\epsilon^{-K})$, while the level-independent convergence rate indicates that $K$ does not affect its order but may affect its coefficient,  e.g., $O(K\epsilon^{-2})$.},  where a larger $K$ results in a slower convergence rate. As such, it cannot match the traditional stochastic gradient descent (SGD) algorithm's convergence rate. Later, some algorithms \citep{zhang2021multilevel,balasubramanian2022stochastic,jiang2022optimal} were proposed to achieve the level-independent convergence rate via leveraging the variance-reduced estimator.  For instance, \cite{zhang2021multilevel} exploited the SPIDER \citep{nguyen2017sarah,fang2018spider}  estimator for both the stochastic function value $f_n^{(k)}(\cdot; \xi_{n}^{(k)})$ and the stochastic Jacobian matrix $\nabla f_n^{(k)}(\cdot; \xi_{n}^{(k)})$  of  each level function  to improve the convergence rate. 

However, existing stochastic multi-level  compositional optimization algorithms have some limitations. On the one hand, they only focus on the single-machine setting. As such,  they cannot be used to solve  the distributed multi-level SCO problem in Eq.~(\ref{loss}). In particular, it is unclear if  the \textit{level-independent convergence rate} is still achievable under the decentralized setting. More particularly, it is unclear whether the consensus error  caused by the decentralized communication scheme will make the level-independent convergence rate unachievable. On the other hand, under the single-machine setting, those algorithms with the variance-reduced estimator have some unrealistic operations, limiting their applications in real-world tasks.  In particular, they apply the variance reduction technique to the stochastic Jacobian matrix of every level function $\nabla f_n^{(k)}(\cdot, \xi_{n}^{(k)})$ where $k\in\{1, 2, \cdots, K\}$, which requires the clipping operation, e.g., Algorithm 3 in \citep{zhang2021multilevel}, or the projection operation, e.g., Eq.~(3) in \citep{jiang2022optimal}, to upper bound the variance-reduced gradient.  These operations either result in a very small learning rate or depend on unknown hyperparameters. These limitations motivate us to \textit{1)  develop  decentralized optimization algorithms for Eq.~(\ref{loss}) to enable  multi-level SCO problems for distributed data, 2) propose the practical algorithm based on the variance-reduced stochastic gradient  under mild conditions, and 3) establish the level-independent convergence rate  for the proposed algorithm. }

To this end, we developed two novel decentralized multi-level  stochastic compositional gradient descent algorithms, both of which can achieve the level-independent convergence rate. Specifically, they have the following contributions. 1) Our first algorithm demonstrates how to achieve the level-independent convergence rate with a novel combination of the inner-level function estimator and the momentum technique. Our second algorithm improves the convergence rate with a novel strategy of utilizing the variance-reduced estimator without impractical operations. In particular, unlike existing algorithms \citep{zhang2021multilevel,jiang2022optimal}, which apply the variance reduction technique to both \textit{all stochastic inner-level  function values} $\{f_n^{(k)}(\cdot; \xi_{n}^{(k)})\}_{k=1}^{K-1}$  and \textit{ all stochastic Jacobian matrices} $\{\nabla f_n^{(k)}(\cdot; \xi_{n}^{(k)})\}_{k=1}^{K}$, our algorithm  leverages the variance-reduced estimator for stochastic  inner-level function values $\{f_n^{(k)}(\cdot; \xi_{n}^{(k)})\}_{k=1}^{K-1}$ and  \textit{the gradient} $\nabla F_n(x; \xi_{n})$. As such, our algorithm does not require the clipping operation for the learning rate or the projection operation for   $\{\nabla f_n^{(k)}(\cdot; \xi_{n}^{(k)})\}_{k=1}^{K}$ as \citep{zhang2021multilevel,jiang2022optimal}.  Thus, it is more friendly to implement. 2) Besides the novel algorithmic design, we established the level-independent convergence rate of our two algorithms for nonconvex problems under the decentralized setting. In particular, our first algorithm, which leverages the momentum technique for the  gradient, enjoys the convergence rate of $O(\epsilon^{-4})$ to achieve the $\epsilon$-stationary point. Our second algorithm, which exploits the variance reduction technique for the  gradient, can achieve the   convergence rate of $O(\epsilon^{-3})$ for nonconvex problems. As far as we know, this is the first decentralized optimization work for multi-level SCO with theoretical guarantees.   3) Extensive experiments on the multi-step model-agnostic meta-learning task confirm the effectiveness of our algorithms.

\section{Related Work}
\subsection{Stochastic Two-Level  Compositional Optimization}
The stochastic two-level  compositional optimization problem has been extensively studied in the past few years. In particular, to address the biased gradient estimator problem,  \cite{wang2017stochastic} developed the stochastic compositional gradient descent (SCGD) algorithm for the first time, where the moving-average technique was leveraged to the estimation of the inner-level function value to control the estimation error. However, its sample complexity is as large as $O(\epsilon^{-8})$ for nonconvex problems, which is worse than $O(\epsilon^{-4})$ of the standard SGD algorithm for non-compositional optimization problems. Then, \cite{ghadimi2020single} applied the momentum technique to stochastic compositional gradient so that it improved the sample complexity to $O(\epsilon^{-4})$. On the contrary, \cite{chen2020solving} leveraged the variance-reduced estimator \citep{cutkosky2019momentum} for the inner-level function, which can also achieve the sample complexity of $O(\epsilon^{-4})$. To further improve the convergence rate, a couple of works exploited the variance-reduced technique to control the estimation error for both the inner-level function value and its Jacobian matrix. For instance, \cite{yuan2019stochastic} leveraged the SPIDER variance reduction technique \citep{nguyen2017sarah,fang2018spider} to improve the sample complexity to $O(\epsilon^{-3})$ for stochastic nonconvex problems. However, this algorithm requires a large batch size. To address this problem, \cite{yuan2020stochastic} employed the STORM variance reduction technique \citep{cutkosky2019momentum}, which can also achieve the sample complexity of $O(\epsilon^{-3})$, but with a small batch size.  As for the nonconvex finite-sum compositional problem, a couple of works \citep{zhang2019composite,zhang2019stochastic,yuan2019stochastic} also utilized the variance-reduction technique to improve the sample complexity to match the counterpart for non-compositional problems.

\subsection{Stochastic  Multi-Level Compositional Optimization}
Even though the aforementioned  algorithms can achieve desired sample complexity for the two-level compositional problem, it is non-trivial to extend them to the multi-level problem for achieving the same sample complexity. 
For instance, \cite{yang2019multilevel} developed an accelerated  stochastic compositional gradient descent algorithm for the stochastic multi-level  compositional optimization problem, which can only achieve the sample complexity of $O(\epsilon^{-(7+K)/2})$ for nonconvex problems. Obviously, this sample complexity depends on the number of function levels $K$, which is far from satisfactory. Later, \cite{balasubramanian2022stochastic} extended the momentum approach \citep{ghadimi2020single} to the multi-level problem, obtaining the $O(\epsilon^{-6})$ sample complexity, which is worse than the counterpart \citep{ghadimi2020single} for the two-level problem. Then,  they added a correction term when using the moving-average technique to estimate each level function so that the sample complexity was improved to $O(\epsilon^{-4})$, which can match the standard momentum stochastic gradient descent algorithm. In \citep{chen2020solving}, the STORM variance-reduction technique is leveraged to estimate each level function, which can also result in the sample complexity of $O(\epsilon^{-4})$. In \citep{zhang2021multilevel}, the SPIDER variance-reduction technique is exploited to estimate both each level function and its gradient so that it can achieve the sample complexity of $O(\epsilon^{-3})$. However, this algorithm requires a large batch size. Moreover, it requires a small learning rate to guarantee the Lipschitz continuousness of the variance-reduced gradient.  Recently, \cite{jiang2022optimal} leveraged the STORM variance-reduction technique to estimate each level function value and its Jacobian matrix, resulting in the sample complexity of $O(\epsilon^{-3})$ with the mini-batch size of $O(1)$. However, this algorithm requires the projection operation for Jacobian matrices such that they are upper bounded.  Thus, these algorithms with the  sample complexity $O(\epsilon^{-3})$ are not practical for real-world applications.
Moreover, it is unclear how to obtain the level-independent sample complexity under the decentralized setting.

\subsection{Decentralized Compositional Optimization}
Decentralized optimization  has been extensively studied for the non-compositional optimization problem from both the computation  \citep{lian2017can,sun2020improving,xin2020near} and communication \citep{koloskova2019decentralized,koloskova2019decentralized2,gao2020periodic,song2022communication,hua2022efficient,ying2021exponential} perspectives in recent years. 
Those algorithms are based on the stochastic gradient, which is an unbiased estimator of the full gradient. 
Thus, they cannot be directly extended to the stochastic compositional optimization problem because its stochastic gradient is a biased estimator of the full gradient. Recently, to address this problem, \cite{gao2021fast} developed the decentralized stochastic compositional gradient descent algorithm for the two-level stochastic compositional problem for the first time, which can achieve the sample complexity of $O(\epsilon^{-6})$. \cite{zhao2022distributed} leveraged the STORM-like technique to estimate the inner-level function and improved the sample complexity to $O(\epsilon^{-4})$.  Moreover, \cite{gao2023achieving} developed the decentralized stochastic compositional gradient descent ascent algorithm for stochastic compositional minimax problems.  On the other hand, a series of decentralized bilevel optimization algorithms have been proposed recently, e.g., \citep{gao2023convergenceaistats,zhang2023communication,lu2022stochastic} and the related works therein.  However, all those existing compositional and bilevel optimization algorithms only focus on the two-level problem. It is unclear how to apply them to the multi-level compositional optimization problem to achieve the level-independent sample complexity.

\section{Decentralized Stochastic Multi-Level Compositional Optimization}

In this section, we present the details of our proposed algorithms under the decentralized setting. Here,  it is assumed  the devices compose a communication graph and perform  peer-to-peer communication. The adjacency matrix $W$ of this  graph satisfies the following assumption. 
\begin{assumption} \label{assumption_graph}
	$W=[w_{ij}]\in \mathbb{R}^{N\times N}$  is a symmetric and doubly stochastic matrix. Its eigenvalues satisfy $|\lambda_N|\leq |\lambda_{N-1}|\leq \cdots \leq |\lambda_2|< |\lambda_1|=1$. 
\end{assumption}
Under this assumption, we can denote the spectral gap as $1-\lambda$ where $\lambda=|\lambda_2|$. Then, we propose two decentralized optimization algorithms for solving Eq.~(\ref{loss}) in the following two subsections.

\begin{algorithm}[h]
	\caption{DSMCGDM}
	\label{alg_dscgdm}
	\begin{algorithmic}[1]
		\REQUIRE ${x}_{n,0}={x}_{0}$, $\alpha>0$, $\beta>0$, $\mu>0$, $\eta>0$.

		\FOR{$t=0,\cdots, T-1$} 
		
		\STATE ${u}_{n,  t}^{(0)}= x_{n, t}$,
		\FOR{$k=1,\cdots, K-1$} 
		\IF {$t==0$}
		\STATE  ${u}_{n,  t}^{(k)}=f_{n}^{(k)}({u}_{n,  t}^{(k-1)}; \xi_{ n, t}^{(k)}) $,
		\ELSE
		\STATE  ${u}_{n,  t}^{(k)}=(1-\beta\eta)({u}_{n,  t-1}^{(k)} - f_{n}^{(k)}({u}_{n,  t-1}^{(k-1)}; \xi_{n,  t}^{(k)}) )+  f_{n}^{(k)}({u}_{n,  t}^{(k-1)}; \xi_{ n, t}^{(k)}) $,
		\ENDIF
		\STATE ${v}_{n,  t}^{(k)}=\nabla  f_{n}^{(k)}({u}_{ n, t}^{(k-1)}; \xi_{n,  t}^{(k)}) $,
		\ENDFOR
		\STATE ${v}_{ n, t}^{(K)}=\nabla  f_{n}^{(K)}({u}_{n, t}^{(K-1)}; \xi_{n,  t}^{(K)}) $, \\ 
		\STATE ${g}_{ n, t}= {v}_{n,  t}^{(1)} {v}_{ n, t}^{(2)}\cdots {v}_{ n, t}^{(K-1)} {v}_{n,  t}^{(K)}$, 
		\IF {$t==0$}
		\STATE $m_{n, t} = {g}_{n,  t}$,  $y_{n, t}= m_{n,t}$
		\ELSE
		\STATE $m_{n, t} = (1-\mu\eta)m_{n, t-1} + \mu \eta{g}_{n,  t}$, \\
		\STATE $y_{n, t}=\sum_{n'\in \mathcal{N}_{n}}w_{nn'} y_{n', t-1} + m_{n,t} - m_{n, t-1} $,
		\ENDIF

		\STATE ${x}_{n,   t+\frac{1}{2}}=\sum_{n'\in \mathcal{N}_{n}}w_{nn'}{x}_{ n', t}  - \alpha {y}_{n,  t}$, \\ 
		${x}_{n,   t+1}=  {x}_{n,   t} + \eta ({x}_{n,   t+\frac{1}{2}} - {x}_{n,   t})$, 
		\ENDFOR
	\end{algorithmic}
\end{algorithm}

\subsection{Decentralized Stochastic Multi-level Compositional Gradient Descent with Momentum}

\paragraph{Challenges.}
The momentum technique is commonly used in optimization. However, facilitating it to multi-level SCGD is non-trivial. Under the single-machine setting, \cite{balasubramanian2022stochastic} developed the first multi-level SCGD with momentum algorithm, which applies the moving-average technique to each inner-level function and the  gradient. However, this straightforward extension can only achieve  the $O(\epsilon^{-6})$ sample complexity, which is worse than $O(\epsilon^{-4})$ of the two-level algorithm. Then, \cite{balasubramanian2022stochastic} introduced a correction term to the inner-level function estimator to address this problem. However, this correction term requires to compute the Jacobian matrix (See its Algorithm 2), which is too complicated and unclear if it works under the decentralized setting.  Especially, \textit{it is unclear whether the consensus error caused by the decentralized communication topology will worsen the convergence rate in the presence of multi-level inner functions. }
Therefore, a natural question follows:  	\textbf{How to design an efficient decentralized multi-level SCGD with momentum algorithm to achieve the level-independent sample complexity $O(\epsilon^{-4})$?}

To answer this  question, in Algorithm~\ref{alg_dscgdm}, we develop the Decentralized Stochastic Multi-level Compositional Gradient Descent with Momentum (DSMCGDM) algorithm. Specifically, to achieve the level-independent sample complexity, which can match the  decentralized SGD with momentum algorithm for non-compositional problem, we leverage the STORM-like technique to estimate the $k$-th level function (where $k\in \{1, 2, \cdots, K-1\}$), which is shown below:
\begin{equation} \label{eq_storm_like}
	\begin{aligned}
		& {u}_{n,  t}^{(k)}=(1-\beta\eta)({u}_{n,  t-1}^{(k)} - f_{n}^{(k)}({u}_{n,  t-1}^{(k-1)}; \xi_{n,  t}^{(k)}) )  \\
		& \quad \quad \quad +  f_{n}^{(k)}({u}_{n,  t}^{(k-1)}; \xi_{ n, t}^{(k)}) \ , 
	\end{aligned}
\end{equation}
where $\beta>0$, $\eta>0$ are two hyperparameters satisfying $\beta\eta<1$, ${u}_{n,  t}^{(k)}$ is the estimation of the $k$-th level function $ f_{n}^{(k)}({u}_{n,  t}^{(k-1)})$ on the $n$-th device.

It is worth noting that we do not apply this variance-reduction technique to the stochastic Jacobian matrix ${v}_{n,  t}^{(k)}\triangleq\nabla  f_{n}^{(k)}({u}_{ n, t}^{(k-1)}; \xi_{n,  t}^{(k)})$. 
After we obtain the stochastic Jacobian matrix ${v}_{n,  t}^{(k)}$ of each level function, we combine them to get the  stochastic compositional gradient ${g}_{ n, t}$ of the objective function $F_n(x)$, which is shown in Line 12. Then, we compute the momentum of this  stochastic compositional gradient in Line 16, where $\mu>0$ is a hyperparameter satisfying $\mu\eta<1$.  After that, we  leverage the gradient-tracking technique in Line 17 to communicate the momentum between different devices according to the communication topology, which is defined below:
\begin{equation}
	y_{n, t}=\sum_{n'\in \mathcal{N}_{n}}w_{nn'} y_{n, t-1} + m_{n,t} - m_{n, t-1}  \ , 
\end{equation}
where $ \mathcal{N}_{n}=\{n'|w_{nn'}>0\}$ denotes the neighbors of the $n$-th device and $w_{nn'}$ is the edge weight of the communication graph. Finally, we can leverage $y_{n, t}$ to update the model parameter on the corresponding device, which is shown in Line 19, where $\alpha>0$ is a hyperparameter.

Note that Eq.~(\ref{eq_storm_like}) has been used for \textit{non-momentum} algorithm under the single-machine setting in \citep{chen2020solving}, rather then the decentralized setting. Therefore, it is still unclear how it affects the convergence for the \textit{momentum} algorithm or the \textit{decentralized} setting.  In fact, this is the first time to apply Eq.~(\ref{eq_storm_like}) to the momentum algorithm. We believe this novel algorithmic design can also be applied to the single-machine setting to accelerate existing algorithms, e.g., \citep{chen2020solving}. Moreover, to the best of our knowledge, this is the first decentralized optimization algorithm for the stochastic  multi-level compositional optimization problem. Meanwhile, this  algorithmic design brings new challenges for convergence analysis due to the interaction between the   estimator of each level function and momentum.  We will address these challenges and show this algorithm can achieve the $O(\epsilon^{-4})$ sample complexity in Section 4.

\subsection{Decentralized Stochastic Multi-Level Compositional Variance-Reduced Gradient Descent}
To  improve the convergence rate, in Algorithm~\ref{alg_dscgdvr},  we propose our second algorithm: Decentralized Stochastic Multi-level Compositional Variance-Reduced Gradient descent algorithm (DSMCVRG).

Similar to Algorithm~\ref{alg_dscgdm}, we leverage the standard STORM technique \footnote{Compared with Algorithm~\ref{alg_dscgdm}, $\eta$ is replaced with $\eta^2$ when estimating each level function.} to estimate each level function, which is shown in  Line 7, where $\beta>0$ and $\beta\eta^2<1$.  Different from  Algorithm~\ref{alg_dscgdm}, we do not exploit the momentum to update model parameters. Instead, we leverage the variance-reduced gradient for local update, which is defined below:
\begin{equation}
	m_{n, t} = (1-\mu\eta^2)(m_{n, t-1} - {g}_{n,  t-1}^{\xi_{t}})+ {g}_{n,  t}^{\xi_{t}} \ , 
\end{equation}
where $\mu>0$ is a hyperparameter satisfying $\mu\eta^2<1$, the stochastic gradients ${g}_{n,  t}^{\xi_{t}}$ and ${g}_{n,  t-1}^{\xi_{t}}$ are defined as:
\begin{equation}
	\begin{aligned}
		& {g}_{ n, t-1}^{\xi_{t}}= \nabla  f_{n}^{(1)}({u}_{n, t-1}^{(0)}; \xi_{n,  t}^{(1)})  \nabla  f_{n}^{(2)}({u}_{n, t-1}^{(1)}; \xi_{n,  t}^{(2)})    \cdots \\
		& \quad \times \nabla  f_{n}^{(K-1)}({u}_{n, t-1}^{(K-2)}; \xi_{n,  t}^{(K-1)})  \nabla  f_{n}^{(K)}({u}_{n, t-1}^{(K-1)}; \xi_{n,  t}^{(K)}) \ ,  \\
		& {g}_{ n, t}^{\xi_{t}}= \nabla  f_{n}^{(1)}({u}_{n, t}^{(0)}; \xi_{n,  t}^{(1)})  \nabla  f_{n}^{(2)}({u}_{n, t}^{(1)}; \xi_{n,  t}^{(2)})   \cdots  \\
		& \quad \times \nabla  f_{n}^{(K-1)}({u}_{n, t}^{(K-2)}; \xi_{n,  t}^{(K-1)})  \nabla  f_{n}^{(K)}({u}_{n, t}^{(K-1)}; \xi_{n,  t}^{(K)}) \ . 
	\end{aligned}
\end{equation}
Then, based on this variance-reduced gradient, we exploit the gradient-tracking technique to update the model parameter on each device, which is shown in Lines 17 and 19. 

\begin{algorithm}[h]
	\caption{DSMCVRG}
	\label{alg_dscgdvr}
	\begin{algorithmic}[1]
		\REQUIRE ${x}_{n,0}={x}_{0}$, $\alpha>0$, $\beta>0$, $\mu>0$, $\eta>0$.

		\FOR{$t=0,\cdots, T-1$} 
		
		\STATE ${u}_{n,  t}^{(0)}= x_{n, t}$, 
		\FOR{$k=1,\cdots, K-1$} 
		
		\IF {$t==0$}
		\STATE With  batch size $S$, compute \\
		${u}_{n,  t}^{(k)}=f_{n}^{(k)}({u}_{n,  t}^{(k-1)}; \xi_{ n, t}^{(k)}) $, \quad \\
		${v}_{n,  t}^{(k)}=\nabla  f_{n}^{(k)}({u}_{ n, t}^{(k-1)}; \xi_{n,  t}^{(k)}) $, 
		\ELSE
		\STATE  ${u}_{n,  t}^{(k)}=(1-\beta\eta^2)({u}_{n,  t-1}^{(k)} - f_{n}^{(k)}({u}_{n,  t-1}^{(k-1)}; \xi_{n,  t}^{(k)}) )+  f_{n}^{(k)}({u}_{n,  t}^{(k-1)}; \xi_{ n, t}^{(k)}) $, \\
		 \STATE ${v}_{n,  t}^{(k)}=\nabla  f_{n}^{(k)}({u}_{ n, t}^{(k-1)}; \xi_{n,  t}^{(k)}) $, 
		\ENDIF
		\ENDFOR

		\IF {$t==0$}
		\STATE ${v}_{ n, t}^{(K)}=\nabla  f_{n}^{(k)}({u}_{n, t}^{(K-1)}; \xi_{n,  t}^{(K)}) $  with batch size $S$,
		\STATE  $m_{n, t} = {g}_{n,  t}^{\xi_{t}}$, \quad $y_{n, t}= m_{n,t}$,
		\ELSE
		\STATE ${v}_{ n, t}^{(K)}=\nabla  f_{n}({u}_{n, t}^{(K-1)}; \xi_{n,  t}^{(K)}) $, 
		\STATE $m_{n, t} = (1-\mu\eta^2)(m_{n, t-1} - {g}_{n,  t-1}^{\xi_{t}})+ {g}_{n,  t}^{\xi_{t}}$, \\ 
		\STATE $y_{n, t}=\sum_{n'\in \mathcal{N}_{n}}w_{nn'} y_{n, t-1} + m_{n,t} - m_{n, t-1} $,
		\ENDIF
		\STATE ${x}_{n,   t+\frac{1}{2}}=\sum_{n'\in \mathcal{N}_{n}}w_{nn'}{x}_{ n', t}  - \alpha {y}_{n,  t}$,  \\ 
		${x}_{n,   t+1}=  {x}_{n,   t} + \eta ({x}_{n,   t+\frac{1}{2}} - {x}_{n,   t})$,
		\ENDFOR
	\end{algorithmic}
\end{algorithm}

\paragraph{Novelty.} Here, we would like to emphasize the novelty on the algorithmic design in Algorithm~\ref{alg_dscgdvr}. Under the single-machine setting, existing variance-reduced multi-level compositional gradient descent algorithms \citep{zhang2021multilevel,jiang2022optimal} apply the variance-reduction technique to each level function and its stochastic Jacobian matrix. For instance, \cite{jiang2022optimal} computes the variance-reduced Jacobian matrix for each level function as follows:
\begin{equation}
	\begin{aligned}
		& {v}_{n,  t}^{(k)}=(1-\beta\eta^2)({v}_{n,  t-1}^{(k)} - \nabla f_{n}^{(k)}({u}_{n,  t-1}^{(k-1)}; \xi_{n,  t}^{(k)}) )  \\
		& \quad \quad +  \nabla f_{n}^{(k)}({u}_{n,  t}^{(k-1)}; \xi_{ n, t}^{(k)}) \ .
	\end{aligned}
\end{equation}
This kind of variance-reduced estimator for each level function suffers from some limitations. On the theoretical analysis side, when bounding the gradient estimation error for  $\nabla F_n(\cdot)$,  it requires ${v}_{n,  t}^{(k)}$ to be upper bounded in all levels and iterations. To do that, \cite{zhang2021multilevel} uses a clipping operation, which may result in a very tiny update (See $\gamma_t$ in Algorithm 3 of  \citep{zhang2021multilevel}), while \cite{jiang2022optimal} employs a projection operation to guarantee  ${v}_{n,  t}^{(k)}$ is upper bounded by the  Lipschitz constant of the deterministic Jacobian matrix  (See Eq.~(3) in \cite{jiang2022optimal}), which is an  unknown hyperparameter so that it is not feasible in practice.  On the implementation side, these algorithms are not friendly for practical applications. For instance, when applying them to the stochastic training of graph neural networks (GNN), computing the variance-reduced Jacobian matrix for each level function (i.e., each layer of GNN) requires to intervene the backpropagation in each layer, which is not easy to implement.

On the contrary, our Algorithm~\ref{alg_dscgdvr} just computes the standard stochastic Jacobian matrix ${v}_{n,  t}^{(k)}$ for each level function. This can naturally avoid the aforementioned impractical operations since the standard stochastic Jacobian matrix  is easy to bound under the commonly used assumptions. Meanwhile, it is easy to compute. However, using standard stochastic Jacobian matrix of each  level function may introduce a large estimation error. Then, a natural question follows: \textit{Can Algorithm~\ref{alg_dscgdvr} achieve  the $O(\epsilon^{-3})$ sample complexity as \citep{zhang2021multilevel,jiang2022optimal} when not using the variance reduction technique for each level function's Jacobian?} In Section 4, we provide an affirmative answer:  Our Algorithm~\ref{alg_dscgdvr} can still achieve the $O(\epsilon^{-3})$ sample complexity, even though we don't use the variance reduced Jacobian for each level function.  

All in all, our algorithm is novel and we believe our idea can be leveraged to improve  existing single-machine algorithms \citep{zhang2021multilevel,jiang2022optimal}.

\section{Convergence Analysis}
To establish the convergence rate of our algorithms, we introduce the following assumptions, which are commonly used in existing multi-level compositional optimization works \citep{yang2019multilevel,zhang2021multilevel,jiang2022optimal}. 

\begin{assumption} \label{assumption_smooth}
	For any $k \in \{1,2, \cdots, K\}$ and any $y_1, y_2 \in \mathbb{R}^{d_{k-1}}$,   there exists $L_{k} >0$ such that $\|\nabla f^{(k)}(y_1) - \nabla f^{(k)}(y_2)\| \leq L_{k} \|y_1 -y_2\|$ and $\mathbb{E}[\|\nabla f^{(k)}(y_1; \xi^{(k)}) - \nabla f^{(k)}(y_2; \xi^{(k)})\| ]\leq L_{k} \|y_1 -y_2\|$. 
	Additionally, $F_n(x)$ is $L_F$-smooth \footnote{Based on the smoothness of each level function, it is easy to prove $F_n$ is smooth \citep{yang2019multilevel,zhang2021multilevel,jiang2022optimal} so that we directly assume it is smooth.} where $L_F>0$. 
\end{assumption}

\begin{assumption} \label{assumption_bound_gradient}
	For any $k \in \{1,2, \cdots, K\}$ and any $y\in \mathbb{R}^{d_{k-1}}$,   there exists $C_k>0$ such that $ \mathbb{E}[\|\nabla f^{(k)}(y; \xi) \|^2] \leq C_k^2$ and $\|\nabla f^{(k)}(y) \|^2 \leq C_k^2 $. 
\end{assumption}

\begin{assumption} \label{assumption_bound_variance}
	For any $k\in \{1,2, \cdots, K\}$ and any $y_1, y_2 \in \mathbb{R}^{d_{k-1}}$,  there exist $\sigma_k>0$ and $\delta_k>0$  such that $\mathbb{E}[\|\nabla f^{(k)}(y; \xi)  - \nabla f^{(k)}(y) \|^2] \leq \sigma_{k}^2$ and $\mathbb{E}[\|   f^{(k)}(y; \xi) -   f^{(k)}(y) \|^2] \leq \delta_{k}^2$.
\end{assumption}
Based on these assumptions, we denote $A_k=  (\sum_{j=k}^{K-1}(\frac{L_{j+1}\prod_{i=1}^{K}C_i}{C_{j+1}} \prod_{i=k+1}^{j}C_{i}))^2$ and  $B_{k} = \frac{\prod_{j=1}^{K}C_j^2 }{C_k^2} $ for $ k\in\{1, \cdots, K-1\}$, as well as $D_{k}=\frac{(\prod_{j=1}^{K} C_{j}^2)L_{k+1}^2}{C_{k+1}^2}$ for $ k\in\{0, \cdots, K-1\} $.  Moreover, we use $\bar{z}_t$ to denote the mean value across devices for any variables throughout this paper. Then, we established the convergence rate of our two algorithms.

\begin{theorem} \label{theorem1}
	Given Assumptions~\ref{assumption_graph}-\ref{assumption_bound_variance}, by setting $\mu>0$, $\beta>0$,  $\alpha\leq \min \{{ (1-\lambda)^2}/\sqrt{\tilde{\alpha}_1}, 1/(4\sqrt{\tilde{\alpha}_2})\}$,   
	$\eta \leq \min\{\tilde{\omega}_k/(8\beta\sum_{j=1}^{K-1}  \tilde{\omega}_jC_j^2  \prod_{i=k+1}^{j}(2C_{i}^2)) , {1}/(2\alpha L_F),\\ {1}/{\beta},    {1}/{\mu}, 1\}$ for any $k\in \{1, 2, \cdots, K-1\}$,  Algorithm~\ref{alg_dscgdm} has the following convergence rate: 
	\vspace{-5pt}
	\begin{equation}
		\begin{aligned}
			& \frac{1}{T}\sum_{t=0}^{T-1} \mathbb{E}[\|\nabla F(\bar{{x}}_{t})\|^2]   \leq \frac{2(F({{x}}_{0}) - F({{x}}_{*}))}{\alpha\eta T}  +O( \frac{\mu K}{ T}) \\
			& \quad +O(\frac{K}{\eta T} ) + O( \frac{K}{\mu \eta T}) + O(\beta^2  \mu^2\eta^3   K)  +O(\mu^2\eta K )\\
			& \quad + O(\mu^3\eta^2 K)   +O( \beta^2\eta^2 K) + O(\mu \eta K) + O(\beta^2\eta K ) \ , \\
		\end{aligned}
	\end{equation}
	where  $\tilde{\omega}_k= \frac{2}{\beta} ((12A_k+8 D_{k})\mu  + 2  A_k  +  2 \beta  \sum_{j=k+1}^{K-1}( 20    A_j C_j^2+ 8 D_{j}  C_{j}^2)  \prod_{i=k+1}^{j}(2C_{i}^2) ) $ for $k\in \{1, 2, \cdots, K-1\}$, and $\tilde{\omega}_{K+1}= L_F^2+ 8(\frac{2  L_F^2}{\mu^2  }+ 8  L_F^2
	+  4 KD_{0} +  K\sum_{k=1}^{K-1}(20 A_k C_k^2+ 2 \tilde{\omega}_kC_k^2 + 8 D_{k}  C_{k}^2)(\prod_{j=1}^{k-1}(2C_{j}^2)) )$, $\tilde{\alpha}_1=4\tilde{\omega}_{K+1} + 
	8 L_F^2/\mu^2  + 32  L_F^2
	+  16 KD_{0}  +  4K\sum_{k=1}^{K-1}(20 A_k C_k^2+ 2 \tilde{\omega}_kC_k^2 + 8 D_{k}  C_{k}^2)(\prod_{j=1}^{k-1}(2C_{j}^2))$, $\tilde{\alpha}_2=2  L_F^2/\mu^2  + 8  L_F^2
	+  4 KD_{0} +  K\sum_{k=1}^{K-1}(20 A_k C_k^2+ 2 \tilde{\omega}_kC_k^2 + 8 D_{k}  C_{k}^2)(\prod_{j=1}^{k-1}(2C_{j}^2))$.

\end{theorem}

\begin{corollary} \label{corollary_stationary_point_1}
	Given Assumptions~\ref{assumption_graph}-\ref{assumption_bound_variance}, by setting $\mu=O(1)$, $\beta=O(1)$, $\alpha=O((1-\lambda)^2)$, $\eta=O(\epsilon^2)$,  $T=O((1-\lambda)^{-2}\epsilon^{-4})$, Algorithm~\ref{alg_dscgdm} can achieve the $\epsilon$-stationary point, i.e.,  $\frac{1}{T}\sum_{t=0}^{T-1} \mathbb{E}[\|\nabla F(\bar{{x}}_{t})\|^2 ] \leq \epsilon^2$. 
\end{corollary}

\begin{remark}
	Given  $\mu=O(1)$ and $\beta=O(1)$,  the hyperparameters $\tilde{\alpha}_i$ ($i=1, 2$) and $\tilde{\omega}_k$ ($k\in \{1, 2, \cdots, K-1, K+1\}$) are independent of the learning rate and spectral gap.  Thus, they do not affect the order of the convergence rate. 
\end{remark}

\begin{remark}
	From Corollary~\ref{corollary_stationary_point_1}, we can know that the convergence rate of Algorithm~\ref{alg_dscgdm} is $O((1-\lambda)^{-2}\epsilon^{-4})$, which is independent of the number of function levels. Meanwhile, it indicates that  the dependence on the spectral gap is $O((1-\lambda)^{-2})$. When the communication graph is fully connected, the convergence rate becomes $O(\epsilon^{-4})$, which can match the single-machine momentum algorithm \citep{balasubramanian2022stochastic}. All in all, the level-independent convergence rate is  achievable under the dencetralized setting. 
\end{remark}

\begin{remark}
	Since the mini-batch size is $O(1)$, the sample complexity is $O((1-\lambda)^{-2}\epsilon^{-4})$. Moreover,  the communication complexity is $O((1-\lambda)^{-2}\epsilon^{-4})$.
\end{remark}

\begin{theorem} \label{theorem2}
	Given Assumptions~\ref{assumption_graph}-\ref{assumption_bound_variance}, by setting  $\mu>0$, $\beta>0$,   $\alpha\leq \min \{{ (1-\lambda)^2}/\sqrt{\tilde{\alpha}_1}, 1/(4\sqrt{\tilde{\alpha}_2})\}$,  
	$\eta \leq \min\{ 0.5\sqrt{\tilde{\omega}_{k}/(2\beta\sum_{j=1}^{K-1}\tilde{\omega}_{j} C_j^2(\prod_{i=k+1}^{j}(2C_{i}^2)))}, {1}/(2\alpha L_F), \\{1}/\sqrt{\beta}, {1}/\sqrt{\mu}, 1\}$ for any $k\in \{1, 2, \cdots, K-1\}$,  Algorithm~\ref{alg_dscgdvr} has the following convergence rate:
	\begin{equation}
		\begin{aligned}
			&  \frac{1}{T}\sum_{t=0}^{T-1}\mathbb{E}[\|\nabla F(\bar{{x}}_{t})\|^2 ]  \leq \frac{2({F}(x_0) -F(x_*))}{\alpha\eta T}  +O( \frac{K}{\eta^2 TS})    \\
			& \quad+ O( \frac{K}{\mu\eta^2 TS})   +  O(\frac{\mu\eta  K}{ TS})    +  O(\frac{K}{\eta T}) +  O(\beta^2  \eta^3K) \\
			& \quad  + O(\mu^2\eta^3K) + O(\beta^2\eta^2 K) + O(\frac{\beta^2  \eta^2 K}{\mu})  \\
			& \quad + O( \mu\eta^2 K) +O( \mu \beta^2  \eta^5 K  )+ O(\mu^3\eta^5 K)   \ ,   \\
		\end{aligned}
	\end{equation}
	where  $\tilde{\omega}_{k}= \frac{16 D_{k} }{\mu N}+ 24   D_{k}   + \frac{4A_{k} }{\beta}  + 16 \sum_{j=1}^{K-1}((\frac{2}{\mu N} +3 )D_{j} C_{j}^2 )  (\prod_{i=k+1}^{j}(2C_{i}^2))$ for $ k\in \{1, 2, \cdots, K-1\}$,  $\tilde{\omega}_{K+2} = 16( \sum_{k=1}^{K-1}((\frac{8K}{\mu N }  +  12  K )D_{k} C_{k}^2 +  2\tilde{\omega}_{k}  C_k^2)(\prod_{j=1}^{k-1}(2C_{j}^2))+ \frac{ 4KD_{0}}{\mu N} + 6   KD_{0} )+  2L_F^2 $,  $\tilde{\alpha}_1=2\tilde{\omega}_{K+2} +4 K [ \sum_{k=1}^{K-1}((\frac{8}{\mu N}  +  12   )D_{k} C_{k}^2 +  2\tilde{\omega}_{k}  C_k^2)(\prod_{j=1}^{k-1}(2C_{j}^2))+ \frac{ 4D_{0}}{\mu N} + 6  D_{0} ] $, $\tilde{\alpha}_2=K\sum_{k=1}^{K-1}((\frac{8}{\mu N }  +  12   )D_{k} C_{k}^2 +  2\tilde{\omega}_{k}  C_k^2)(\prod_{j=1}^{k-1}(2C_{j}^2))+ \frac{ 4KD_{0}}{\mu N} + 6  KD_{0} $.

\end{theorem}

\begin{corollary} \label{corollary_stationary_point_2}
	Given Assumptions~\ref{assumption_graph}-\ref{assumption_bound_variance}, by setting $\mu=O(1)$, $\beta=O(1)$, $\alpha=O((1-\lambda)^2)$, $S=O(\epsilon^{-1})$, $\eta=O(\epsilon)$,  $T=O((1-\lambda)^{-2}\epsilon^{-3})$, Algorithm~\ref{alg_dscgdvr} can achieve $\epsilon$-stationary point.
\end{corollary}
\begin{remark}
	Given  $\mu=O(1)$ and $\beta=O(1)$,  the hyperparameters $\tilde{\alpha}_i$ ($i=1, 2$) and $\tilde{\omega}_k$ ($k\in \{1, 2, \cdots, K-1, K+2\}$) also do not affect the order of the convergence rate. 
\end{remark}

\begin{remark}
	From Corollary~\ref{corollary_stationary_point_2}, we can know that the convergence rate of Algorithm~\ref{alg_dscgdvr} is $O((1-\lambda)^{-2}\epsilon^{-3}))$, which is also independent of the number of function levels and has  the dependence on the spectral gap with $O((1-\lambda)^{-2})$. Moreover, this convergence rate is better than  Algorithm~\ref{alg_dscgdm}. Additionally, when the communication graph is fully connected, the convergence rate  can match the single-machine  algorithms \citep{zhang2021multilevel,jiang2022optimal}, but our Algorithm~\ref{alg_dscgdvr} requires much milder operations than \citep{zhang2021multilevel,jiang2022optimal}.
\end{remark}

\begin{remark}
	Since the mini-batch size is $O(1)$ except the first iteration, the sample complexity is $O((1-\lambda)^{-2}\epsilon^{-3}))$. Similarly, we can know that the communication complexity is $O((1-\lambda)^{-2}\epsilon^{-3}))$.
\end{remark}

\textbf{Discussions.} Due to the multi-level nested structure and the decentralized communication scheme, it is quite challenging to establish the convergence rate of our algorithms. Specifically, compared with the decentralized  two-level compositional optimization problem, the multi-level nested structure makes the convergence analysis more difficult. For instance, when bounding $ \mathbb{E}[\|u_{n, t}^{(k)} - f_{n}^{(k)}(u_{n, t}^{(k-1)}) \|^2]$ in Lemma~\ref{lemma_u_var_momentum}, its upper bound depends on the update of the lower-level function estimator $\mathbb{E}[\|u_{n, t-1}^{(k-1)}- u_{n,t}^{(k-1)} \|^2] $, which further has a quite complicated upper bound as below:
\vspace{-5pt}
\begin{equation} \label{eq_u_inc}
	\begin{aligned}
		&  \mathbb{E}[\|u_{n,t}^{(k-1)}  - u_{n, t-1}^{(k-1)} \|^2 ] \leq  \Big(\prod_{j=1}^{k-1}(2C_{j}^2)\Big)\mathbb{E}[\|u_{n, t-1}^{(0)} - u_{n,t}^{(0)} \|^2] \\
		& \quad + 2\beta^2  \eta^2\sum_{j=1}^{k-1} \Big(\prod_{i=j+1}^{k-1}(2C_{i}^2)\Big)\mathbb{E}[\|u_{n, t-1}^{(j)}- f_{n}^{(j)}(u_{n, t-1}^{(j-1)})  \|^2 ]  \\
		& \quad + 2\beta^2  \eta^2\sum_{j=1}^{k-1} \Big(\prod_{i=j+1}^{k-1}(2C_{i}^2)\Big)\delta_{j}^2  \ . \\
	\end{aligned}
\end{equation}
On the contrary, in the two-level compositional optimization problem, $\mathbb{E}[\|u_{n, t-1}^{(k-1)}- u_{n,t}^{(k-1)} \|^2] $ becomes the update of model parameters, which is much easier to bound.  On the other hand, compared with the single-machine multi-level compositional optimization problem, $\mathbb{E}[\|u_{n, t-1}^{(0)} - u_{n,t}^{(0)} \|^2] $ in Eq.~(\ref{eq_u_inc}) involves the decentralized communication operation, which makes it more difficult to bound.  

Furthermore, the  multi-level structure and the decentralized communication scheme  bring more challenges to bound the consensus error, e.g., Lemma~\ref{lemma_y_consensus_momentum} and Lemma~\ref{lemma_y_consensus_var}.  Last but not least, our algorithm does not apply the variance-reduction technique to the stochastic Jacobian matrix of each level function. Thus, we need to carefully bound the gradient estimation error to guarantee the desired convergence rate. This has never been studied before so that we need to develop new strategies to bound the gradient  estimation error, e.g., Lemma~\ref{lemma_m_var_var}. All in all, the theoretical analysis is challenging.


To address those challenges, we developed novel potential functions to establish the convergence rate of our algorithms. In particular, to prove Theorem~\ref{theorem1}, we proposed the following potential function:
\begin{equation}
	\begin{aligned}
			& 	\mathcal{H}_{t} = \mathbb{E}[F({\bar{x}}_{t}) ] + \omega_0\frac{1}{N}\sum_{n=1}^{N}\mathbb{E}\Big[\Big\|{m}_{n,t} -\nabla F_{n}({{x}}_{n,t})\Big\|^2\Big]  \\
			& + \frac{1}{N} \sum_{n=1}^{N}\sum_{k=1}^{K-1}\omega_k \mathbb{E}[\|u_{n, t}^{(k)} -f_{n}^{(k)}(u_{n, t}^{(k-1)}) \|^2]   \\
			& + \omega_{K}\mathbb{E}\Big[\Big\|\frac{1}{N}\sum_{n=1}^{N}{m}_{n,t} -\frac{1}{N}\sum_{n=1}^{N}\nabla F_{n}({{x}}_{n,t})\Big\|^2\Big]   \\
			& + \omega_{K+1}\frac{1}{N} \mathbb{E}[\|X_{t} - \bar{X}_{t}\|_{F}^2 ]+ \omega_{K+2}\frac{1}{N} \mathbb{E}[\|Y_{t} - \bar{Y}_{t}\|_{F}^2 ]  \ , \\
		\end{aligned}
\end{equation}
where $\omega_i>0$ ($i\in \{0, 1, \cdots, K+2\}$) are determined in our proof, which actually is challenging  due to  the interaction between the  multi-level structure and the decentralized communication scheme. 

Moreover, since this potential function cannot be applied to Theorem~\ref{theorem2}, we proposed the following potential function to prove Theorem~\ref{theorem2}: 
\begin{equation}
	\begin{aligned}
			& 	\mathcal{H}_{t} = \mathbb{E}[F({\bar{x}}_{t}) ] + \frac{1}{N} \sum_{n=1}^{N}\sum_{k=1}^{K-1}\omega_k \mathbb{E}[\|u_{n, t}^{(k)} -f_{n}^{(k)}(u_{n, t}^{(k-1)}) \|^2 ]  \\
			& + \omega_{K} \mathbb{E}[\|\bar{m}_{t}- \bar{h}_{t}\|^2]   + \omega_{K+1}\frac{1}{N} \sum_{n=1}^{N}\mathbb{E}[\|{m}_{n,t} -h_{n,t}\|^2 ] \\
			& + \omega_{K+2}\frac{1}{N}\mathbb{E}[\|X_{t} - \bar{X}_{t}\|_{F}^2] + \omega_{K+3}\frac{1}{N}\mathbb{E}[\|Y_{t} - \bar{Y}_{t}\|_{F}^2 ] \ ,
		\end{aligned}
\end{equation}
where $h_{n,t}= \nabla f_{n}^{(1)} ( u_{n,t}^{(0)})\cdots\nabla f_{n}^{(K)}( u_{n,t}^{(K-1)}) $, and $\omega_i>0$ ($i\in \{1, \cdots, K+3\}$) are determined in our proof. 

\begin{figure*}[t]
	\centering 
	\hspace{-10pt}
	\subfigure[Ring Graph (Support)]{
		\includegraphics[scale=0.27]{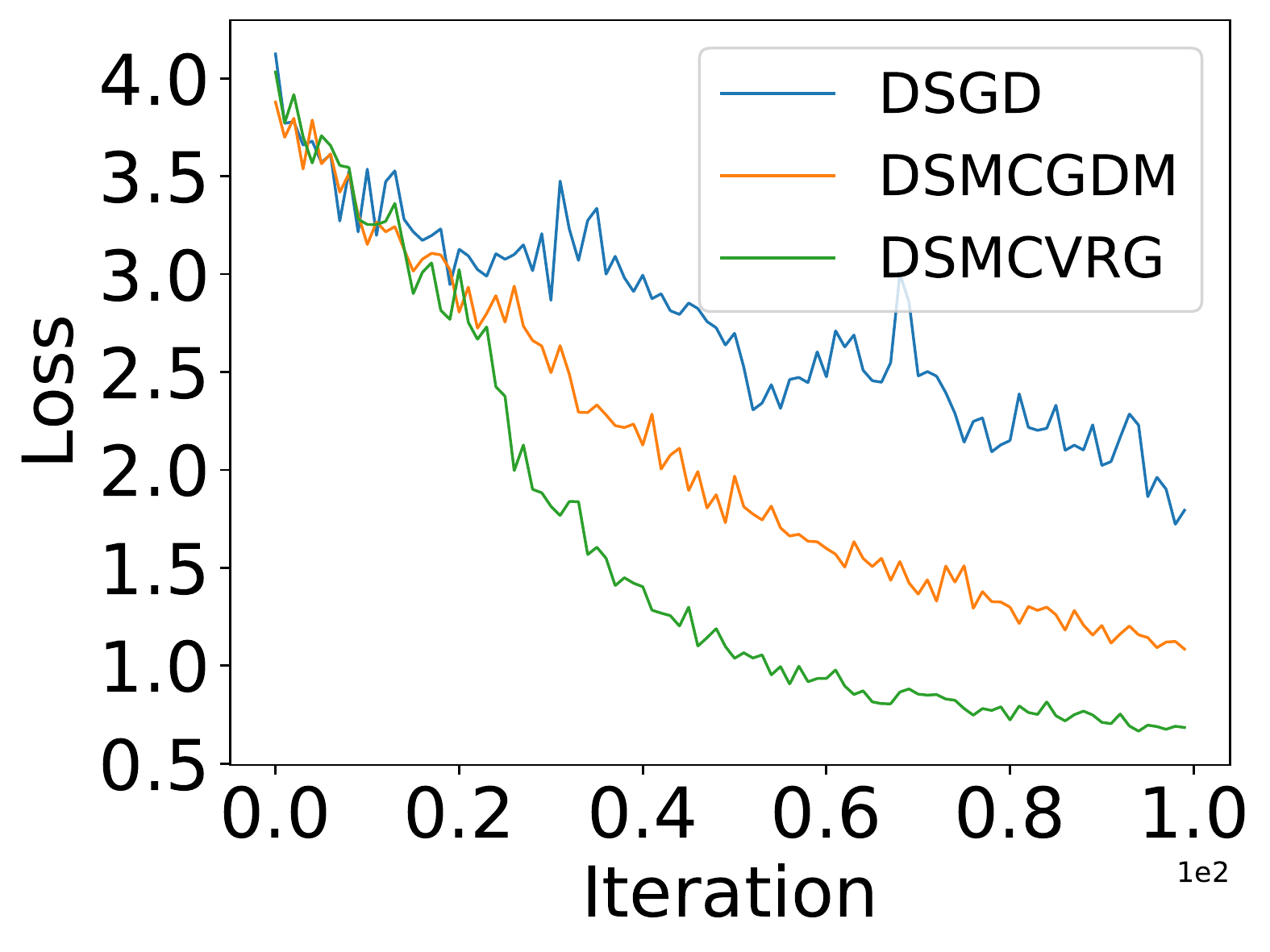}
	}
	\hspace{-15pt}
	\subfigure[Ring Graph (Query)]{
		\includegraphics[scale=0.27]{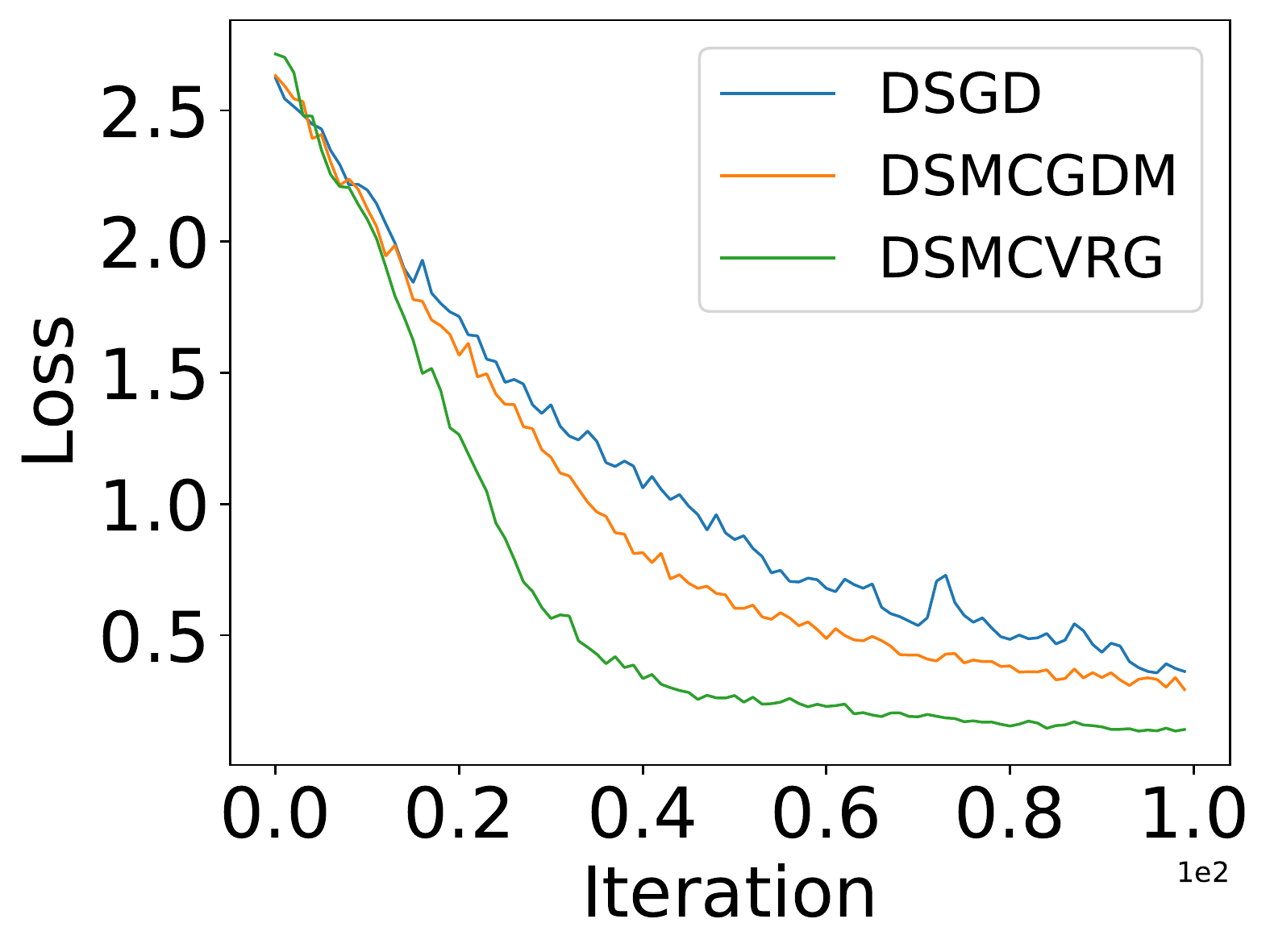}
	}
	\hspace{-15pt}
	\subfigure[Random Graph (Support)]{
		\includegraphics[scale=0.27]{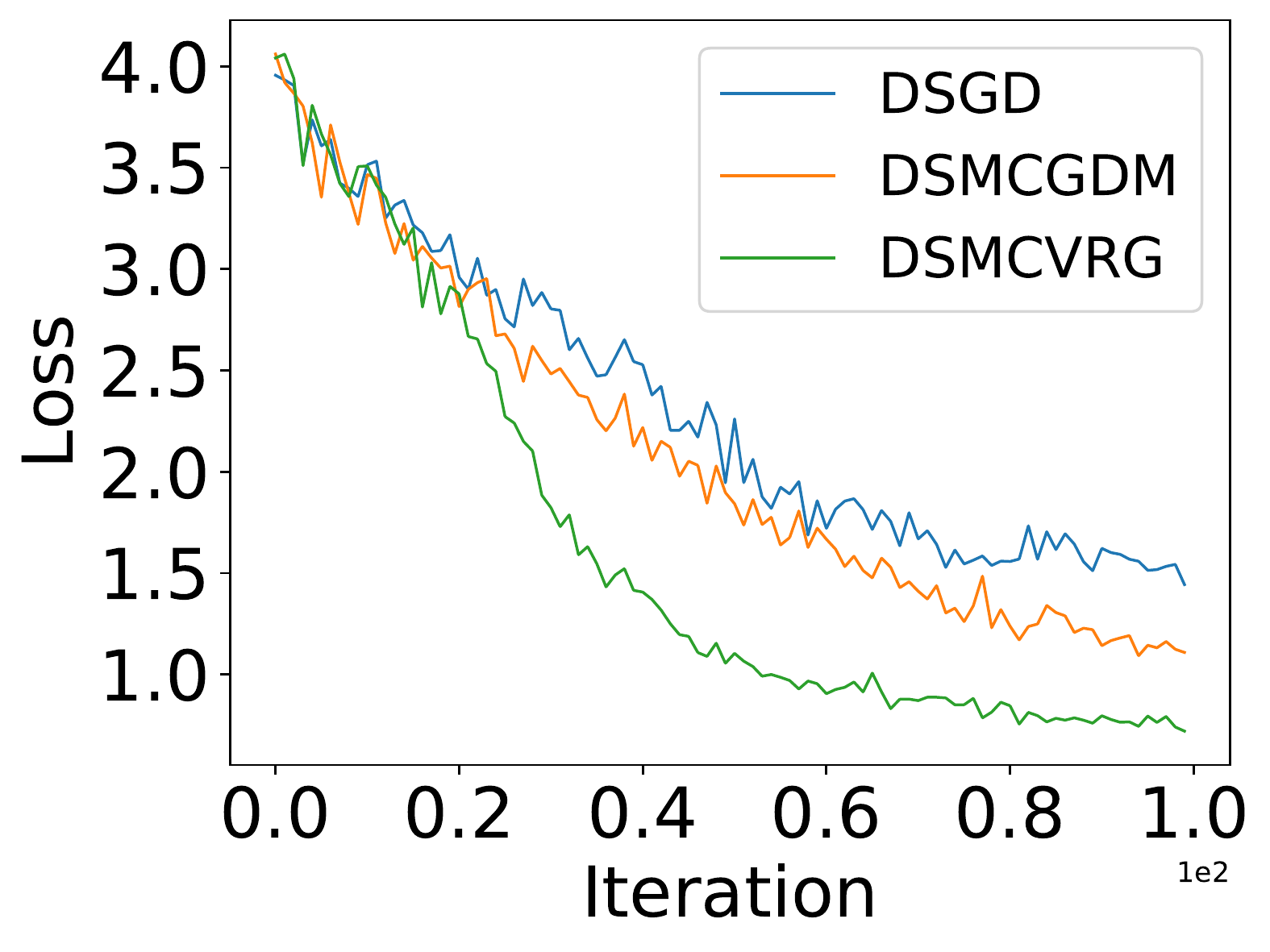}
	}
	\hspace{-15pt}
	\subfigure[Random Graph (Query)]{
		\includegraphics[scale=0.27]{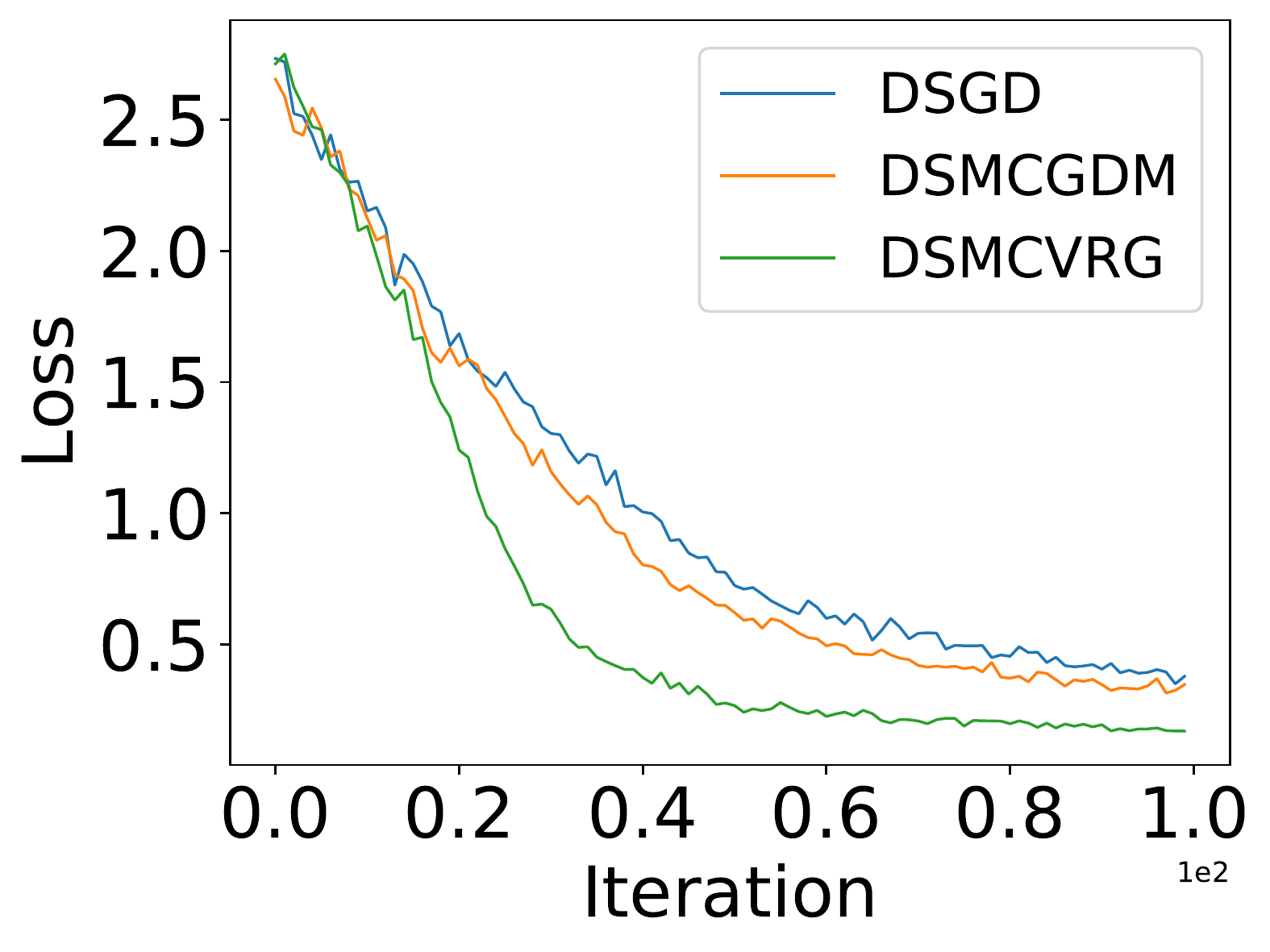}
	}
	\caption{Regression: The  loss function value on support and query sets versus the number of iterations for the ring and random graph.  }
	\label{regression}
\end{figure*}

\begin{figure*}[ht]
	\centering 
	\hspace{-10pt}
	\subfigure[Ring Graph (Support)]{
		\includegraphics[scale=0.27]{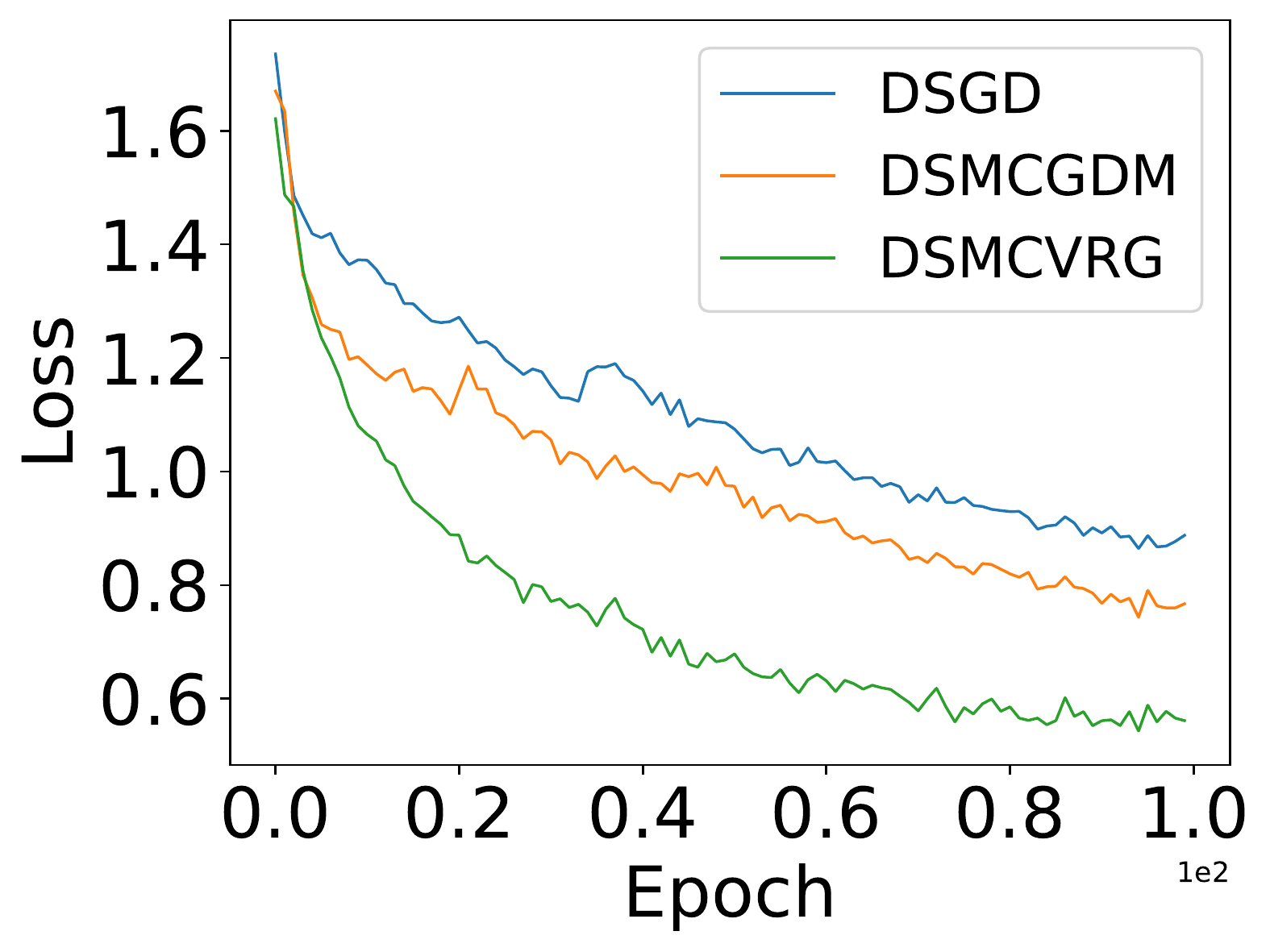}
	}
	\hspace{-15pt}
	\subfigure[Ring Graph (Query)]{
		\includegraphics[scale=0.27]{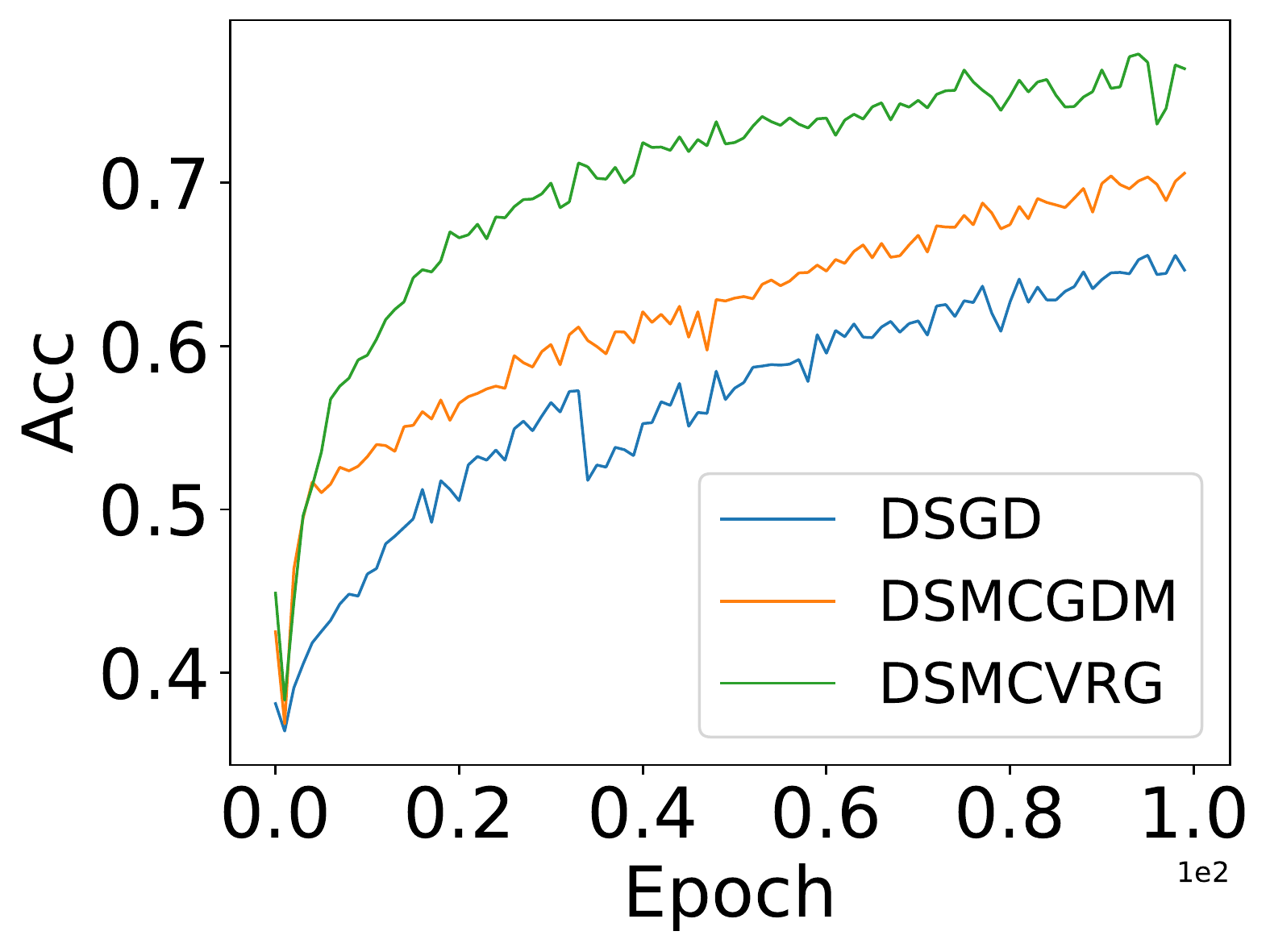}
	}
	\hspace{-15pt}
	\subfigure[Random Graph (Query)]{
		\includegraphics[scale=0.27]{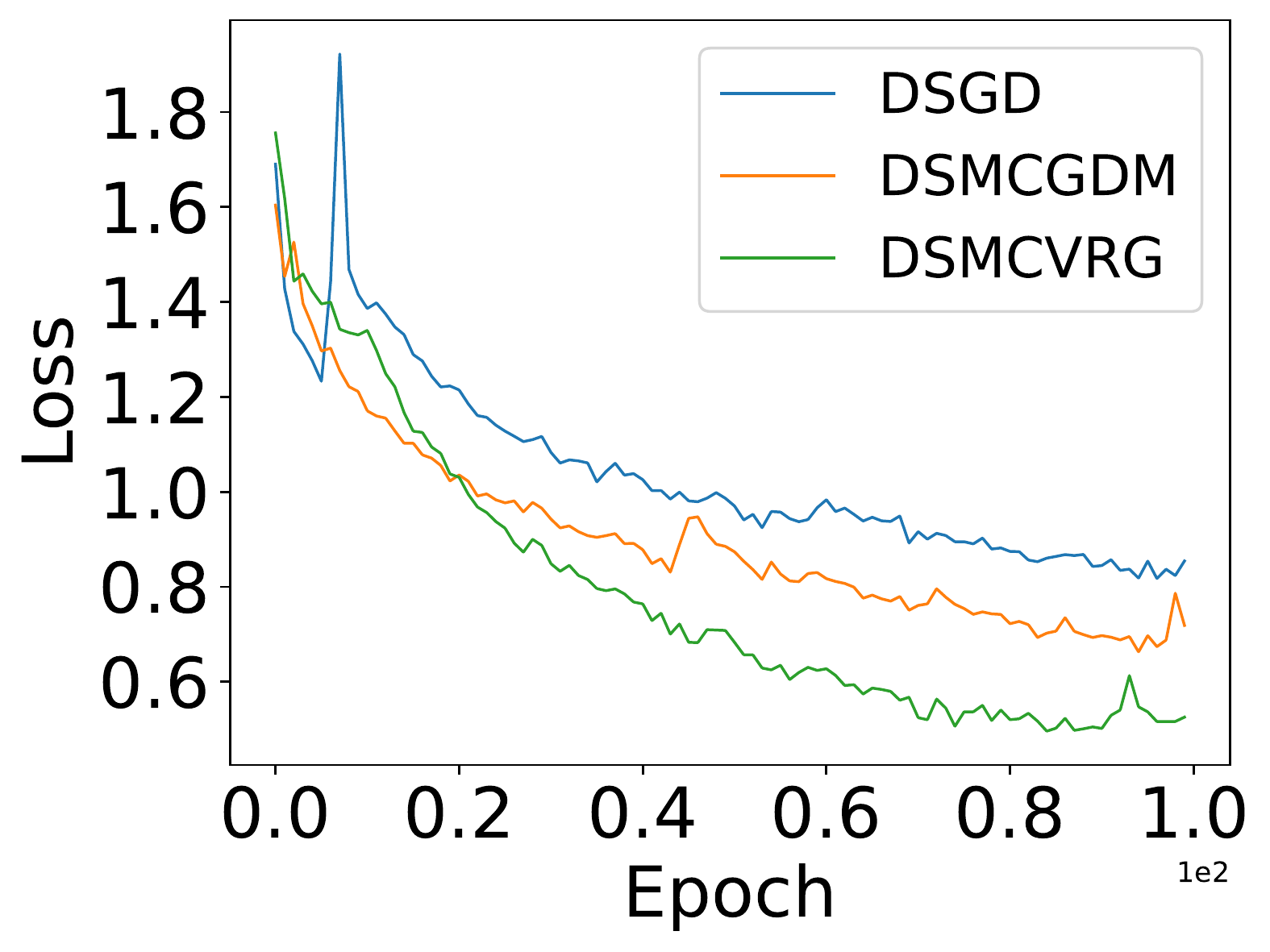}
	}
	\hspace{-15pt}
	\subfigure[Random Graph (Support)]{
		\includegraphics[scale=0.27]{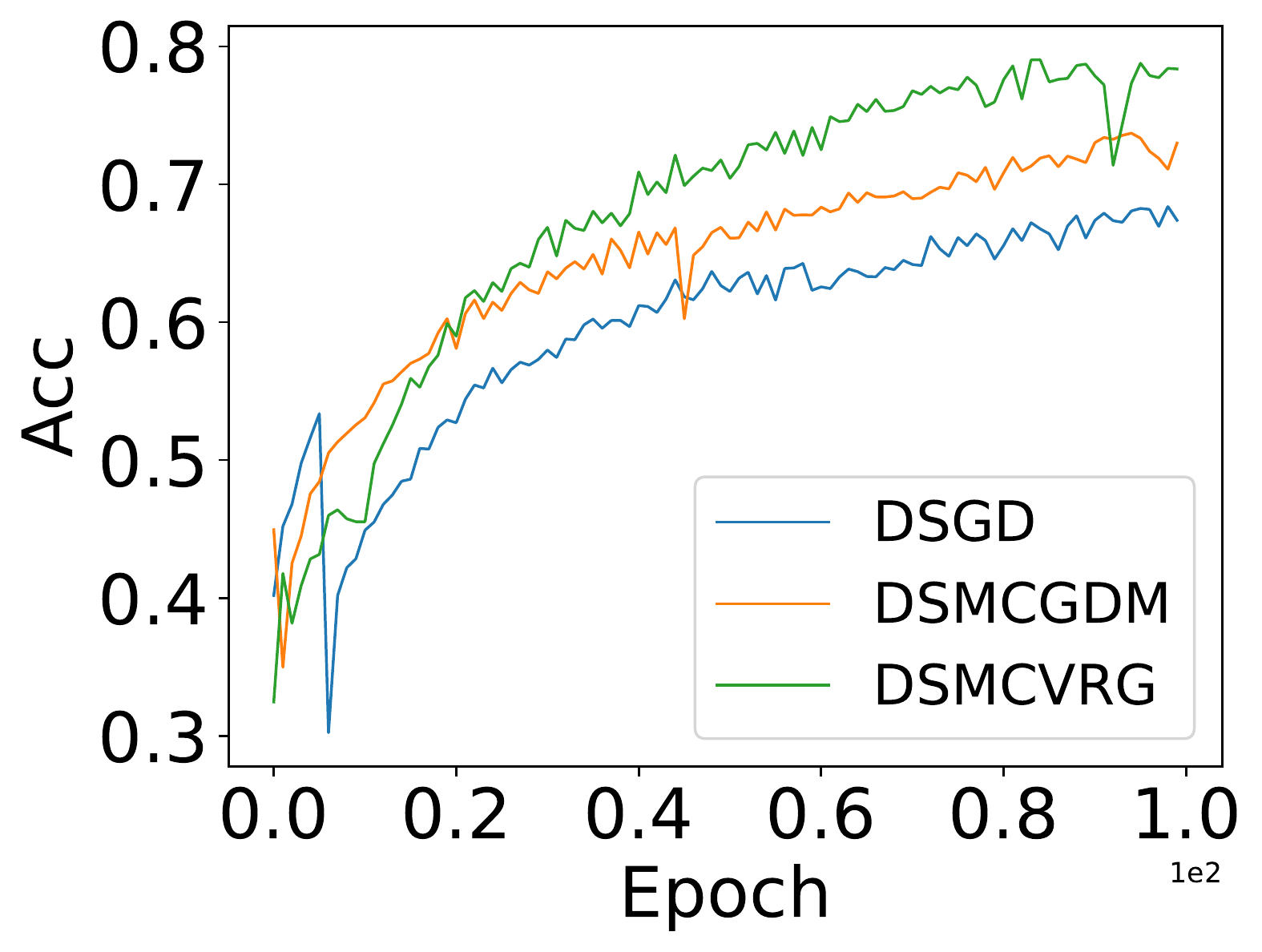}
	}
	\caption{Classification: The loss function value on support set and test accuracy on query set versus the number of epochs for the ring and random graph.  }
	\label{classification}
\end{figure*}

Based on these two novel potential functions, the task boils down to studying how each term evolves across iterations and determining its coefficient. The detailed proof can be found in Appendix.

\section{Experiment}
In this section, we apply our proposed algorithms to the multi-step model-agnostic meta-learning task to verify the performance of our algorithms.

\subsection{Multi-Step Model-Agnostic Meta-Learning}
Model-agnostic meta-learning (MAML) \citep{finn2017model} is to learn an initialization model that can be adapted to a new task via a couple of steps of stochastic gradient descent. Basically, the one-step MAML under the decentralized setting is defined as below:
\begin{align}
	&	\min _{x  \in \mathbb{R}^d}  \frac{1}{N}\sum_{n=1}^{N} \mathbb{E}_{i\sim \mathcal{P}_{n, \text{task}}, \zeta_{n} \sim \mathcal{D}_{n, \text {query}_i}}  \left[\mathcal{L}_{n,i}\left(y; \zeta_{n}\right)\right], \\
	& \text{where}  \ 	y =  x- \nu  \mathbb{E}_{\xi_n \sim \mathcal{D}_{n, \text {support}_i}}\nabla \mathcal{L}_{n,i}\left(x;  \xi_{n}\right)  \label{eq_mamal_one_step} \ , 
\end{align}
where Eq.~(\ref{eq_mamal_one_step}) denotes one-step gradient descent, $\nu$ is the learning rate, $\mathcal{P}_{n, \text{task}}$ denotes the task distribution on the $n$-th device, $\mathcal{D}_{n, \text {query}_i}$ ($\mathcal{D}_{n, \text {support}_i}$) represents the query (support) set of the $i$-th task on the $n$-th device. This one-step update can be viewed as a two-level compositional optimization problem. If taking multiple gradient descent steps, this problem becomes a multi-level compositional optimization problem \citep{jiang2022optimal,chen2020solving}. Therefore, we can apply our algorithms to the multi-step MAML problem. In our experiment, we will focus on two tasks: regression and classification tasks.

\subsection{Experimental Settings and Results}
\textbf{Regression.} For the regression problem, we follow \citep{finn2017model} to generate a sinewave dataset. Specifically, when generating the sine wave, the amplitude is randomly picked from $[0.1, 5.0]$, the phase is from $[0, \pi]$, and the input is from $[-5, 5]$. The model used for this task is a fully-connected neural network with the dimensionality as $[1,40,40, 1]$. For the support set, the meta-batch size (tasks) on each device is set to 200 and the number of samples for each task is 10. For the query set, the meta-batch size is 500 and the number of samples in each task is also 10. Moreover, the number of gradient descent updates in Eq.~(\ref{eq_mamal_one_step}) is $3$ so that it is a four-level compositional optimization problem. 
The learning rate $\nu$ is  $0.01$.

\textbf{Classification.} In this experiment, we use Omniglot dataset, which has 1,623 characters (tasks) and each character has 20 images. 1,200 tasks are used as the support set and the left tasks are used as the query set. Following \citep{finn2017model}, we employ the 5-way-1-shot setting.  The model we used has four convolutional layers, where each layer has 64 $3\times 3$ filters, and one linear layer.  The meta-batch size (tasks) on each device is set to 8. The number of gradient descent updates in Eq.~(\ref{eq_mamal_one_step}) is set to 3 so that it is also a four-level compositional optimization problem. The learning rate $\nu$ is  $0.01$ too.

In our experiments, we select $\mu$ and $\beta$ from $\{1, 3, 5, 7, 9\}$, and fix $\alpha$ to $1.0$. Additionally, we set $\epsilon^2=0.1$. Then,  we set the learning rate $\eta=\epsilon^2$ for Algorithm~\ref{alg_dscgdm} in terms of Corollary~\ref{corollary_stationary_point_1}, and $\eta=\epsilon$ for Algorithm~\ref{alg_dscgdvr} according to Corollary~\ref{corollary_stationary_point_2}. Moreover, we use four devices in our experiments. The topology we used includes the ring graph and random graph. Here, the random graph is generated from an Erdos-Renyi random graph with the edge probability being 0.4. As for the baseline algorithm, we use the standard decentralized SGD (DSGD) \citep{lian2017can} since there does not exist other decentralized multi-level compositional  algorithms. In our experiments, the learning rate of DSGD is 0.1

In Figure~\ref{regression}, we report the support and query loss function values versus the number of iterations for the regression task. It is easy to find that our two algorithms outperform the standard DSGD algorithm. The reason is that our algorithms leverage the variance-reduction technique to control the estimation error for each level function. Moreover, our second  algorithm DSMCVRG converges faster than the first algorithm DSMCGDM, which confirms the correctness of  our theoretical results. 

In Figure~\ref{classification}, we show the  loss function value on the support set and the  accuracy on the query set for the classification task. It can also be found that our two algorithms outperform the baseline algorithm and DSMCVRG converges faster than DSMCGDM, which further confirms the correctness of our theoretical results.

\subsection{More Experiments}

To  further demonstrate the performance of our algorithms, we set the number of inner steps of multi-step MAML to 4 and 5 so that we have the five-level and six-level compositional optimization problems. In Figure~\ref{highlevel}, we show the  loss function values on the support set versus the number of iterations for sinewave dataset. From this figure, we can still find that our two algorithms outperform DSGD and our second algorithm DSMCVRG converges faster than DSMCGDM, which confirms the effectiveness and correctness of our proposed algorithms. 
\begin{figure}[h]
	\centering 
	\hspace{-10pt}
	\subfigure[Five-level SCO problem]{
		\includegraphics[scale=0.26]{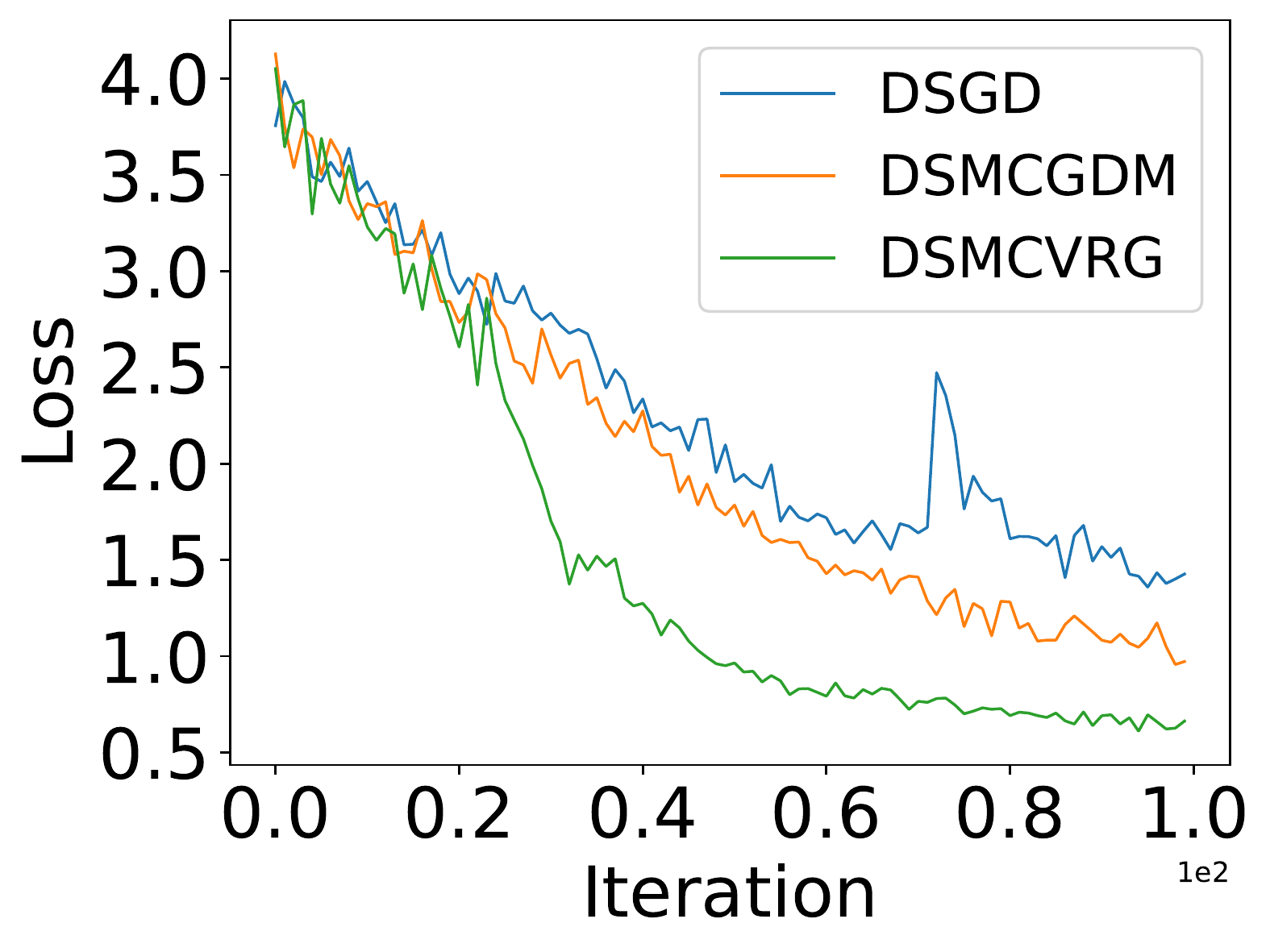}
	}
	\hspace{-15pt}
	\subfigure[Six-level SCO problem]{
		\includegraphics[scale=0.26]{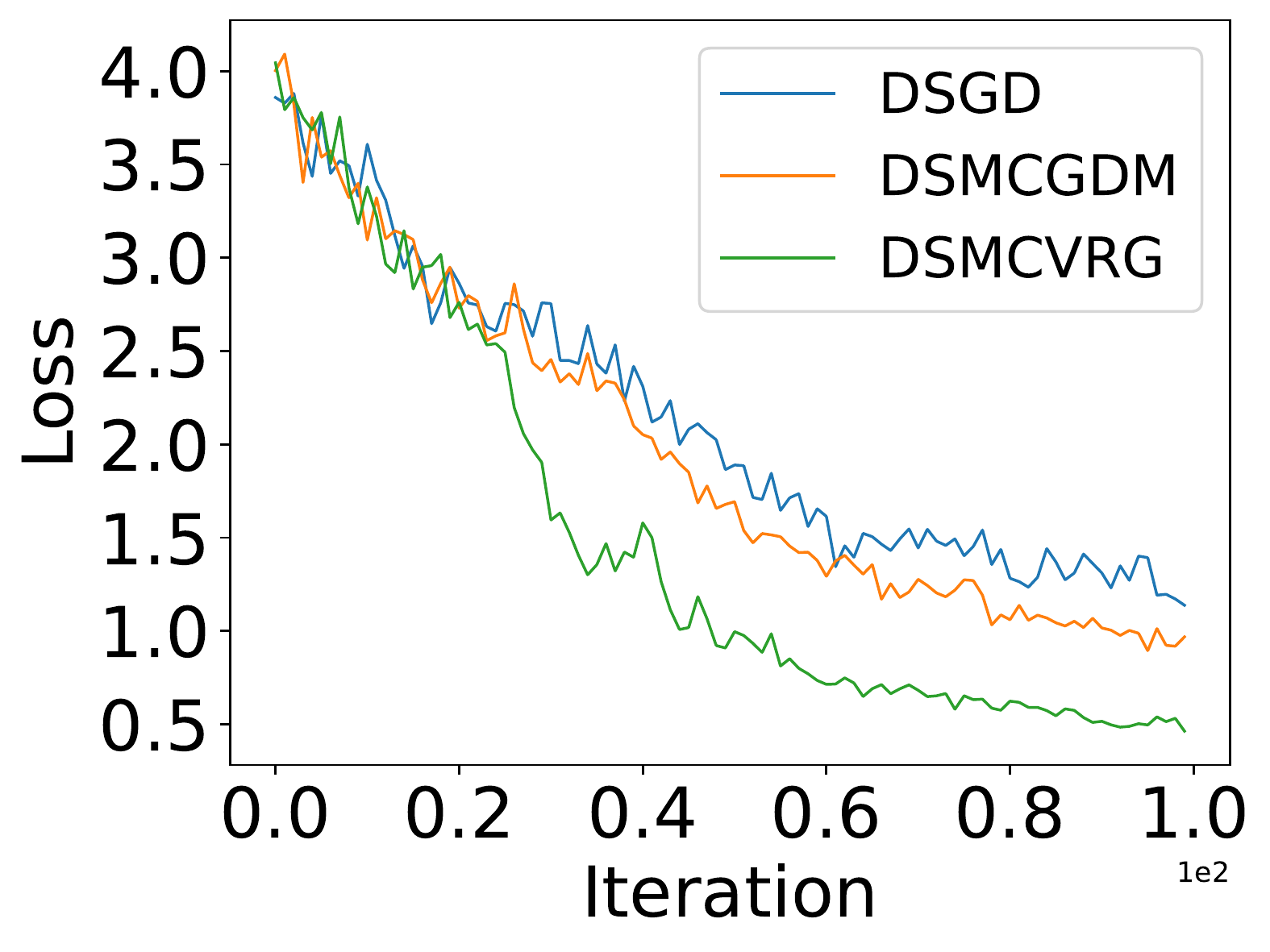}
	}
	\caption{The  loss function value on support test versus the number of iterations for the regression task.  Ring graph is used. }
	\label{highlevel}
\end{figure}

Moreover, we show the acceleration benefit of our decentralized optimization algorithms. In particular, we compare the convergence performance when using four and eight devices. Here, the meta-batch size is set to 200 when using four devices and it is set to 100 when using eight devices. Other hyperparameters are the same as previous experiments. In Figure~\ref{acceleration}, we show the  loss function value on the support set versus the consumed time for the regression task when using the ring graph. It is easy to find that using more devices can accelerate the convergence speed, which confirms the efficacy of our algorithms. 
\begin{figure}[h]
	\centering 
	\hspace{-10pt}
	\subfigure[DSMCGDM]{
		\includegraphics[scale=0.266]{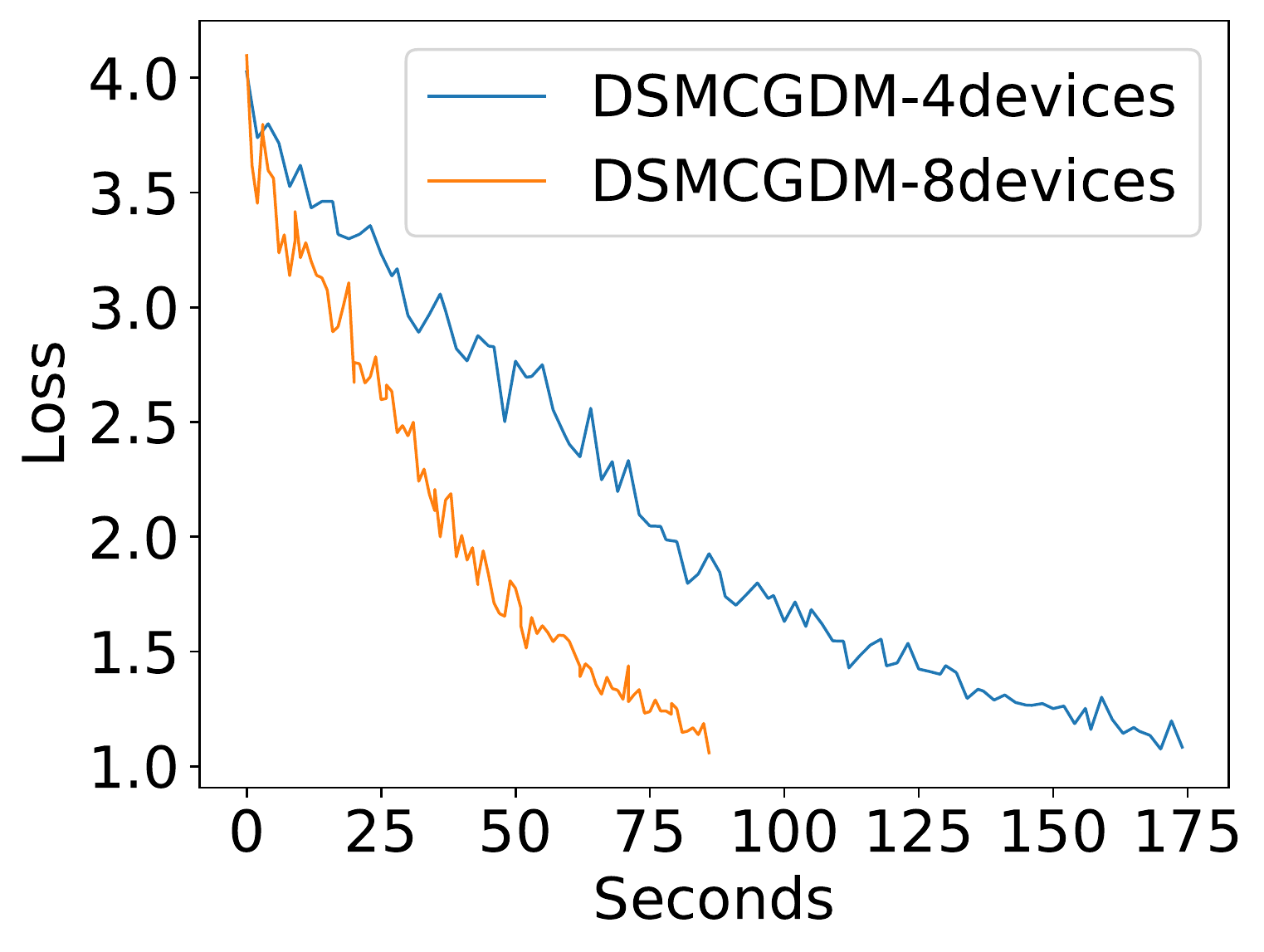}
	}
	\hspace{-15pt}
	\subfigure[DSMCVRG]{
		\includegraphics[scale=0.266]{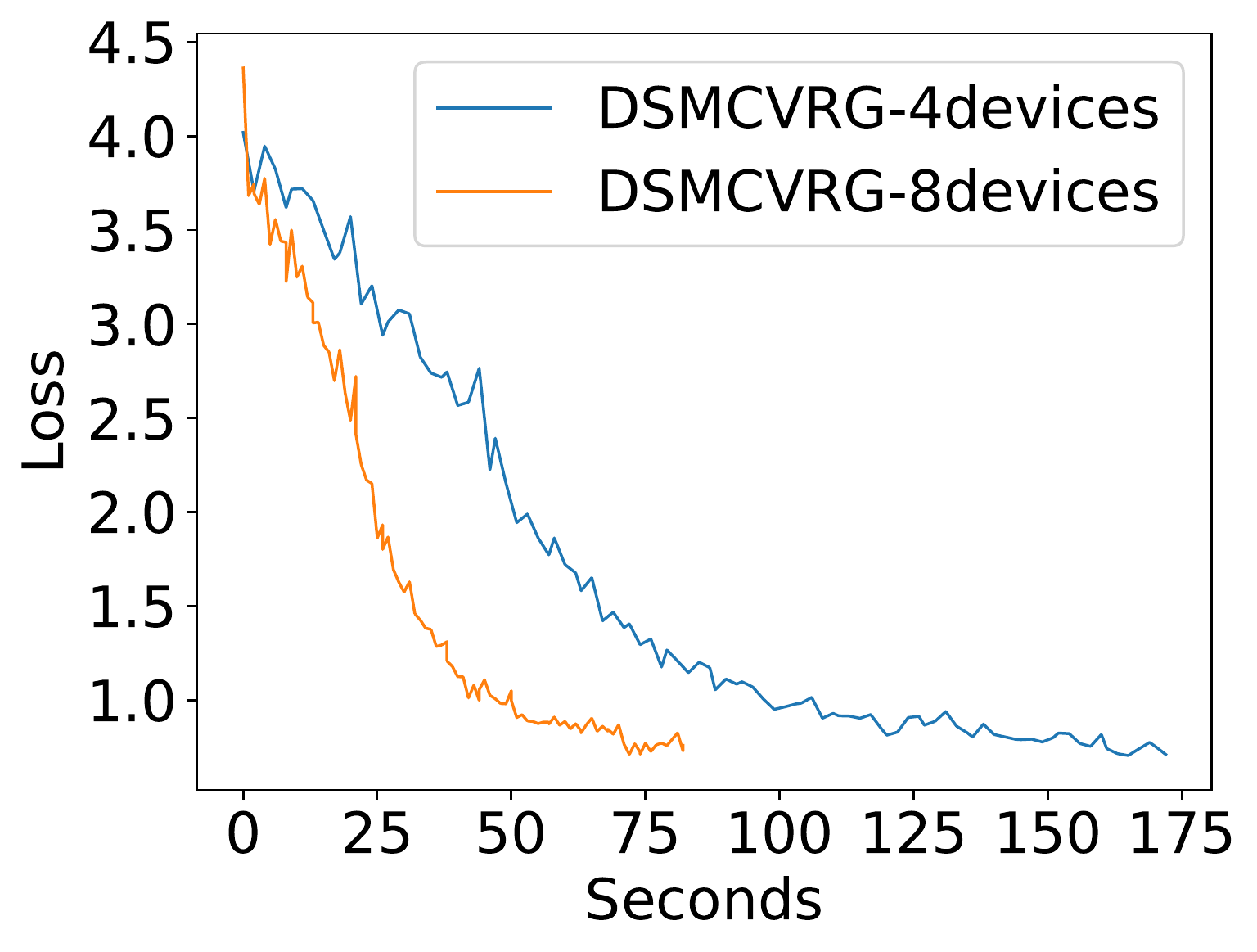}
	}
	\caption{The  loss function value on support set versus the consumed time for the regression task and ring graph.  }
	\label{acceleration}
\end{figure}

\section{Conclusion}
In this paper, we developed two novel decentralized stochastic multi-level compositional optimization algorithms. They both can achieve the level-independent convergence rate with practical operations.  In particular, we developed a novel strategy for applying the variance reduction technique to estimate the gradient. 
Extensive experimental results confirm the effectiveness  of  our algorithms. We believe our novel algorithmic design and theoretical analysis strategies can benefit the development of multi-level compositional optimization problems for both single-machine and distributed settings.

\bibliographystyle{abbrvnat}

\bibliography{example_paper}

\begin{thebibliography}{33}
\providecommand{\natexlab}[1]{#1}
\providecommand{\url}[1]{\texttt{#1}}
\expandafter\ifx\csname urlstyle\endcsname\relax
  \providecommand{\doi}[1]{doi: #1}\else
  \providecommand{\doi}{doi: \begingroup \urlstyle{rm}\Url}\fi

\bibitem[Balasubramanian et~al.(2022)Balasubramanian, Ghadimi, and
  Nguyen]{balasubramanian2022stochastic}
K.~Balasubramanian, S.~Ghadimi, and A.~Nguyen.
\newblock Stochastic multilevel composition optimization algorithms with
  level-independent convergence rates.
\newblock \emph{SIAM Journal on Optimization}, 32\penalty0 (2):\penalty0
  519--544, 2022.

\bibitem[Chen et~al.(2020)Chen, Sun, and Yin]{chen2020solving}
T.~Chen, Y.~Sun, and W.~Yin.
\newblock Solving stochastic compositional optimization is nearly as easy as
  solving stochastic optimization.
\newblock \emph{arXiv preprint arXiv:2008.10847}, 2020.

\bibitem[Cong et~al.(2021)Cong, Ramezani, and Mahdavi]{cong2021importance}
W.~Cong, M.~Ramezani, and M.~Mahdavi.
\newblock On the importance of sampling in training gcns: Tighter analysis and
  variance reduction.
\newblock \emph{arXiv e-prints}, pages arXiv--2103, 2021.

\bibitem[Cutkosky and Orabona(2019)]{cutkosky2019momentum}
A.~Cutkosky and F.~Orabona.
\newblock Momentum-based variance reduction in non-convex sgd.
\newblock In \emph{Advances in Neural Information Processing Systems}, pages
  15236--15245, 2019.

\bibitem[Fang et~al.(2018)Fang, Li, Lin, and Zhang]{fang2018spider}
C.~Fang, C.~J. Li, Z.~Lin, and T.~Zhang.
\newblock Spider: Near-optimal non-convex optimization via stochastic
  path-integrated differential estimator.
\newblock \emph{Advances in Neural Information Processing Systems}, 31, 2018.

\bibitem[Finn et~al.(2017)Finn, Abbeel, and Levine]{finn2017model}
C.~Finn, P.~Abbeel, and S.~Levine.
\newblock Model-agnostic meta-learning for fast adaptation of deep networks.
\newblock In \emph{International Conference on Machine Learning}, pages
  1126--1135. PMLR, 2017.

\bibitem[Gao(2023)]{gao2023achieving}
H.~Gao.
\newblock Achieving linear speedup in decentralized stochastic compositional
  minimax optimization.
\newblock \emph{arXiv preprint arXiv:2307.13430}, 2023.

\bibitem[Gao and Huang(2020)]{gao2020periodic}
H.~Gao and H.~Huang.
\newblock Periodic stochastic gradient descent with momentum for decentralized
  training.
\newblock \emph{arXiv preprint arXiv:2008.10435}, 2020.

\bibitem[Gao and Huang(2021)]{gao2021fast}
H.~Gao and H.~Huang.
\newblock Fast training method for stochastic compositional optimization
  problems.
\newblock \emph{Advances in Neural Information Processing Systems}, 34, 2021.

\bibitem[Gao et~al.(2023)Gao, Gu, and Thai]{gao2023convergenceaistats}
H.~Gao, B.~Gu, and M.~T. Thai.
\newblock On the convergence of distributed stochastic bilevel optimization
  algorithms over a network.
\newblock In \emph{International Conference on Artificial Intelligence and
  Statistics}, pages 9238--9281. PMLR, 2023.

\bibitem[Ghadimi et~al.(2020)Ghadimi, Ruszczynski, and Wang]{ghadimi2020single}
S.~Ghadimi, A.~Ruszczynski, and M.~Wang.
\newblock A single timescale stochastic approximation method for nested
  stochastic optimization.
\newblock \emph{SIAM Journal on Optimization}, 30\penalty0 (1):\penalty0
  960--979, 2020.

\bibitem[Hua et~al.(2022)Hua, Miller, Bertozzi, Qian, and
  Wang]{hua2022efficient}
Y.~Hua, K.~Miller, A.~L. Bertozzi, C.~Qian, and B.~Wang.
\newblock Efficient and reliable overlay networks for decentralized federated
  learning.
\newblock \emph{SIAM Journal on Applied Mathematics}, 82\penalty0 (4):\penalty0
  1558--1586, 2022.

\bibitem[Jiang et~al.(2022)Jiang, Wang, Wang, Zhang, and
  Yang]{jiang2022optimal}
W.~Jiang, B.~Wang, Y.~Wang, L.~Zhang, and T.~Yang.
\newblock Optimal algorithms for stochastic multi-level compositional
  optimization.
\newblock \emph{arXiv preprint arXiv:2202.07530}, 2022.

\bibitem[Koloskova et~al.(2019{\natexlab{a}})Koloskova, Lin, Stich, and
  Jaggi]{koloskova2019decentralized}
A.~Koloskova, T.~Lin, S.~U. Stich, and M.~Jaggi.
\newblock Decentralized deep learning with arbitrary communication compression.
\newblock \emph{arXiv preprint arXiv:1907.09356}, 2019{\natexlab{a}}.

\bibitem[Koloskova et~al.(2019{\natexlab{b}})Koloskova, Stich, and
  Jaggi]{koloskova2019decentralized2}
A.~Koloskova, S.~Stich, and M.~Jaggi.
\newblock Decentralized stochastic optimization and gossip algorithms with
  compressed communication.
\newblock In \emph{International Conference on Machine Learning}, pages
  3478--3487. PMLR, 2019{\natexlab{b}}.

\bibitem[Lian and Liu(2018)]{lian2018revisit}
X.~Lian and J.~Liu.
\newblock Revisit batch normalization: New understanding from an optimization
  view and a refinement via composition optimization.
\newblock \emph{arXiv preprint arXiv:1810.06177}, 2018.

\bibitem[Lian et~al.(2017)Lian, Zhang, Zhang, Hsieh, Zhang, and
  Liu]{lian2017can}
X.~Lian, C.~Zhang, H.~Zhang, C.-J. Hsieh, W.~Zhang, and J.~Liu.
\newblock Can decentralized algorithms outperform centralized algorithms? a
  case study for decentralized parallel stochastic gradient descent.
\newblock \emph{arXiv preprint arXiv:1705.09056}, 2017.

\bibitem[Lu et~al.(2022)Lu, Zeng, Cui, Squillante, Horesh, Kingsbury, Liu, and
  Hong]{lu2022stochastic}
S.~Lu, S.~Zeng, X.~Cui, M.~Squillante, L.~Horesh, B.~Kingsbury, J.~Liu, and
  M.~Hong.
\newblock A stochastic linearized augmented lagrangian method for decentralized
  bilevel optimization.
\newblock \emph{Advances in Neural Information Processing Systems},
  35:\penalty0 30638--30650, 2022.

\bibitem[Nguyen et~al.(2017)Nguyen, Liu, Scheinberg, and
  Tak{\'a}{\v{c}}]{nguyen2017sarah}
L.~M. Nguyen, J.~Liu, K.~Scheinberg, and M.~Tak{\'a}{\v{c}}.
\newblock Sarah: A novel method for machine learning problems using stochastic
  recursive gradient.
\newblock In \emph{International Conference on Machine Learning}, pages
  2613--2621. PMLR, 2017.

\bibitem[Song et~al.(2022)Song, Li, Jin, Shi, Yan, Yin, and
  Yuan]{song2022communication}
Z.~Song, W.~Li, K.~Jin, L.~Shi, M.~Yan, W.~Yin, and K.~Yuan.
\newblock Communication-efficient topologies for decentralized learning with $
  o (1) $ consensus rate.
\newblock \emph{arXiv preprint arXiv:2210.07881}, 2022.

\bibitem[Sun et~al.(2020)Sun, Lu, and Hong]{sun2020improving}
H.~Sun, S.~Lu, and M.~Hong.
\newblock Improving the sample and communication complexity for decentralized
  non-convex optimization: Joint gradient estimation and tracking.
\newblock In \emph{International Conference on Machine Learning}, pages
  9217--9228. PMLR, 2020.

\bibitem[Wang et~al.(2017)Wang, Fang, and Liu]{wang2017stochastic}
M.~Wang, E.~X. Fang, and H.~Liu.
\newblock Stochastic compositional gradient descent: algorithms for minimizing
  compositions of expected-value functions.
\newblock \emph{Mathematical Programming}, 161\penalty0 (1-2):\penalty0
  419--449, 2017.

\bibitem[Xin et~al.(2020)Xin, Khan, and Kar]{xin2020near}
R.~Xin, U.~A. Khan, and S.~Kar.
\newblock A near-optimal stochastic gradient method for decentralized
  non-convex finite-sum optimization.
\newblock \emph{arXiv preprint arXiv:2008.07428}, 2020.

\bibitem[Yang et~al.(2019)Yang, Wang, and Fang]{yang2019multilevel}
S.~Yang, M.~Wang, and E.~X. Fang.
\newblock Multilevel stochastic gradient methods for nested composition
  optimization.
\newblock \emph{SIAM Journal on Optimization}, 29\penalty0 (1):\penalty0
  616--659, 2019.

\bibitem[Ying et~al.(2021)Ying, Yuan, Chen, Hu, Pan, and
  Yin]{ying2021exponential}
B.~Ying, K.~Yuan, Y.~Chen, H.~Hu, P.~Pan, and W.~Yin.
\newblock Exponential graph is provably efficient for decentralized deep
  training.
\newblock \emph{Advances in Neural Information Processing Systems},
  34:\penalty0 13975--13987, 2021.

\bibitem[Yu et~al.(2022)Yu, Wang, Wang, Liu, Yang, and Ji]{yu2022graphfm}
H.~Yu, L.~Wang, B.~Wang, M.~Liu, T.~Yang, and S.~Ji.
\newblock Graphfm: Improving large-scale gnn training via feature momentum.
\newblock In \emph{International Conference on Machine Learning}, pages
  25684--25701. PMLR, 2022.

\bibitem[Yuan and Hu(2020)]{yuan2020stochastic}
H.~Yuan and W.~Hu.
\newblock Stochastic recursive momentum method for non-convex compositional
  optimization.
\newblock \emph{arXiv preprint arXiv:2006.01688}, 2020.

\bibitem[Yuan et~al.(2019)Yuan, Lian, and Liu]{yuan2019stochastic}
H.~Yuan, X.~Lian, and J.~Liu.
\newblock Stochastic recursive variance reduction for efficient smooth
  non-convex compositional optimization.
\newblock \emph{arXiv preprint arXiv:1912.13515}, 2019.

\bibitem[Zhang and Xiao(2019{\natexlab{a}})]{zhang2019composite}
J.~Zhang and L.~Xiao.
\newblock A composite randomized incremental gradient method.
\newblock In \emph{International Conference on Machine Learning}, pages
  7454--7462, 2019{\natexlab{a}}.

\bibitem[Zhang and Xiao(2019{\natexlab{b}})]{zhang2019stochastic}
J.~Zhang and L.~Xiao.
\newblock A stochastic composite gradient method with incremental variance
  reduction.
\newblock In \emph{Advances in Neural Information Processing Systems}, pages
  9078--9088, 2019{\natexlab{b}}.

\bibitem[Zhang and Xiao(2021)]{zhang2021multilevel}
J.~Zhang and L.~Xiao.
\newblock Multilevel composite stochastic optimization via nested variance
  reduction.
\newblock \emph{SIAM Journal on Optimization}, 31\penalty0 (2):\penalty0
  1131--1157, 2021.

\bibitem[Zhang et~al.(2023)Zhang, Thai, Wu, and Gao]{zhang2023communication}
Y.~Zhang, M.~T. Thai, J.~Wu, and H.~Gao.
\newblock On the communication complexity of decentralized bilevel
  optimization.
\newblock \emph{arXiv preprint arXiv:2311.11342}, 2023.

\bibitem[Zhao and Liu(2022)]{zhao2022distributed}
S.~Zhao and Y.~Liu.
\newblock Distributed stochastic compositional optimization problems over
  directed networks.
\newblock \emph{arXiv preprint arXiv:2203.11074}, 2022.

\end{thebibliography}

\appendix
\onecolumn

\aistatstitle{Supplementary Materials}

\section{Appendix}
\subsection{Terminologies}
Before presenting the detailed proof, we first introduce some terminologies as below. First, we denote the function up to the $k$-th level as below:
\begin{equation}
	\begin{aligned}
		& F_n^{(k)}(x) =  f_n^{(1)} (x)  f_n^{(2)} (F_n^{(1)}(x)) \cdots   f_n^{(k-1)} (F_n^{(k-2)}(x)) f^{(k)}(F^{(k-1)}(x))  \ ,\\
	\end{aligned}
\end{equation}
where $f_n^{(k)}(\cdot)   = \mathbb{E}[ f_n^{(k)}(\cdot;  \xi^{(k)}) ]$ and $\xi^{(k)}$ denotes the random sample.
It is easy to know $ F_n(x) = F_n^{(K)}(x)=  f_n^{(1)} (x)  f_n^{(2)} (F_n^{(1)}(x)) \cdots   f_n^{(K-1)} (F_n^{(K-2)}(x)) f_n^{(K)}(F_n^{(K-1)}(x))$. Then, the gradient  of $\nabla F_n^{(k)}(x)$ can be represented as below:
\begin{equation}
	\begin{aligned}
		& \nabla F_n^{(k)}(x) =  \nabla f_n^{(1)} (x)\nabla  f_n^{(2)} (F_n^{(1)}(x)) \cdots \nabla  f_n^{(k-1)} (F_n^{(k-2)}(x)) \nabla f_n^{(k)}(F_n^{(k-1)}(x))  \ , \\
	\end{aligned}
\end{equation}
where $\nabla f_n^{(k)}(\cdot)  = \mathbb{E}[\nabla f_n^{(k)}(\cdot;  \xi^{(k)}) ]$.

Throughout the proof, we assume $\prod_{i}^{k} a_i=1$ when $i>k$. Additionally, we denote 
$X_{t}=[x_{1, t},  \cdots, x_{N, t}]$, $Y_{t}=[y_{1, t}, \cdots, y_{N, t}]$, $M_{t}=[m_{1,t}, \cdots, m_{N,t}]$,  $G_{t}=[g_{1,t}, \cdots, g_{N,t}]$, $\bar{X}_t=[\frac{1}{N}\sum_{n=1}^{N}x_{n, t}, \cdots, \frac{1}{N}\sum_{n=1}^{N}x_{n, t}]$, $\bar{Y}_t=[\frac{1}{N}\sum_{n=1}^{N}y_{n, t}, \cdots, \frac{1}{N}\sum_{n=1}^{N}y_{n, t}]$, $\bar{M}_t=[\frac{1}{N}\sum_{n=1}^{N}m_{n, t}, \cdots, \frac{1}{N}\sum_{n=1}^{N}m_{n, t}]$.

\subsection{Proof of Theorem~\ref{theorem1}}
\begin{lemma} \label{lemma_u_F_momentum}
	For $k\in \{1, \cdots, K-1\}$, given Assumptions~\ref{assumption_smooth}-\ref{assumption_bound_variance}, we can get 
	\begin{equation}
		\begin{aligned}
			& \quad \|u_{n,t}^{(k)} - F_{n}^{(k)}(x_{n,t}) \| \leq \sum_{j=1}^{k} \Big(\prod_{i=j+1}^{k}C_{i}\Big)\|u_{n,t}^{(j)} -f_{n}^{(j)}(u_{n,t}^{(j-1)}) \| \ . 
		\end{aligned}
	\end{equation} 
\end{lemma}

\begin{proof}
	When $k=1$, we have $ \|u_{n,t}^{(1)} - F_{n}^{(1)}(x_{n,t}) \|   = \|u_{n,t}^{(1)}  - f_{n}^{(1)}(u_{n,t}^{(0)})\|$.
	Assume for $k>1$, we have
	\begin{equation}
		\begin{aligned}
			& \quad \|u_{n,t}^{(k)} - F_{n}^{(k)}(x_{n,t}) \| \leq \sum_{j=1}^{k} \Big(\prod_{i=j+1}^{k}C_{i}\Big)\|u_{n,t}^{(j)} -f_{n}^{(j)}(u_{n,t}^{(j-1)}) \| \  .  \\
		\end{aligned}
	\end{equation} 
	
	Then, for $k+1$, we have
	\begin{equation}
		\begin{aligned}
			& \quad \| u_{n,t}^{(k+1)} - F_{n}^{(k+1)}(x_{n,t}) \| \\
			& = \| u_{n,t}^{(k+1)} - f_{n}^{(k+1)}(F_{n}^{(k)}(x_{n,t})) \| \\
			& \leq \| u_{n,t}^{(k+1)} - f_{n}^{(k+1)}( u_{n,t}^{(k)}) \|+ \|f_{n}^{(k+1)}( u_{n,t}^{(k)}) - f_{n}^{(k+1)}(F_{n}^{(k)}(x_{n,t})) \| \\
			& \leq \| u_{n,t}^{(k+1)} - f_{n}^{(k+1)}( u_{n,t}^{(k)}) \|+ C_{k+1}\| u_{n,t}^{(k)} - F_{n}^{(k)}(x_{n,t}) \| \\
			& \leq \| u_{n,t}^{(k+1)} - f_{n}^{(k+1)}( u_{n,t}^{(k)}) \|+ C_{k+1}\sum_{j=1}^{k} \Big(\prod_{i=j+1}^{k}C_{i}\Big)\| u_{n,t}^{(j)} -f_{n}^{(j)}( u_{n,t}^{(j-1)}) \| \\
			&  = \sum_{j=1}^{k+1} \Big(\prod_{i=j+1}^{k+1}C_{i}\Big)\| u_{n,t}^{(j)} -f_{n}^{(j)}( u_{n,t}^{(j-1)}) \| \ , 
		\end{aligned}
	\end{equation} 
which completes the proof. 
	
\end{proof}

\newpage
\begin{lemma} \label{lemma_u_inc_momentum}
		For $k\in \{2, \cdots, K\}$, given Assumptions~\ref{assumption_smooth}-\ref{assumption_bound_variance},  we can get 
	\begin{equation}
		\begin{aligned}
			&  \mathbb{E}[\|u_{n,t}^{(k-1)}  - u_{n, t-1}^{(k-1)} \|^2 ] \leq  \Big(\prod_{j=1}^{k-1}(2C_{j}^2)\Big)\mathbb{E}[\|u_{n, t-1}^{(0)} - u_{n,t}^{(0)} \|^2] + 2\beta^2  \eta^2\sum_{j=1}^{k-1} \Big(\prod_{i=j+1}^{k-1}(2C_{i}^2)\Big)\mathbb{E}[\|u_{n, t-1}^{(j)}- f_{n}^{(j)}(u_{n, t-1}^{(j-1)})  \|^2 ] \\
			& \quad \quad \quad \quad\quad\quad \quad \quad \quad\quad + 2\beta^2  \eta^2\sum_{j=1}^{k-1} \Big(\prod_{i=j+1}^{k-1}(2C_{i}^2)\Big)\delta_{j}^2 \ ,  
		\end{aligned}
	\end{equation}
	and
	\begin{equation}
	\begin{aligned}
		&  \mathbb{E}[\|u_{n,t}^{(k-1)}  - u_{n, t-1}^{(k-1)} \|^2] \leq  2C_{k-1}^2 \mathbb{E}[\|u_{n, t-1}^{(k-2)} - u_{n,t}^{(k-2)} \|^2  + 2 \beta^2  \eta^2\mathbb{E}[ \|u_{n, t-1}^{(k-1)}- f_{n}^{(k-1)}(u_{n, t-1}^{(k-2)})  \|^2] + 2 \beta^2  \eta^2 \delta_{k-1}^2 \ .\\
	\end{aligned}
\end{equation}

\end{lemma}

\begin{proof}
	
For any $k>1$, we can get
	\begin{equation}
	\begin{aligned}
		& \quad \mathbb{E}[\|u_{n,t}^{(k-1)}  - u_{n, t-1}^{(k-1)} \|^2]  \\
		& = \mathbb{E}[\|(1-\beta \eta)(u_{n, t-1}^{(k-1)} -f_{n}^{(k-1)}(u_{n, t-1}^{(k-2)}; \xi_{n, t}^{(k-1)})   )+  f_{n}^{(k-1)}(u_{n,t}^{(k-2)}; \xi_{n, t}^{(k-1)})   - u_{n, t-1}^{(k-1)} \|^2]  \\
		&  =\mathbb{E}[\|-\beta \eta(u_{n, t-1}^{(k-1)}- f_{n}^{(k-1)}(u_{n, t-1}^{(k-2)})   + f_{n}^{(k-1)}(u_{n, t-1}^{(k-2)})   -f_{n}^{(k-1)}(u_{n, t-1}^{(k-2)}; \xi_{n, t}^{(k-1)})   )  \\
		& \quad -f_{n}^{(k-1)}(u_{n, t-1}^{(k-2)}; \xi_{n, t}^{(k-1)})   +  f_{n}^{(k-1)}(u_{n,t}^{(k-2)}; \xi_{n, t}^{(k-1)})   \|^2  \\
		& \leq 2\mathbb{E}[\|-\beta \eta(u_{n, t-1}^{(k-1)}- f_{n}^{(k-1)}(u_{n, t-1}^{(k-2)})   + f_{n}^{(k-1)}(u_{n, t-1}^{(k-2)})   -f_{n}^{(k-1)}(u_{n, t-1}^{(k-2)}; \xi_{n, t}^{(k-1)})   )  \|^2]\\
		& \quad + 2\mathbb{E}[\|-f_{n}^{(k-1)}(u_{n, t-1}^{(k-2)}; \xi_{n, t}^{(k-1)})   +  f_{n}^{(k-1)}(u_{n,t}^{(k-2)}; \xi_{n, t}^{(k-1)})   \|^2]  \\
		& \leq  2C_{k-1}^2 \mathbb{E}[\|u_{n, t-1}^{(k-2)} - u_{n,t}^{(k-2)} \|^2  + 2 \beta^2  \eta^2\mathbb{E}[ \|u_{n, t-1}^{(k-1)}- f_{n}^{(k-1)}(u_{n, t-1}^{(k-2)})  \|^2] + 2 \beta^2  \eta^2 \delta_{k-1}^2 \ ,\\
	\end{aligned}
\end{equation}
where the last step holds due to Assumption~\ref{assumption_bound_gradient} and Assumption~\ref{assumption_bound_variance}, and $\mathbb{E}[f_{n}^{(k-1)}(u_{n, t-1}^{(k-2)}; \xi_{n, t}^{(k-1)})  ] =f_{n}^{(k-1)}(u_{n, t-1}^{(k-2)})$.

Then, by recursively expanding this inequality, we can get
\begin{equation}
	\begin{aligned}
		&  \mathbb{E}[\|u_{n,t}^{(k-1)}  - u_{n, t-1}^{(k-1)} \|^2 ] \leq  \Big(\prod_{j=1}^{k-1}(2C_{j}^2)\Big)\mathbb{E}[\|u_{n, t-1}^{(0)} - u_{n,t}^{(0)} \|^2] \\
		& \quad + 2\beta^2  \eta^2\sum_{j=1}^{k-1} \Big(\prod_{i=j+1}^{k-1}(2C_{i}^2)\Big)\mathbb{E}[\|u_{n, t-1}^{(j)}- f_{n}^{(j)}(u_{n, t-1}^{(j-1)})  \|^2 ]+ 2\beta^2  \eta^2\sum_{j=1}^{k-1} \Big(\prod_{i=j+1}^{k-1}(2C_{i}^2)\Big)\delta_{j}^2 \ . 
	\end{aligned}
\end{equation}

\end{proof}

\begin{lemma} \label{lemma_f_u_f_x_momentum}
	Given Assumptions~\ref{assumption_smooth}-\ref{assumption_bound_variance}, we can get
	\begin{equation}
		\begin{aligned}
			&\quad   \frac{1}{N}\sum_{n=1}^{N}\|\nabla f_{n}^{(1)} ( u_{n,t}^{(0)})\nabla  f_{n}^{(2)} ( u_{n,t}^{(1)}) \cdots \nabla  f_{n}^{(K-1)} ( u_{n,t}^{(K-2)})\nabla f_{n}^{(K)}( u_{n,t}^{(K-1)})  \\
			& \quad  - \nabla f_{n}^{(1)} (x_{n,t})\nabla  f_{n}^{(2)} (F_{n}^{(1)}(x_{n,t})) \cdots \nabla  f_{n}^{(K-1)} (F_{n}^{(K-2)}(x_{n,t}))\nabla f_{n}^{(K)}(F_{n}^{(K-1)}(x_{n,t})) \|^2 \\
			& \leq \frac{K}{N}\sum_{n=1}^{N}\sum_{k=1}^{K-1}A_k\| u_{n,t}^{(k)} -f_{n}^{(k)}( u_{n,t}^{(k-1)}) \|^2  \ ,  \\
		\end{aligned}
	\end{equation}
	
	where $A_k= \Bigg(\sum_{j=k}^{K-1}\Big(\frac{L_{j+1}\prod_{i=1}^{K}C_i}{C_{j+1}} \prod_{i=k+1}^{j}C_{i}\Big)\Bigg)^2 $ .

\end{lemma}

\begin{proof}
	Because $u_{n,t}^{(0)}=x_{n,t}$, we can get
	\begin{equation}
		\begin{aligned}
			&\quad  \|\nabla f_{n}^{(1)} ( u_{n,t}^{(0)})\nabla  f_{n}^{(2)} ( u_{n,t}^{(1)}) \cdots \nabla  f_{n}^{(K-1)} ( u_{n,t}^{(K-2)})\nabla f_{n}^{(K)}( u_{n,t}^{(K-1)})  \\
			& \quad  - \nabla f_{n}^{(1)} (x_{n,t})\nabla  f_{n}^{(2)} (F_{n}^{(1)}(x_{n,t})) \cdots \nabla  f_{n}^{(K-1)} (F_{n}^{(K-2)}(x_{n,t}))\nabla f_{n}^{(K)}(F_{n}^{(K-1)}(x_{n,t})) \| \\
			& = \|\nabla f_{n}^{(1)} (x_{n,t})\nabla  f_{n}^{(2)} ( u_{n,t}^{(1)}) \cdots \nabla  f_{n}^{(K-1)} ( u_{n,t}^{(K-2)})\nabla f_{n}^{(K)}( u_{n,t}^{(K-1)})  \\
			& \quad  - \nabla f_{n}^{(1)} (x_{n,t})\nabla  f_{n}^{(2)} (F_{n}^{(1)}(x_{n,t})) \nabla  f_{n}^{(3)} ( u_{n,t}^{(2)}) \cdots \nabla  f_{n}^{(K-1)} ( u_{n,t}^{(K-2)})\nabla f_{n}^{(K)}( u_{n,t}^{(K-1)})  \\
			& \quad  + \nabla f_{n}^{(1)} (x_{n,t})\nabla  f_{n}^{(2)} (F_{n}^{(1)}(x_{n,t})) \nabla  f_{n}^{(3)} ( u_{n,t}^{(2)}) \cdots \nabla  f_{n}^{(K-1)} ( u_{n,t}^{(K-2)})\nabla f_{n}^{(K)}( u_{n,t}^{(K-1)})  \\
			& \quad  - \nabla f_{n}^{(1)} (x_{n,t})\nabla  f_{n}^{(2)} (F_{n}^{(1)}(x_{n,t})) \nabla  f_{n}^{(3)} (F_{n}^{(2)}(x_{n,t})) \cdots \nabla  f_{n}^{(K-1)} ( u_{n,t}^{(K-2)})\nabla f_{n}^{(K)}( u_{n,t}^{(K-1)})  \\
			& \quad +\cdots \\
			& \quad + \nabla f_{n}^{(1)} (x_{n,t})\nabla  f_{n}^{(2)} (F_{n}^{(1)}(x_{n,t})) \cdots \nabla  f_{n}^{(K-1)} (F_{n}^{(K-2)}(x_{n,t}))\nabla f_{n}^{(K)}( u_{n,t}^{(K-1)})  \\
			& \quad  - \nabla f_{n}^{(1)} (x_{n,t})\nabla  f_{n}^{(2)} (F_{n}^{(1)}(x_{n,t})) \cdots \nabla  f_{n}^{(K-1)} (F_{n}^{(K-2)}(x_{n,t}))\nabla f_{n}^{(K)}(F_{n}^{(K-1)}(x_{n,t})) \| \\
			& \leq \|\nabla f_{n}^{(1)} (x_{n,t})\nabla  f_{n}^{(2)} ( u_{n,t}^{(1)}) \cdots \nabla  f_{n}^{(K-1)} ( u_{n,t}^{(K-2)})\nabla f_{n}^{(K)}( u_{n,t}^{(K-1)})  \\
			& \quad \quad  - \nabla f_{n}^{(1)} (x_{n,t})\nabla  f_{n}^{(2)} (F_{n}^{(1)}(x_{n,t})) \nabla  f_{n}^{(3)} ( u_{n,t}^{(2)}) \cdots \nabla  f_{n}^{(K-1)} ( u_{n,t}^{(K-2)})\nabla f_{n}^{(K)}( u_{n,t}^{(K-1)}) \| \\
			& \quad  +\|\nabla f_{n}^{(1)} (x_{n,t})\nabla  f_{n}^{(2)} (F_{n}^{(1)}(x_{n,t})) \nabla  f_{n}^{(3)} ( u_{n,t}^{(2)}) \cdots \nabla  f_{n}^{(K-1)} ( u_{n,t}^{(K-2)})\nabla f_{n}^{(K)}( u_{n,t}^{(K-1)})  \\
			& \quad \quad  - \nabla f_{n}^{(1)} (x_{n,t})\nabla  f_{n}^{(2)} (F_{n}^{(1)}(x_{n,t})) \nabla  f_{n}^{(3)} (F_{n}^{(2)}(x_{n,t})) \cdots \nabla  f_{n}^{(K-1)} ( u_{n,t}^{(K-2)})\nabla f_{n}^{(K)}( u_{n,t}^{(K-1)})  \|\\
			& \quad +\cdots \\
			& \quad + \|\nabla f_{n}^{(1)} (x_{n,t})\nabla  f_{n}^{(2)} (F_{n}^{(1)}(x_{n,t})) \cdots \nabla  f_{n}^{(K-1)} (F_{n}^{(K-2)}(x_{n,t}))\nabla f_{n}^{(K)}( u_{n,t}^{(K-1)})  \\
			& \quad  \quad - \nabla f_{n}^{(1)} (x_{n,t})\nabla  f_{n}^{(2)} (F_{n}^{(1)}(x_{n,t})) \cdots \nabla  f_{n}^{(K-1)} (F_{n}^{(K-2)}(x_{n,t}))\nabla f_{n}^{(K)}(F_{n}^{(K-1)}(x_{n,t}) )\| \\
			& \leq  \frac{L_2\prod_{k=1}^{K}C_k}{C_2}\| u_{n,t}^{(1)}- F_{n}^{(1)}(x_{n,t})  \| +  \frac{L_3\prod_{k=1}^{K}C_k}{C_3}\| u_{n,t}^{(2)}  - F_{n}^{(2)}(x_{n,t}) \| \\
			& \quad + \cdots + \frac{L_K\prod_{k=1}^{K}C_k}{C_K}\| u_{n,t}^{(K-1)}  - F_{n}^{(K-1)}(x_{n,t}) \| \\
			& =\sum_{k=1}^{K-1}  \frac{L_{k+1}\prod_{j=1}^{K}C_j}{C_{k+1}}\| u_{n,t}^{(k)}- F_{n}^{(k)}(x_{n,t})  \| \\
			& \leq \sum_{k=1}^{K-1}  \frac{L_{k+1}\prod_{j=1}^{K}C_j}{C_{k+1}} \sum_{j=1}^{k} \Big(\prod_{i=j+1}^{k}C_{i}\Big)\|u_{n,t}^{(j)} -f_{n}^{(j)}(u_{n,t}^{(j-1)}) \| \\
			& \leq \sum_{k=1}^{K-1}   \Bigg(\sum_{j=k}^{K-1}\Big(\frac{L_{j+1}\prod_{i=1}^{K}C_i}{C_{j+1}} \prod_{i=k+1}^{j}C_{i}\Big)\Bigg)\|u_{n,t}^{(k)} -f_{n}^{(k)}(u_{n,t}^{(k-1)}) \|  \ , \\
		\end{aligned}
	\end{equation}
where the second to last step holds due to Lemma~\ref{lemma_u_F_momentum}. Then, we complete the proof by taking the squared operation on both sides.

\end{proof}

\begin{lemma} \label{lemma_f_u_inc_momentum}
		Given Assumptions~\ref{assumption_smooth}-\ref{assumption_bound_variance} and $D_{k}=\frac{(\prod_{j=1}^{K} C_{j}^2)L_{k+1}^2}{C_{k+1}^2}$ where $k\in\{0, \cdots, K-1\}$, we can get
	\begin{equation}
		\begin{aligned}
			& \quad \mathbb{E}\Big[\Big\|  \nabla f_{n}^{(1)} ( u_{n,t-1}^{(0)})\nabla  f_{n}^{(2)} ( u_{n,t-1}^{(1)}) \cdots \nabla  f_{n}^{(K-1)} ( u_{n,t-1}^{(K-2)})\nabla f_{n}^{(K)}( u_{n,t-1}^{(K-1)})\\
			& \quad  \quad   - \nabla f_{n}^{(1)} ( u_{n,t}^{(0)})\nabla  f_{n}^{(2)} ( u_{n,t}^{(1)}) \cdots \nabla  f_{n}^{(K-1)} ( u_{n,t}^{(K-2)})\nabla f_{n}^{(K)}( u_{n,t}^{(K-1)}) \Big\|^2\Big]  \\
			& \leq KD_{0}\mathbb{E}[ \|x_{n,t}-x_{n,t-1}\|^2 ] +  2K\sum_{k=1}^{K-1}D_{k}  C_{k}^2 \mathbb{E}[ \|u_{n, t-1}^{(k-1)} - u_{n,t}^{(k-1)} \|^2] \\
			& \quad  + 2 \beta^2  \eta^2 K\sum_{k=1}^{K-1}D_{k}  \mathbb{E}[ \|u_{n, t-1}^{(k)}- f_{n}^{(k)}(u_{n, t-1}^{(k-1)})  \|^2 ]+ 2 \beta^2  \eta^2 K\sum_{k=1}^{K-1}D_{k}  \delta_{k}^2 \ . \\
		\end{aligned}
	\end{equation}

\end{lemma}

\begin{proof}
	\begin{equation}
		\begin{aligned}
			& \quad \mathbb{E}\Big[\Big\|  \nabla f_{n}^{(1)} ( u_{n,t-1}^{(0)})\nabla  f_{n}^{(2)} ( u_{n,t-1}^{(1)}) \cdots \nabla  f_{n}^{(K-1)} ( u_{n,t-1}^{(K-2)})\nabla f_{n}^{(K)}( u_{n,t-1}^{(K-1)})\\
			& \quad  \quad   - \nabla f_{n}^{(1)} ( u_{n,t}^{(0)})\nabla  f_{n}^{(2)} ( u_{n,t}^{(1)}) \cdots \nabla  f_{n}^{(K-1)} ( u_{n,t}^{(K-2)})\nabla f_{n}^{(K)}( u_{n,t}^{(K-1)}) \Big\|^2\Big]  \\
			& = \mathbb{E}\Big[\Big\|  \nabla f_{n}^{(1)} ( u_{n,t-1}^{(0)})\nabla  f_{n}^{(2)} ( u_{n,t-1}^{(1)}) \cdots \nabla  f_{n}^{(K-1)} ( u_{n,t-1}^{(K-2)})\nabla f_{n}^{(K)}( u_{n,t-1}^{(K-1)})\\
			& \quad \quad  - \nabla f_{n}^{(1)} ( u_{n,t}^{(0)})\nabla  f_{n}^{(2)} ( u_{n,t-1}^{(1)}) \cdots \nabla  f_{n}^{(K-1)} ( u_{n,t-1}^{(K-2)})\nabla f_{n}^{(K)}( u_{n,t-1}^{(K-1)}) \\
			& \quad \quad  +  \nabla f_{n}^{(1)} ( u_{n,t}^{(0)})\nabla  f_{n}^{(2)} ( u_{n,t-1}^{(1)}) \cdots \nabla  f_{n}^{(K-1)} ( u_{n,t-1}^{(K-2)})\nabla f_{n}^{(K)}( u_{n,t-1}^{(K-1)}) \\
			& \quad \quad  - \nabla f_{n}^{(1)} ( u_{n,t}^{(0)})\nabla  f_{n}^{(2)} ( u_{n,t}^{(1)}) \cdots \nabla  f_{n}^{(K-1)} ( u_{n,t-1}^{(K-2)})\nabla f_{n}^{(K)}( u_{n,t-1}^{(K-1)}) \\
			& \quad \quad   \cdots \\
			& \quad \quad + \nabla f_{n}^{(1)} ( u_{n,t}^{(0)})\nabla  f_{n}^{(2)} ( u_{n,t}^{(1)}) \cdots \nabla  f_{n}^{(K-1)} ( u_{n,t}^{(K-2)})\nabla f_{n}^{(K)}( u_{n,t-1}^{(K-1)}) \\
			& \quad \quad  - \nabla f_{n}^{(1)} ( u_{n,t}^{(0)})\nabla  f_{n}^{(2)} ( u_{n,t}^{(1)}) \cdots \nabla  f_{n}^{(K-1)} ( u_{n,t}^{(K-2)})\nabla f_{n}^{(K)}( u_{n,t}^{(K-1)}) \Big\|^2\Big]  \\
			& \leq  K\mathbb{E}\Big[\Big\|  \nabla f_{n}^{(1)} ( u_{n,t-1}^{(0)})\nabla  f_{n}^{(2)} ( u_{n,t-1}^{(1)}) \cdots \nabla  f_{n}^{(K-1)} ( u_{n,t-1}^{(K-2)})\nabla f_{n}^{(K)}( u_{n,t-1}^{(K-1)}) \\
			& \quad \quad  - \nabla f_{n}^{(1)} ( u_{n,t}^{(0)})\nabla  f_{n}^{(2)} ( u_{n,t-1}^{(1)}) \cdots \nabla  f_{n}^{(K-1)} ( u_{n,t-1}^{(K-2)})\nabla f_{n}^{(K)}( u_{n,t-1}^{(K-1)}) \Big\|^2\Big] \\
			& \quad   + K \mathbb{E}\Big[\Big\|\nabla f_{n}^{(1)} ( u_{n,t}^{(0)})\nabla  f_{n}^{(2)} ( u_{n,t-1}^{(1)}) \cdots \nabla  f_{n}^{(K-1)} ( u_{n,t-1}^{(K-2)})\nabla f_{n}^{(K)}( u_{n,t-1}^{(K-1)}) \\
			& \quad \quad  - \nabla f_{n}^{(1)} ( u_{n,t}^{(0)})\nabla  f_{n}^{(2)} ( u_{n,t}^{(1)}) \cdots \nabla  f_{n}^{(K-1)} ( u_{n,t-1}^{(K-2)})\nabla f_{n}^{(K)}( u_{n,t-1}^{(K-1)})\Big\|^2\Big]  \\
			& \quad   + \cdots \\
			& \quad  + K\mathbb{E}\Big[\Big\|\nabla f_{n}^{(1)} ( u_{n,t}^{(0)})\nabla  f_{n}^{(2)} ( u_{n,t}^{(1)}) \cdots \nabla  f_{n}^{(K-1)} ( u_{n,t}^{(K-2)})\nabla f_{n}^{(K)}( u_{n,t-1}^{(K-1)}) \\
			& \quad \quad  - \nabla f_{n}^{(1)} ( u_{n,t}^{(0)})\nabla  f_{n}^{(2)} ( u_{n,t}^{(1)}) \cdots \nabla  f_{n}^{(K-1)} ( u_{n,t}^{(K-2)})\nabla f_{n}^{(K)}( u_{n,t}^{(K-1)}) \Big\|^2\Big]  \\
			& \leq  K\frac{(\prod_{j=1}^{K} C_{j}^2)L_1^2}{C_1^2}\mathbb{E}[ \|u_{n,t}^{(0)}-u_{n,t-1}^{(0)}\|^2]  +  K\frac{(\prod_{j=1}^{K} C_{j}^2)L_2^2}{C_2^2} \mathbb{E}[ \|u_{n,t}^{(1)}-u_{n,t-1}^{(1)}\|^2] \\
			& \quad  + \cdots + K\frac{(\prod_{j=1}^{K} C_{j}^2)L_K^2}{C_K^2} \mathbb{E}[ \|u_{n,t}^{(K-1)}-u_{n,t-1}^{(K-1)}\|^2]\\
			& \leq KD_{0} \mathbb{E}[ \|x_{n,t}-x_{n,t-1}\|^2]  +  K\sum_{k=1}^{K-1}D_{k} \mathbb{E}[ \|u_{n,t}^{(k)}-u_{n,t-1}^{(k)}\|^2] \\
			& \leq KD_{0}\mathbb{E}[ \|x_{n,t}-x_{n,t-1}\|^2]  +  2K\sum_{k=1}^{K-1}D_{k}  C_{k}^2 \mathbb{E}[ \|u_{n, t-1}^{(k-1)} - u_{n,t}^{(k-1)} \|^2] \\
			& \quad  + 2 \beta^2  \eta^2 K\sum_{k=1}^{K-1}D_{k} \mathbb{E}[  \|u_{n, t-1}^{(k)}- f_{n}^{(k)}(u_{n, t-1}^{(k-1)})  \|^2 ]+ 2 \beta^2  \eta^2 K\sum_{k=1}^{K-1}D_{k}  \delta_{k}^2 \ , \\
		\end{aligned}
	\end{equation}
where $D_{k}=\frac{(\prod_{j=1}^{K} C_{j}^2)L_{k+1}^2}{C_{k+1}^2}$, and the last step holds due to Lemma~\ref{lemma_u_inc_momentum}.

\end{proof}

\begin{lemma} \label{lemma_g_var_momentum}
			Given Assumptions~\ref{assumption_smooth}-\ref{assumption_bound_variance}, we can get
	\begin{equation}
		\begin{aligned}
			& \quad  \mathbb{E}\Big[\Big\|\frac{1}{N} \sum_{n=1}^{N} ( g_{n, t}-  \nabla f_{n}^{(1)} ( u_{n,t}^{(0)})\nabla  f_{n}^{(2)} ( u_{n,t}^{(1)}) \cdots \nabla  f_{n}^{(K-1)} ( u_{n,t}^{(K-2)})\nabla f_{n}^{(K)}( u_{n,t}^{(K-1)})) \Big\|^2\Big] \leq K\sum_{k=1}^{K}B_k \frac{\sigma_{k}^2 }{N} \ , \\
		\end{aligned}
	\end{equation}
	where $B_{k} = \frac{\prod_{j=1}^{K}C_j^2 }{C_k^2}$.
\end{lemma}

\begin{proof}
	
	\begin{equation}
		\begin{aligned}
			& \quad  \mathbb{E}\Big[\Big\|\frac{1}{N} \sum_{n=1}^{N} ( g_{n, t}-  \nabla f_{n}^{(1)} ( u_{n,t}^{(0)})\nabla  f_{n}^{(2)} ( u_{n,t}^{(1)}) \cdots \nabla  f_{n}^{(K-1)} ( u_{n,t}^{(K-2)})\nabla f_{n}^{(K)}( u_{n,t}^{(K-1)})) \Big\|^2\Big] \\
			& = \mathbb{E}\Big[\Big\|\frac{1}{N} \sum_{n=1}^{N} \Big({v}_{n,  t}^{(1)} {v}_{ n, t}^{(2)}\cdots {v}_{n,  t}^{(K-1)} {v}_{n,  t}^{(K)}  - \nabla f_{n}^{(1)} (u_{n,t}^{(0)}) {v}_{n,  t}^{(2)}\cdots {v}_{n,  t}^{(K-1)} {v}_{n,  t}^{(K)} \\
			& \quad  + \nabla f_{n}^{(1)} (u_{n,t}^{(0)}) {v}_{n,  t}^{(2)}\cdots {v}_{n,  t}^{(K-1)} {v}_{ n, t}^{(K)}  -  \nabla f_{n}^{(1)} (u_{n,t}^{(0)}) \nabla  f_{n}^{(2)} (u_{n,t}^{(1)})\cdots {v}_{ n, t}^{(K-1)} {v}_{n,  t}^{(K)} \\
			& \quad + \cdots \\
			& \quad +  \nabla f_{n}^{(1)} (u_{n,t}^{(0)})\nabla  f_{n}^{(2)} (u_{n,t}^{(1)}) \cdots \nabla  f_{n}^{(K-1)} (u_{n,t}^{(K-2)}){v}_{n,  t}^{(K)} \\
			& \quad \quad -  \nabla f_{n}^{(1)} (u_{n,t}^{(0)})\nabla  f_{n}^{(2)} (u_{n,t}^{(1)}) \cdots \nabla  f_{n}^{(K-1)} (u_{n,t}^{(K-2)})\nabla f_{n}^{(K)}(u_{n,t}^{(K-1)})\Big)\Big\|^2 \Big] \\
			& \leq  K\mathbb{E}\Big[\Big\|\frac{1}{N} \sum_{n=1}^{N} \Big({v}_{n,  t}^{(1)} {v}_{ n, t}^{(2)}\cdots {v}_{n,  t}^{(K-1)} {v}_{n,  t}^{(K)}  - \nabla f_{n}^{(1)} (u_{n,t}^{(0)}) {v}_{n,  t}^{(2)}\cdots {v}_{n,  t}^{(K-1)} {v}_{n,  t}^{(K)} \Big)\Big\|^2\Big]  \\
			& \quad  +K\mathbb{E}\Big[\Big\|\frac{1}{N} \sum_{n=1}^{N} \Big( \nabla f_{n}^{(1)} (u_{n,t}^{(0)}) {v}_{n,  t}^{(2)}\cdots {v}_{n,  t}^{(K-1)} {v}_{ n, t}^{(K)}  -  \nabla f_{n}^{(1)} (u_{n,t}^{(0)}) \nabla  f_{n}^{(2)} (u_{n,t}^{(1)})\cdots {v}_{ n, t}^{(K-1)} {v}_{n,  t}^{(K)} \Big)\Big\|^2\Big]  \\
			& \quad + \cdots \\
			& \quad +  K\mathbb{E}\Big[\Big\|\frac{1}{N} \sum_{n=1}^{N} \Big(\nabla f_{n}^{(1)} (u_{n,t}^{(0)})\nabla  f_{n}^{(2)} (u_{n,t}^{(1)}) \cdots \nabla  f_{n}^{(K-1)} (u_{n,t}^{(K-2)}){v}_{n,  t}^{(K)}  \Big)\Big\|^2\Big]  \\
			& \quad \quad -  \nabla f_{n}^{(1)} (u_{n,t}^{(0)})\nabla  f_{n}^{(2)} (u_{n,t}^{(1)}) \cdots \nabla  f_{n}^{(K-1)} (u_{n,t}^{(K-2)})\nabla f_{n}^{(K)}(u_{n,t}^{(K-1)})\Big)\Big\|^2 \Big] \\
			& \leq K\sum_{k=1}^{K}\frac{\prod_{j=1}^{K}C_j^2 }{C_k^2} \frac{\sigma_{k}^2}{N} \ , \\
		\end{aligned}
	\end{equation}
	where the last step holds due to the fact that the sampling procedure on different workers and different levels are independent.  For instance, for the first level, we can get
	\begin{equation}
		\small
		\begin{aligned}
			& \quad \mathbb{E}\Big[\Big\|\frac{1}{N} \sum_{n=1}^{N} \Big({v}_{n,  t}^{(1)} {v}_{ n, t}^{(2)}\cdots {v}_{n,  t}^{(K-1)} {v}_{n,  t}^{(K)}  - \nabla f_{n}^{(1)} (u_{n,t}^{(0)}) {v}_{n,  t}^{(2)}\cdots {v}_{n,  t}^{(K-1)} {v}_{n,  t}^{(K)} \Big)\Big\|^2\Big]  \\
			& =\frac{1}{N^2}  \mathbb{E}\Big[  \sum_{n=1}^{N}\Big\| \Big({v}_{n,  t}^{(1)} -\nabla f_{n}^{(1)} (u_{n,t}^{(0)})\Big){v}_{ n, t}^{(2)}\cdots {v}_{n,  t}^{(K-1)} {v}_{n,  t}^{(K)}  \Big\|^2\Big]\\
			& \quad + \frac{1}{N^2}  \mathbb{E}\Big[ \sum_{n=1}^{N}\sum_{n'=1, n'\neq n}^{N}\Big\langle  \Big({v}_{n,  t}^{(1)} -\nabla f_{n}^{(1)} (u_{n,t}^{(0)})\Big){v}_{ n, t}^{(2)}\cdots {v}_{n,  t}^{(K-1)} {v}_{n,  t}^{(K)} , \Big({v}_{n',  t}^{(1)} -\nabla f_{n'}^{(1)} (u_{n',t}^{(0)})\Big){v}_{ n', t}^{(2)}\cdots {v}_{n',  t}^{(K-1)} {v}_{n',  t}^{(K)}  \Big\rangle\Big] \\
			& =\frac{1}{N^2}  \mathbb{E}\Big[  \sum_{n=1}^{N}\Big\| \Big({v}_{n,  t}^{(1)} -\nabla f_{n}^{(1)} (u_{n,t}^{(0)})\Big){v}_{ n, t}^{(2)}\cdots {v}_{n,  t}^{(K-1)} {v}_{n,  t}^{(K)}  \Big\|^2\Big]\\
			& \leq \frac{\prod_{j=1}^{K}C_j^2 }{C_1^2} \frac{\sigma_{1}^2}{N} \ , \\
		\end{aligned}
	\end{equation}
	where the third step follows from the fact that the sampling procedure on different workers and different levels are independent.

\end{proof}

Similarly, we can prove the following lemma regarding the stochastic gradient on each device. 
\begin{lemma}\label{lemma_g_var_momentum2}
		Given Assumptions~\ref{assumption_smooth}-\ref{assumption_bound_variance}, we can get
	\begin{equation}
		\begin{aligned}
			& \quad  \mathbb{E}\Big[\Big\| g_{n, t}-  \nabla f_{n}^{(1)} ( u_{n,t}^{(0)})\nabla  f_{n}^{(2)} ( u_{n,t}^{(1)}) \cdots \nabla  f_{n}^{(K-1)} ( u_{n,t}^{(K-2)})\nabla f_{n}^{(K)}( u_{n,t}^{(K-1)}) \Big\|^2\Big] \leq  K\sum_{k=1}^{K}B_k \sigma_{k}^2  \ , \\
		\end{aligned}
	\end{equation}
	where $B_{k} = \frac{\prod_{j=1}^{K}C_j^2 }{C_k^2} $.
\end{lemma}

\begin{lemma} \label{lemma_u_var_momentum}
	For $k\in\{1, \cdots, K-1\}$, 	given Assumptions~\ref{assumption_smooth}-\ref{assumption_bound_variance} and $\eta\leq \frac{1}{\beta}$, we can get
	\begin{equation}
		\begin{aligned}
			&  \mathbb{E}[\|u_{n, t}^{(k)} - f_{n}^{(k)}(u_{n, t}^{(k-1)}) \|^2] \leq (1-\beta\eta)\mathbb{E}[\|u_{n, t-1}^{(k)} -f_{n}^{(k)}(u_{n, t-1}^{(k-1)}) \|^2] \\
			& +2 C_k^2\mathbb{E}[\|u_{n, t-1}^{(k-1)}- u_{n,t}^{(k-1)} \|^2 ]+ 2\beta^2\eta^2\delta_{k}^2  \ .
		\end{aligned}
	\end{equation}
\end{lemma}
\begin{proof}

	\begin{equation}
		\begin{aligned}
			& \quad \mathbb{E}[\| u_{n, t}^{(k)} -  f_{n}^{(k)}(u_{n, t}^{(k-1)}) \|^2] \\
			& =  \mathbb{E}[\|(1-\beta\eta) (u_{n, t-1}^{(k)} -f_{n}^{(k)}(u_{n, t-1}^{(k-1)}; \xi_{n, t}^{(k)})  )+  f_{n}^{(k)}(u_{n,t}^{(k-1)}; \xi_{n, t}^{(k)})  -  f_{n}^{(k)}(u_{n, t}^{(k-1)}) \|^2] \\
			& = \mathbb{E}[\|(1-\beta\eta)  (u_{n, t-1}^{(k)} -f_{n}^{(k)}(u_{n, t-1}^{(k-1)}) )\\
			& \quad +  (f_{n}^{(k)}(u_{n, t-1}^{(k-1)}) - f_{n}^{(k)}(u_{n, t}^{(k-1)})  - f_{n}^{(k)}(u_{n, t-1}^{(k-1)}; \xi_{n, t}^{(k)})   + f_{n}^{(k)}(u_{n,t}^{(k-1)}; \xi_{n, t}^{(k)})  )\\
			& \quad + \beta\eta  (f_{n}^{(k)}(u_{n, t-1}^{(k-1)}; \xi_{n, t}^{(k)})  - f_{n}^{(k)}(u_{n, t-1}^{(k-1)}) )\|^2 ]\\
			& = \mathbb{E}[\|(1-\beta\eta)  (u_{n, t-1}^{(k)} -f_{n}^{(k)}(u_{n, t-1}^{(k-1)}) )\|^2 ]\\
			& \quad +  \mathbb{E}[\|(f_{n}^{(k)}(u_{n, t-1}^{(k-1)}) - f_{n}^{(k)}(u_{n, t}^{(k-1)})  - f_{n}^{(k)}(u_{n, t-1}^{(k-1)}; \xi_{n, t}^{(k)})   + f_{n}^{(k)}(u_{n,t}^{(k-1)}; \xi_{n, t}^{(k)})  )\\
			& \quad + \beta\eta  (f_{n}^{(k)}(u_{n, t-1}^{(k-1)}; \xi_{n, t}^{(k)})  - f_{n}^{(k)}(u_{n, t-1}^{(k-1)}) )\|^2 ]\\
			& \leq  (1-\beta\eta)^2\mathbb{E}[\| u_{n, t-1}^{(k)} - f_{n}^{(k)}(u_{n, t-1}^{(k-1)}) \|^2]\\
			& \quad + 2\mathbb{E}[\|f_{n}^{(k)}(u_{n, t-1}^{(k-1)}) - f_{n}^{(k)}(u_{n, t}^{(k-1)})  - f_{n}^{(k)}(u_{n, t-1}^{(k-1)}; \xi_{n, t}^{(k)})   + f_{n}^{(k)}(u_{n,t}^{(k-1)}; \xi_{n, t}^{(k)})   \|^2]\\
			& \quad + 2\beta^2\eta^2\mathbb{E}[\| f_{n}^{(k)}(u_{n, t-1}^{(k-1)}; \xi_{n, t}^{(k)})  - f_{n}^{(k)}(u_{n, t-1}^{(k-1)}) \|^2] \\
			& \leq  (1-\beta\eta)^2\mathbb{E}[\| u_{n, t-1}^{(k)} - f_{n}^{(k)}(u_{n, t-1}^{(k-1)}) \|^2]\\
			& \quad + 2\mathbb{E}[\|f_{n}^{(k)}(u_{n, t-1}^{(k-1)}; \xi_{n, t}^{(k)})   - f_{n}^{(k)}(u_{n,t}^{(k-1)}; \xi_{n, t}^{(k)})   \|^2] + 2\beta^2\eta^2\delta_{k}^2\\
			& \leq  (1-\beta\eta)\mathbb{E}[\| u_{n, t-1}^{(k)} - f_{n}^{(k)}(u_{n, t-1}^{(k-1)}) \|^2] +2 C_k^2\mathbb{E}[\|u_{n, t-1}^{(k-1)}- u_{n,t}^{(k- 1)} \|^2] +2\beta^2\eta^2\delta_{k}^2 \ , \\
		\end{aligned}
	\end{equation}
where the second to last step holds due to Assumption~\ref{assumption_bound_variance}, the last step holds due to Assumption~\ref{assumption_bound_gradient}. 
\end{proof}

\begin{lemma} \label{lemma_m_var_momentum}
			Given Assumptions~\ref{assumption_smooth}-\ref{assumption_bound_variance}, if  $\mu\eta\in (0, 1)$, we can get
\begin{equation}
	\begin{aligned}
		& \quad \mathbb{E}\Big[\Big\|\frac{1}{N}\sum_{n=1}^{N}{m}_{n,t} -\frac{1}{N}\sum_{n=1}^{N}\nabla F_{n}({{x}}_{n,t})\Big\|^2\Big]  \\
		& \leq (1-\mu\eta )\mathbb{E}\Big[\Big\|\frac{1}{N}\sum_{n=1}^{N}{m}_{n,t-1} -\frac{1}{N}\sum_{n=1}^{N}\nabla F_{n}({{x}}_{n,t-1})\Big\|^2\Big] \\
		& \quad + 2\mu\eta  K\frac{1}{N}\sum_{n=1}^{N}\sum_{k=1}^{K-1}A_k\mathbb{E}\Big[\Big\|u_{n,t-1}^{(k)} -f_{n}^{(k)}(u_{n,t-1}^{(k-1)}) \Big\|^2\Big] \\
		& \quad +4\mu\eta  K\frac{1}{N}\sum_{n=1}^{N}\sum_{k=1}^{K-1}A_k C_k^2\|{u}_{ n, t-1}^{(k-1)}- {u}_{n,  t}^{(k-1)} \Big\|^2\Big] \\
		& \quad+ 
		\frac{2 L_F^2}{\mu \eta }\frac{1}{N}\sum_{n=1}^{N}\mathbb{E}\Big[\Big\|{x}_{n,t} - {x}_{n,t-1}\Big\|^2\Big]  +4\mu \beta^2\eta^3 K\sum_{k=1}^{K-1}A_k \delta_{k}^2 + \mu^2\eta^2 K\sum_{k=1}^{K}B_{k}\frac{\sigma_{k}^2}{N} \ . \\
	\end{aligned}
\end{equation}
\end{lemma}

\begin{proof}

	\begin{equation}
		\begin{aligned}
			& \quad \mathbb{E}\Big[\Big\|\frac{1}{N}\sum_{n=1}^{N}{m}_{n,t} -\frac{1}{N}\sum_{n=1}^{N}\nabla F_{n}({{x}}_{n,t})\Big\|^2\Big] \\
			& = \mathbb{E}\Big[\Big\|\frac{1}{N}\sum_{n=1}^{N}\Big((1-\mu\eta )({m}_{n,t-1} -\nabla F_{n}({{x}}_{n,t-1})) + (1-\mu\eta )(\nabla F_{n}({{x}}_{n,t-1})- \nabla F_{n}({{x}}_{n,t})) \\
			& \quad + \mu\eta \Big({g}_{n,t} - \nabla f_{n}^{(1)} ( u_{n,t}^{(0)})\nabla  f_{n}^{(2)} ( u_{n,t}^{(1)}) \cdots \nabla  f_{n}^{(K-1)} ( u_{n,t}^{(K-2)})\nabla f_{n}^{(K)}( u_{n,t}^{(K-1)}) \\
			& \quad \quad + \nabla f_{n}^{(1)} ( u_{n,t}^{(0)})\nabla  f_{n}^{(2)} ( u_{n,t}^{(1)}) \cdots \nabla  f_{n}^{(K-1)} ( u_{n,t}^{(K-2)})\nabla f_{n}^{(K)}( u_{n,t}^{(K-1)})  \\
			& \quad\quad  - \nabla f_{n}^{(1)} (x_{n,t})\nabla  f_{n}^{(2)} (F_{n}^{(1)}(x_{n,t})) \cdots \nabla  f_{n}^{(K-1)} (F_{n}^{(K-2)}(x_{n,t}))\nabla f_{n}^{(K)}(F_{n}^{(K-1)}(x_{n,t})) \Big)\Big)\Big\|^2\Big] \\
			& = \mathbb{E}\Big[\Big\|\frac{1}{N}\sum_{n=1}^{N}\Big((1-\mu\eta )({m}_{n,t-1} -\nabla F_{n}({{x}}_{n,t-1})) + (1-\mu\eta )(\nabla F_{n}({{x}}_{n,t-1})- \nabla F_{n}({{x}}_{n,t})) \\
			& \quad + \mu\eta \Big(\nabla f_{n}^{(1)} ( u_{n,t}^{(0)})\nabla  f_{n}^{(2)} ( u_{n,t}^{(1)}) \cdots \nabla  f_{n}^{(K-1)} ( u_{n,t}^{(K-2)})\nabla f_{n}^{(K)}( u_{n,t}^{(K-1)})  \\
			& \quad\quad  - \nabla f_{n}^{(1)} (x_{n,t})\nabla  f_{n}^{(2)} (F_{n}^{(1)}(x_{n,t})) \cdots \nabla  f_{n}^{(K-1)} (F_{n}^{(K-2)}(x_{n,t}))\nabla f_{n}^{(K)}(F_{n}^{(K-1)}(x_{n,t}))  \Big) \Big)\Big\|^2\Big] \\
			& \quad + \mu^2\eta^2 \mathbb{E}\Big[\Big\|\frac{1}{N}\sum_{n=1}^{N}\Big({g}_{n,t} - \nabla f_{n}^{(1)} ( u_{n,t}^{(0)})\nabla  f_{n}^{(2)} ( u_{n,t}^{(1)}) \cdots \nabla  f_{n}^{(K-1)} ( u_{n,t}^{(K-2)})\nabla f_{n}^{(K)}( u_{n,t}^{(K-1)})\Big)\Big\|^2\Big] \\
			& \leq (1-\mu\eta )^2(1+a)\mathbb{E}\Big[\Big\|\frac{1}{N}\sum_{n=1}^{N}{m}_{n,t-1} -\frac{1}{N}\sum_{n=1}^{N}\nabla F_{n}({{x}}_{n,t-1})\Big\|^2\Big] \\
			& \quad + 2(1+a^{-1})(1-\mu\eta )^2\mathbb{E}\Big[\Big\|\frac{1}{N}\sum_{n=1}^{N}(\nabla F_{n}({{x}}_{n,t-1})- \nabla F_{n}({{x}}_{n,t}))\Big\|^2\Big] \\
			& \quad + 2(1+a^{-1})\mu^2\eta^2 \mathbb{E}\Big[\Big\|\frac{1}{N}\sum_{n=1}^{N}\Big(\nabla f_{n}^{(1)} ( u_{n,t}^{(0)})\nabla  f_{n}^{(2)} ( u_{n,t}^{(1)}) \cdots \nabla  f_{n}^{(K-1)} ( u_{n,t}^{(K-2)})\nabla f_{n}^{(K)}( u_{n,t}^{(K-1)})  \\
			& \quad\quad  - \nabla f_{n}^{(1)} (x_{n,t})\nabla  f_{n}^{(2)} (F_{n}^{(1)}(x_{n,t})) \cdots \nabla  f_{n}^{(K-1)} (F_{n}^{(K-2)}(x_{n,t}))\nabla f_{n}^{(K)}(F_{n}^{(K-1)}(x_{n,t}))  \Big)\Big\|^2\Big] \\
			& \quad + \mu^2\eta^2 \mathbb{E}\Big[\Big\|\frac{1}{N}\sum_{n=1}^{N}\Big({g}_{n,t} - \nabla f_{n}^{(1)} ( u_{n,t}^{(0)})\nabla  f_{n}^{(2)} ( u_{n,t}^{(1)}) \cdots \nabla  f_{n}^{(K-1)} ( u_{n,t}^{(K-2)})\nabla f_{n}^{(K)}( u_{n,t}^{(K-1)})\Big)\Big\|^2\Big] \\
			& \leq (1-\mu\eta )\mathbb{E}\Big[\Big\|\frac{1}{N}\sum_{n=1}^{N}{m}_{n,t-1} -\frac{1}{N}\sum_{n=1}^{N}\nabla F_{n}({{x}}_{n,t-1})\Big\|^2\Big]+ 
			\frac{2 L_F^2}{\mu \eta }\frac{1}{N}\sum_{n=1}^{N}\mathbb{E}\Big[\Big\|{x}_{n,t} - {x}_{n,t-1}\Big\|^2\Big] \\
			& \quad + 2\mu\eta \mathbb{E}\Big[\Big\|\frac{1}{N}\sum_{n=1}^{N}\Big(\nabla f_{n}^{(1)} ( u_{n,t}^{(0)})\nabla  f_{n}^{(2)} ( u_{n,t}^{(1)}) \cdots \nabla  f_{n}^{(K-1)} ( u_{n,t}^{(K-2)})\nabla f_{n}^{(K)}( u_{n,t}^{(K-1)})  \\
			& \quad\quad  - \nabla f_{n}^{(1)} (x_{n,t})\nabla  f_{n}^{(2)} (F_{n}^{(1)}(x_{n,t})) \cdots \nabla  f_{n}^{(K-1)} (F_{n}^{(K-2)}(x_{n,t}))\nabla f_{n}^{(K)}(F_{n}^{(K-1)}(x_{n,t})) \Big) \Big\|^2\Big] \\
			& \quad + \mu^2\eta^2 K\sum_{k=1}^{K}B_{k}\frac{\sigma_{k}^2}{N}\\
			& \leq (1-\mu\eta )\mathbb{E}\Big[\Big\|\frac{1}{N}\sum_{n=1}^{N}{m}_{n,t-1} -\frac{1}{N}\sum_{n=1}^{N}\nabla F_{n}({{x}}_{n,t-1})\Big\|^2\Big]+ 
			\frac{2 L_F^2}{\mu \eta }\frac{1}{N}\sum_{n=1}^{N}\mathbb{E}\Big[\Big\|{x}_{n,t} - {x}_{n,t-1}\Big\|^2\Big] \\
			& \quad + 2\mu\eta  K\frac{1}{N}\sum_{n=1}^{N}\sum_{k=1}^{K-1}A_k \mathbb{E}\Big[\Big\|u_{n,t}^{(k)} -f_{n}^{(k)}(u_{n,t}^{(k-1)}) \Big\|^2\Big]+ \mu^2\eta^2 K\sum_{k=1}^{K}B_{k}\frac{\sigma_{k}^2}{N} \ , \\
		\end{aligned}
	\end{equation}
	where the second to last step holds due to $a=\frac{\mu\eta}{1-\mu\eta}$ and Lemma~\ref{lemma_g_var_momentum}, the last step holds due to Lemma~\ref{lemma_f_u_f_x_momentum}.  Then, by combining it with Lemma~\ref{lemma_u_var_momentum}, we complete the proof.

\end{proof}

\begin{lemma} \label{lemma_x_consensus_momentum}
	Given Assumptions~\ref{assumption_graph}-\ref{assumption_bound_variance}, we can get
	\begin{equation}
		\begin{aligned}
			& \quad  \mathbb{E}[\|X_{t+1} - \bar{X}_{t+1}\|_{F}^2 ]\leq  (1-\eta\frac{1-\lambda^2}{2}) \mathbb{E}[\|X_{t}  - \bar{X}_{t} \|_F^2 ] + \frac{2\eta\alpha^2}{1-\lambda^2}  \mathbb{E}[\| Y_{t} -  \bar{Y}_{t} \|_F^2] \ . \\
		\end{aligned}
	\end{equation}
	
\end{lemma}

\begin{proof}

	\begin{equation}
		\begin{aligned}
			& \quad \mathbb{E}[ \|X_{t+1} - \bar{X}_{t+1}\|_{F}^2] \\
			& =  \mathbb{E}[\|X_{t} + \eta (X_{t+\frac{1}{2}} - X_{t}) - \bar{X}_{t} - \eta (\bar{X}_{t+\frac{1}{2}} - \bar{X}_{t})\|_{F}^2 ]\\
			& \leq (1-\eta)^2(1+a)  \mathbb{E}[\|X_{t}  - \bar{X}_{t} \|_F^2] + \eta^2 (1+a^{-1})  \mathbb{E}[\|X_{t+\frac{1}{2}} - \bar{X}_{t+\frac{1}{2}}\|_F^2] \\
			& \leq (1-\eta) \mathbb{E}[\|X_{t}  - \bar{X}_{t} \|_F^2] + \eta  \mathbb{E}[\|X_{t+\frac{1}{2}} - \bar{X}_{t+\frac{1}{2}}\|_F^2] \\
			& \leq (1-\eta) \mathbb{E}[\|X_{t}  - \bar{X}_{t} \|_F^2 ]+ \eta  \mathbb{E}[\|X_{t}W -  \alpha Y_{t} - \bar{X}_{t} + \alpha \bar{Y}_{t} \|_F^2] \\
			& \leq (1-\eta) \mathbb{E}[\|X_{t}  - \bar{X}_{t} \|_F^2] + \eta (1+a) \mathbb{E}[\|X_{t}W  - \bar{X}_{t} \|_F^2]  + \eta\alpha^2 (1+a^{-1}) \mathbb{E}[\| Y_{t} -  \bar{Y}_{t} \|_F^2]\\
			& \leq (1-\eta) \mathbb{E}[\|X_{t}  - \bar{X}_{t} \|_F^2] + \eta\lambda^2 (1+a) \mathbb{E}[\|X_{t}  - \bar{X}_{t} \|_F^2]  + \eta\alpha^2 (1+a^{-1}) \mathbb{E}[\| Y_{t} -  \bar{Y}_{t} \|_F^2]\\
			& \leq (1-\eta) \mathbb{E}[\|X_{t}  - \bar{X}_{t} \|_F^2] + \eta\frac{1+\lambda^2}{2} \mathbb{E}[\|X_{t}  - \bar{X}_{t} \|_F^2 ] + \frac{\eta\alpha^2(1+\lambda^2)}{1-\lambda^2}  \mathbb{E}[\| Y_{t} -  \bar{Y}_{t} \|_F^2]\\
			& \leq \Big(1-\eta\frac{1-\lambda^2}{2}\Big) \mathbb{E}[\|X_{t}  - \bar{X}_{t} \|_F^2]  + \frac{2\eta\alpha^2}{1-\lambda^2}  \mathbb{E}[\| Y_{t} -  \bar{Y}_{t} \|_F^2] \ , \\
		\end{aligned}
	\end{equation}
	where the third step holds due to $a=\frac{1-\lambda}{\lambda}$, the last step holds due to $a=\frac{1-\lambda^2}{2\lambda^2}$.
\end{proof}

\begin{lemma}\label{lemma_x_inc_momentum}
		Given Assumptions~\ref{assumption_graph}-\ref{assumption_bound_variance}, we can get
	\begin{equation}
		\begin{aligned}
			&   \mathbb{E}[\|X_{t+1} - X_{t}\|_F^2] \leq 8\eta^2 \mathbb{E}[\|X_{t}- \bar{X}_{t}\|_F^2] + 4\alpha^2\eta^2  \mathbb{E}[\|Y_{t} -\bar{Y}_{t}\|_F^2] +  4\alpha^2\eta^2  \mathbb{E}[\|\bar{M}_{t}\|_F^2 ] \ . \\
		\end{aligned}
	\end{equation}
\end{lemma}

\begin{proof}
	\begin{equation}
		\begin{aligned}
			& \quad  \mathbb{E}[\|X_{t+1} - X_{t}\|_F^2] \\
			& =   \mathbb{E}[\|X_{t}  + \eta (X_{t+\frac{1}{2}} - X_{t})- X_{t}\|_F^2]\\ 
			& = \eta^2 \mathbb{E}[\| X_{t+\frac{1}{2}} - X_{t}\|_F^2 ]\\ 
			& = \eta^2 \mathbb{E}[\|X_{t}W -\alpha Y_{t}  - X_{t}\|_F^2] \\
			& \leq 2\eta^2 \mathbb{E}[\|X_{t}W  - X_{t}\|_F^2] + 2\alpha^2\eta^2 \mathbb{E}[ \|Y_{t} \|_F^2 ]\\
			& \leq 2\eta^2 \mathbb{E}[\|(X_{t}- \bar{X}_{t})(W-I)\|_F^2] + 2\alpha^2\eta^2 \mathbb{E}[ \|Y_{t} -\bar{Y}_{t} + \bar{Y}_{t}\|_F^2] \\
			& \leq 8\eta^2 \mathbb{E}[\|X_{t}- \bar{X}_{t}\|_F^2] + 4\alpha^2\eta^2  \mathbb{E}[\|Y_{t} -\bar{Y}_{t}\|_F^2] +  4\alpha^2\eta^2  \mathbb{E}[\|\bar{M}_{t}\|_F^2 ] \ . \\
		\end{aligned}
	\end{equation}
	
\end{proof}

\begin{lemma}\label{lemma_y_consensus_momentum}
	Given Assumptions~\ref{assumption_graph}-\ref{assumption_bound_variance}, we can get
\begin{equation}
	\begin{aligned}
		&   \mathbb{E}[\|Y_{t+1}-   \bar{Y}_{t+1}\|_F^2]  \leq \lambda\mathbb{E}[\|Y_{t}-   \bar{Y}_{t}\|_F^2] + \frac{4\mu^2\eta^2}{1-\lambda} \sum_{n=1}^{N}\mathbb{E}[\|m_{n, t} -  \nabla F_{n}(x_{n, t}) \|^2]\\
		&+\sum_{n=1}^{N}\sum_{k=1}^{K-1} \frac{\mu^2\eta^2}{1-\lambda} K(4A_k+8 \beta^2  \eta^2D_{k})\mathbb{E}[\| u_{n,t}^{(k)} -f_{n}^{(k)}( u_{n,t}^{(k-1)}) \|^2]  \\
		& \quad  + \frac{4\mu^2\eta^2 }{1-\lambda} KD_{0}\sum_{n=1}^{N} \mathbb{E}[\|x_{n,t+1}-x_{n,t}\|^2 ] +  \frac{8\mu^2\eta^2}{1-\lambda} K\sum_{n=1}^{N}\sum_{k=1}^{K-1}D_{k}  C_{k}^2  \mathbb{E}[\|u_{n, t}^{(k-1)} - u_{n,t+1}^{(k-1)} \|^2] \\
		& \quad  +\frac{8 \beta^2  \mu^2\eta^4}{1-\lambda}   KN\sum_{k=1}^{K-1}D_{k}  \delta_{k}^2  + \frac{4\mu^2\eta^2}{1-\lambda}  KN\sum_{k=1}^{K}B_k \sigma_{k}^2  \ . \\
	\end{aligned}
\end{equation}

\end{lemma}

\begin{proof}
	
	\begin{equation}
		\begin{aligned}
			&\quad   \mathbb{E}[\|Y_{t+1}-   \bar{Y}_{t+1}\|_F^2] = \mathbb{E}[\|Y_{t}W  + M_{t+1} - M_{t}-   \bar{Y}_{t} - \bar{M}_{t+1} + \bar{M}_{t}\|_F^2] \\
			& \leq (1+a) \mathbb{E}[\|Y_{t}W-   \bar{Y}_{t}\|_F^2 ]+ (1+a^{-1})  \mathbb{E}[\|M_{t+1} - M_{t} - \bar{M}_{t+1} + \bar{M}_{t}\|_F^2]\\
			& \leq (1+a)\lambda^2 \mathbb{E}[\|Y_{t}-   \bar{Y}_{t}\|_F^2] + (1+a^{-1})  \mathbb{E}[\|M_{t+1} - M_{t}\|_F^2]\\
			& \leq \lambda \mathbb{E}[\|Y_{t}-   \bar{Y}_{t}\|_F^2] + \frac{1}{1-\lambda}  \mathbb{E}[\|(1-\mu\eta)M_{t} + \mu\eta G_{t+1} - M_{t}\|_F^2]\\
			& = \lambda \mathbb{E}[\|Y_{t}-   \bar{Y}_{t}\|_F^2]+ \frac{\mu^2\eta^2}{1-\lambda}  \mathbb{E}[\|M_{t} -   G_{t+1} \|_F^2] \ , \\
		\end{aligned}
	\end{equation}
where the second inequality holds due to $a=\frac{1-\lambda}{\lambda}$. 
	Furthermore, by combining it with the following inequality, we can complete the proof.  
	\begin{equation}
		\begin{aligned}
			& \quad   \mathbb{E}[\|M_{t} -   G_{t+1} \|_F^2] \leq 4\sum_{n=1}^{N} \mathbb{E}[\|m_{n, t} -  \nabla F_{n}(x_{n, t}) \|^2]\\
			& \quad  + 4\sum_{n=1}^{N} \mathbb{E}[\| \nabla F_{n}(x_{n, t}) -  \nabla f_{n}^{(1)} ( u_{n,t}^{(0)})\nabla  f_{n}^{(2)} ( u_{n,t}^{(1)}) \cdots \nabla  f_{n}^{(K-1)} ( u_{n,t}^{(K-2)})\nabla f_{n}^{(K)}( u_{n,t}^{(K-1)}) \|^2]\\
			& \quad  + 4\sum_{n=1}^{N} \mathbb{E}[\|\nabla f_{n}^{(1)} ( u_{n,t}^{(0)})\nabla  f_{n}^{(2)} ( u_{n,t}^{(1)}) \cdots \nabla  f_{n}^{(K-1)} ( u_{n,t}^{(K-2)})\nabla f_{n}^{(K)}( u_{n,t}^{(K-1)}) \\
			& \quad \quad -  \nabla f_{n}^{(1)} ( u_{n,t+1}^{(0)})\nabla  f_{n}^{(2)} ( u_{n,t+1}^{(1)}) \cdots \nabla  f_{n}^{(K-1)} ( u_{n,t+1}^{(K-2)})\nabla f_{n}^{(K)}( u_{n,t+1}^{(K-1)})\|^2] \\
			& \quad +  4\sum_{n=1}^{N} \mathbb{E}[\|\nabla f_{n}^{(1)} ( u_{n,t+1}^{(0)})\nabla  f_{n}^{(2)} ( u_{n,t+1}^{(1)}) \cdots \nabla  f_{n}^{(K-1)} ( u_{n,t+1}^{(K-2)})\nabla f_{n}^{(K)}( u_{n,t+1}^{(K-1)}) - g_{n, t+1}\|^2 ]\\
			& \leq 4\sum_{n=1}^{N} \mathbb{E}[\|m_{n, t} -  \nabla F_{n}(x_{n, t}) \|^2]+ 4K\sum_{n=1}^{N}\sum_{k=1}^{K-1}A_k \mathbb{E}[\| u_{n,t}^{(k)} -f_{n}^{(k)}( u_{n,t}^{(k-1)}) \|^2]  \\
			& \quad  + 4KD_{0}\sum_{n=1}^{N} \mathbb{E}[\|x_{n,t+1}-x_{n,t}\|^2 ] +  8K\sum_{n=1}^{N}\sum_{k=1}^{K-1}D_{k}  C_{k}^2 \mathbb{E}[ \|u_{n, t}^{(k-1)} - u_{n,t+1}^{(k-1)} \|^2 ]\\
			& \quad  + 8 \beta^2  \eta^2 K\sum_{n=1}^{N}\sum_{k=1}^{K-1}D_{k}   \mathbb{E}[\|u_{n, t}^{(k)}- f_{n}^{(k)}(u_{n, t}^{(k-1)})  \|^2] + 8 \beta^2  \eta^2 KN\sum_{k=1}^{K-1}D_{k}  \delta_{k}^2  +  4KN\sum_{k=1}^{K}B_k \sigma_{k}^2 \\
			& \leq 4\sum_{n=1}^{N} \mathbb{E}[\|m_{n, t} -  \nabla F_{n}(x_{n, t}) \|^2]+ K\sum_{n=1}^{N}\sum_{k=1}^{K-1}(4A_k+8 \beta^2  \eta^2D_{k}) \mathbb{E}[\| u_{n,t}^{(k)} -f_{n}^{(k)}( u_{n,t}^{(k-1)}) \|^2 ]  + 8 \beta^2  \eta^2 KN\sum_{k=1}^{K-1}D_{k}  \delta_{k}^2   \\
			& \quad  + 4KD_{0}\sum_{n=1}^{N} \mathbb{E}[\|x_{n,t+1}-x_{n,t}\|^2]  +  8K\sum_{n=1}^{N}\sum_{k=1}^{K-1}D_{k}  C_{k}^2  \mathbb{E}[\|u_{n, t}^{(k-1)} - u_{n,t+1}^{(k-1)} \|^2]  +  4KN\sum_{k=1}^{K}B_k \sigma_{k}^2  \ , \\
		\end{aligned}
	\end{equation}
where the third step holds due to Lemma~\ref{lemma_f_u_f_x_momentum}, Lemma~\ref{lemma_f_u_inc_momentum}, and  Lemma~\ref{lemma_g_var_momentum}.

\end{proof}

By following the proof of Lemma~\ref{lemma_m_var_momentum}, it is easy to prove the following lemma. 
\begin{lemma} \label{lemma_m_var_momentum2}
	Given Assumptions~\ref{assumption_smooth}-\ref{assumption_bound_variance}, if  $\mu\eta\in (0, 1)$, we can get
\begin{equation}
	\begin{aligned}
		& \quad \frac{1}{N}\sum_{n=1}^{N}\mathbb{E}\Big[\Big\|{m}_{n,t+1} -\nabla F_{n}({{x}}_{n,t+1})\Big\|^2\Big]  \leq (1-\mu\eta )\frac{1}{N}\sum_{n=1}^{N}\mathbb{E}\Big[\Big\|{m}_{n,t} -\nabla F_{n}({{x}}_{n,t})\Big\|^2\Big] \\
		& \quad + 2\mu\eta  K\frac{1}{N}\sum_{n=1}^{N}\sum_{k=1}^{K-1}A_k\mathbb{E}\Big[\Big\|u_{n,t}^{(k)} -f_{n}^{(k)}(u_{n,t}^{(k-1)}) \Big\|^2\Big] +4\mu\eta  K\frac{1}{N}\sum_{n=1}^{N}\sum_{k=1}^{K-1}A_k C_k^2\|{u}_{ n, t}^{(k-1)}- {u}_{n,  t+1}^{(k-1)} \Big\|^2\Big] \\
		& \quad+ 
		\frac{2 L_F^2}{\mu \eta }\frac{1}{N}\sum_{n=1}^{N}\mathbb{E}\Big[\Big\|{x}_{n,t+1} - {x}_{n,t}\Big\|^2\Big]  +4\mu \beta^2\eta^3 K\sum_{k=1}^{K-1}A_k \delta_{k}^2 + \mu^2\eta^2 K\sum_{k=1}^{K}B_{k}\sigma_{k}^2  \ . \\
	\end{aligned}
\end{equation}
\end{lemma}

Based on these lemmas, we prove Theorem~\ref{theorem1} below. 

\begin{proof}
\begin{equation}
	\begin{aligned}
		& F\left(\bar{x}_{t+1}\right) \leq F(\bar{x}_{t})+\langle \nabla F(\bar{x}_{t}),  \bar{x}_{t+1}-\bar{x}_{t} \rangle+\frac{L_{F}}{2}\|\bar{x}_{t+1}-\bar{x}_{t}\|^{2} \\
		& = F(\bar{{x}}_{t})-\alpha\eta\langle \nabla F(\bar{{x}}_{t}), \bar{m}_{t} \rangle+\frac{\alpha^2\eta^2L_{F}}{2}\|\bar{m}_{t}\|^{2} \\
		& =  F(\bar{{x}}_{t})-\frac{\alpha\eta}{2}\|\nabla F(\bar{{x}}_{t})\|^2 -(\frac{\alpha\eta}{2}-\frac{\alpha^2\eta^2L_{F}}{2})\|\bar{m}_t\|^{2}  + \frac{\alpha\eta}{2} \|\bar{m}_t - \nabla F(\bar{{x}}_{t})\|^2\\
		& \leq F(\bar{{x}}_{t})-\frac{\alpha\eta}{2}\|\nabla F(\bar{{x}}_{t})\|^2 -\frac{\alpha\eta}{4}\|\bar{m}_t\|^{2}  + \frac{\alpha\eta}{2} \|\bar{m}_t - \nabla F(\bar{{x}}_{t})\|^2\\
		& \leq F(\bar{{x}}_{t})-\frac{\alpha\eta}{2}\|\nabla F(\bar{{x}}_{t})\|^2 -\frac{\alpha\eta}{4}\|\bar{m}_t\|^{2}  +\alpha\eta \|\bar{m}_t - \frac{1}{N} \sum_{n=1}^{N}\nabla F_{n}({{x}}_{n, t})\|^2 + \alpha\eta \|\frac{1}{N} \sum_{n=1}^{N}\nabla F_{n}({{x}}_{n, t})- \nabla F(\bar{{x}}_{t})\|^2\\
		& \leq F(\bar{{x}}_{t})-\frac{\alpha\eta}{2}\|\nabla F(\bar{{x}}_{t})\|^2 -\frac{\alpha\eta}{4}\|\bar{m}_t\|^{2}  +\alpha\eta \|\bar{m}_t - \frac{1}{N} \sum_{n=1}^{N}\nabla F_{n}({{x}}_{n, t})\|^2 + \alpha\eta L_F^2\frac{1}{N} \sum_{n=1}^{N}\|{x}_{n, t}-\bar{{x}}_{t}\|^2 \ ,
	\end{aligned}
\end{equation}
where the fourth step holds due to $\eta\leq \frac{1}{2\alpha L_F}$.  

To prove Theorem~\ref{theorem1}, we introduce the following potential function:
\begin{equation}
	\begin{aligned}
		& 	\mathcal{H}_{t+1} = \mathbb{E}[F({\bar{x}}_{t+1}) ] + \omega_0\frac{1}{N}\sum_{n=1}^{N}\mathbb{E}\Big[\Big\|{m}_{n,t+1} -\nabla F_{n}({{x}}_{n,t+1})\Big\|^2\Big]  + \frac{1}{N} \sum_{n=1}^{N}\sum_{k=1}^{K-1}\omega_k \mathbb{E}[\|u_{n, t+1}^{(k)} -f_{n}^{(k)}(u_{n, t+1}^{(k-1)}) \|^2]   \\
		& + \omega_{K}\mathbb{E}\Big[\Big\|\frac{1}{N}\sum_{n=1}^{N}{m}_{n,t+1} -\frac{1}{N}\sum_{n=1}^{N}\nabla F_{n}({{x}}_{n,t+1})\Big\|^2\Big]   + \omega_{K+1}\frac{1}{N} \mathbb{E}[\|X_{t+1} - \bar{X}_{t+1}\|_{F}^2 ]+ \omega_{K+2}\frac{1}{N} \mathbb{E}[\|Y_{t+1} - \bar{Y}_{t+1}\|_{F}^2 ]  \ . 
	\end{aligned}
\end{equation}
Then, based on Lemmas~\ref{lemma_u_var_momentum},~\ref{lemma_m_var_momentum},~\ref{lemma_m_var_momentum2},~\ref{lemma_x_consensus_momentum},~\ref{lemma_y_consensus_momentum}, we can get
\begin{equation}
	\begin{aligned}
		& \quad \mathcal{H}_{t+1} - \mathcal{H}_{t} \\
		&  \leq -\frac{\alpha\eta}{2}\mathbb{E}[\|\nabla F(\bar{{x}}_{t})\|^2] -\frac{\alpha\eta}{4}\mathbb{E}[\|\bar{m}_t\|^{2} ]  +4\omega_{K} \mu \beta^2\eta^3 K\sum_{k=1}^{K-1}A_k \delta_{k}^2 + \omega_{K} \mu^2\eta^2 K\sum_{k=1}^{K}B_{k}\frac{\sigma_{k}^2}{N} + 2\beta^2\eta^2\sum_{k=1}^{K-1}\omega_k\delta_{k}^2\\
		& \quad  +\omega_{K+2}\frac{8 \beta^2  \mu^2\eta^4}{1-\lambda}   K\sum_{k=1}^{K-1}D_{k}  \delta_{k}^2  +\omega_{K+2} \frac{4\mu^2\eta^2}{1-\lambda}  K\sum_{k=1}^{K}B_k \sigma_{k}^2 +4\omega_{0}\mu \beta^2\eta^3 K\sum_{k=1}^{K-1}A_k \delta_{k}^2 + \omega_{0}\mu^2\eta^2 K\sum_{k=1}^{K}B_{k}\sigma_{k}^2 \\
				& \quad + \Big(\omega_{K+2} \frac{4\mu^2\eta^2}{1-\lambda}-\mu\eta \omega_{0}\Big)\frac{1}{N}\sum_{n=1}^{N}\mathbb{E}\Big[\Big\|{m}_{n,t} -\nabla F_{n}({{x}}_{n,t})\Big\|^2\Big] \\
				& \quad + \Big(\alpha\eta -\mu\eta \omega_{K} \Big)\mathbb{E}\Big[\Big\|\frac{1}{N}\sum_{n=1}^{N}{m}_{n,t} -\frac{1}{N}\sum_{n=1}^{N}\nabla F_{n}({{x}}_{n,t})\Big\|^2\Big] \\
		& \quad  + \frac{1}{N} \sum_{n=1}^{N}\sum_{k=1}^{K-1}\Big(\frac{\mu^2\eta^2}{1-\lambda} K(4A_k+8 \beta^2  \eta^2D_{k})\omega_{K+2}+ 2\omega_{K} \mu\eta  KA_k+ 2\omega_{0}\mu\eta  KA_k- \omega_k\beta\eta\Big) \mathbb{E}\Big[\Big\|u_{n,t}^{(k)} -f_{n}^{(k)}(u_{n,t}^{(k-1)}) \Big\|^2\Big] \\
		& \quad+ 
		\Big(\frac{2 L_F^2}{\mu \eta }\omega_{K}+ \omega_{0}
		\frac{2 L_F^2}{\mu \eta } +  \omega_{K+2}\frac{4\mu^2\eta^2 }{1-\lambda} KD_{0}\Big)\frac{1}{N}\mathbb{E}[\|{X}_{t+1} - {X}_{t}\|_F^2] \\
		& \quad + \Big(\alpha\eta L_F^2-\eta\frac{1-\lambda^2}{2}\omega_{K+1}\Big)\frac{1}{N}\mathbb{E}[\|X_{t}  - \bar{X}_{t} \|_F^2]   + \Big(\frac{2\eta\alpha^2}{1-\lambda^2} \omega_{K+1} - (1-\lambda) \omega_{K+2}\Big)\frac{1}{N}\mathbb{E}[\| Y_{t} -  \bar{Y}_{t} \|_F^2]\\
				& \quad  + \frac{1}{N} \sum_{n=1}^{N}\sum_{k=1}^{K-1}\Big(4\omega_{K} \mu\eta  KA_k C_k^2+ 4\omega_{0}\mu\eta  KA_k C_k^2+ 2\omega_kC_k^2 + \omega_{K+2}\frac{8\mu^2\eta^2}{1-\lambda} KD_{k}  C_{k}^2\Big)\mathbb{E}[\|u_{n, t}^{(k-1)}- u_{n,t+1}^{(k-1)} \|^2]  \ .  \\
	\end{aligned}
\end{equation}
Then, according to Lemma~\ref{lemma_u_inc_momentum}, we can get
\begin{equation}
	\small
	\begin{aligned}
		&  \quad \mathcal{H}_{t+1} - \mathcal{H}_{t}  \\
		& \leq -\frac{\alpha\eta}{2}\mathbb{E}[\|\nabla F(\bar{{x}}_{t})\|^2] -\frac{\alpha\eta}{4}\mathbb{E}[\|\bar{m}_t\|^{2}] +4\omega_{K} \mu \beta^2\eta^3 K\sum_{k=1}^{K-1}A_k \delta_{k}^2 + \omega_{K} \mu^2\eta^2 K\sum_{k=1}^{K}B_{k}\frac{\sigma_{k}^2}{N} + 2\beta^2\eta^2\sum_{k=1}^{K-1}\omega_k\delta_{k}^2 \\
		& \quad  +\omega_{K+2}\frac{8 \beta^2  \mu^2\eta^4}{1-\lambda}   K\sum_{k=1}^{K-1}D_{k}  \delta_{k}^2  +\omega_{K+2} \frac{4\mu^2\eta^2}{1-\lambda}  K\sum_{k=1}^{K}B_k \sigma_{k}^2 +4\omega_{0}\mu \beta^2\eta^3 K\sum_{k=1}^{K-1}A_k \delta_{k}^2 + \omega_{0}\mu^2\eta^2 K\sum_{k=1}^{K}B_{k}\sigma_{k}^2 \\
		& \quad + \Big(\omega_{K+2} \frac{4\mu^2\eta^2}{1-\lambda}-\mu\eta \omega_{0}\Big)\frac{1}{N}\sum_{n=1}^{N}\mathbb{E}\Big[\Big\|{m}_{n,t} -\nabla F_{n}({{x}}_{n,t})\Big\|^2\Big] + \Big(\alpha\eta -\mu\eta \omega_{K} \Big)\mathbb{E}\Big[\Big\|\frac{1}{N}\sum_{n=1}^{N}{m}_{n,t} -\frac{1}{N}\sum_{n=1}^{N}\nabla F_{n}({{x}}_{n,t})\Big\|^2\Big] \\
		& \quad  + \frac{1}{N} \sum_{n=1}^{N}\sum_{k=1}^{K-1}\Big(\frac{\mu^2\eta^2}{1-\lambda} K(4A_k+8 \beta^2  \eta^2D_{k})\omega_{K+2}+ 2\omega_{K} \mu\eta  KA_k+ 2\omega_{0}\mu\eta  KA_k- \omega_k\beta\eta\Big) \mathbb{E}\Big[\Big\|u_{n,t}^{(k)} -f_{n}^{(k)}(u_{n,t}^{(k-1)}) \Big\|^2\Big] \\
				& \quad + \Big(\alpha\eta L_F^2-\eta\frac{1-\lambda^2}{2}\omega_{K+1}\Big)\frac{1}{N}\mathbb{E}[\|X_{t}  - \bar{X}_{t} \|_F^2]   + \Big(\frac{2\eta\alpha^2}{1-\lambda^2} \omega_{K+1} - (1-\lambda) \omega_{K+2}\Big)\frac{1}{N}\mathbb{E}[\| Y_{t} -  \bar{Y}_{t} \|_F^2]\\
		& \quad+ 
		\Big(\frac{2 L_F^2}{\mu \eta }\omega_{K}+ \omega_{0}
		\frac{2 L_F^2}{\mu \eta } +  \omega_{K+2}\frac{4\mu^2\eta^2 }{1-\lambda} KD_{0}\Big)\frac{1}{N}\mathbb{E}[\|{X}_{t+1} - {X}_{t}\|_F^2] \\
		& \quad  + \frac{1}{N} \sum_{n=1}^{N}\sum_{k=1}^{K-1}\Big(4\omega_{K} \mu\eta  KA_k C_k^2+ 4\omega_{0}\mu\eta  KA_k C_k^2+ 2\omega_kC_k^2 + \omega_{K+2}\frac{8\mu^2\eta^2}{1-\lambda} KD_{k}  C_{k}^2\Big)\Big(\prod_{j=1}^{k-1}(2C_{j}^2)\Big)\mathbb{E}[\|u_{n, t}^{(0)} - u_{n,t+1}^{(0)} \|^2]  \\
		& \quad  +2\beta^2  \eta^2 \frac{1}{N} \sum_{n=1}^{N}\sum_{k=1}^{K-1}\Big(4\omega_{K} \mu\eta  KA_k C_k^2+ 4\omega_{0}\mu\eta  KA_k C_k^2+ 2\omega_kC_k^2\\
		& \quad \quad  + \omega_{K+2}\frac{8\mu^2\eta^2}{1-\lambda} KD_{k}  C_{k}^2\Big) \sum_{j=1}^{k-1} \Big(\prod_{i=j+1}^{k-1}(2C_{i}^2)\Big)\mathbb{E}[\|u_{n, t}^{(j)}- f_{n}^{(j)}(u_{n, t}^{(j-1)})  \|^2] \\
		& \quad  + 2\beta^2  \eta^2\frac{1}{N} \sum_{n=1}^{N}\sum_{k=1}^{K-1}\Big(4\omega_{K} \mu\eta  KA_k C_k^2+ 4\omega_{0}\mu\eta  KA_k C_k^2+ 2\omega_kC_k^2 + \omega_{K+2}\frac{8\mu^2\eta^2}{1-\lambda} KD_{k}  C_{k}^2\Big) \sum_{j=1}^{k-1} \Big(\prod_{i=j+1}^{k-1}(2C_{i}^2)\Big)\delta_{j}^2 \ .  \\
		& \leq -\frac{\alpha\eta}{2}\mathbb{E}[\|\nabla F(\bar{{x}}_{t})\|^2] -\frac{\alpha\eta}{4}\mathbb{E}[\|\bar{m}_t\|^{2}]  +4\omega_{K} \mu \beta^2\eta^3 K\sum_{k=1}^{K-1}A_k \delta_{k}^2 + \omega_{K} \mu^2\eta^2 K\sum_{k=1}^{K}B_{k}\frac{\sigma_{k}^2}{N} + 2\beta^2\eta^2\sum_{k=1}^{K-1}\omega_k\delta_{k}^2\\
		& \quad  +\omega_{K+2}\frac{8 \beta^2  \mu^2\eta^4}{1-\lambda}   K\sum_{k=1}^{K-1}D_{k}  \delta_{k}^2  +\omega_{K+2} \frac{4\mu^2\eta^2}{1-\lambda}  K\sum_{k=1}^{K}B_k \sigma_{k}^2 +4\omega_{0}\mu \beta^2\eta^3 K\sum_{k=1}^{K-1}A_k \delta_{k}^2 + \omega_{0}\mu^2\eta^2 K\sum_{k=1}^{K}B_{k}\sigma_{k}^2 \\
		& \quad + \Big(\omega_{K+2} \frac{4\mu^2\eta^2}{1-\lambda}-\mu\eta \omega_{0}\Big)\frac{1}{N}\sum_{n=1}^{N}\mathbb{E}\Big[\Big\|{m}_{n,t} -\nabla F_{n}({{x}}_{n,t})\Big\|^2\Big]+ \Big(\alpha\eta -\mu\eta \omega_{K} \Big)\mathbb{E}\Big[\Big\|\frac{1}{N}\sum_{n=1}^{N}{m}_{n,t} -\frac{1}{N}\sum_{n=1}^{N}\nabla F_{n}({{x}}_{n,t})\Big\|^2\Big] \\
		& \quad  + \frac{1}{N} \sum_{n=1}^{N}\sum_{k=1}^{K-1}\Big(\frac{\mu^2\eta^2}{1-\lambda} K(4A_k+8 \beta^2  \eta^2D_{k})\omega_{K+2}+ 2\omega_{K} \mu\eta  KA_k+ 2\omega_{0}\mu\eta  KA_k- \omega_k\beta\eta\Big) \mathbb{E}\Big[\Big\|u_{n,t}^{(k)} -f_{n}^{(k)}(u_{n,t}^{(k-1)}) \Big\|^2\Big] \\
		& \quad+ 
		\Big(\frac{2 L_F^2}{\mu \eta }\omega_{K}+ \omega_{0}
		\frac{2 L_F^2}{\mu \eta } +  \omega_{K+2}\frac{4\mu^2\eta^2 }{1-\lambda} KD_{0} \\
		& \quad \quad + \sum_{k=1}^{K-1}\Big(4\omega_{K} \mu\eta  KA_k C_k^2+ 4\omega_{0}\mu\eta  KA_k C_k^2+ 2\omega_kC_k^2 + \omega_{K+2}\frac{8\mu^2\eta^2}{1-\lambda} KD_{k}  C_{k}^2\Big)\Big(\prod_{j=1}^{k-1}(2C_{j}^2)\Big)\Big)\frac{1}{N}\mathbb{E}[\|{X}_{t+1} - {X}_{t}\|_F^2] \\
		& \quad + \Big(\alpha\eta L_F^2-\eta\frac{1-\lambda^2}{2}\omega_{K+1}\Big)\frac{1}{N}\mathbb{E}[\|X_{t}  - \bar{X}_{t} \|_F^2] + \Big(\frac{2\eta\alpha^2}{1-\lambda^2} \omega_{K+1} - (1-\lambda) \omega_{K+2}\Big)\frac{1}{N}\mathbb{E}[\| Y_{t} -  \bar{Y}_{t} \|_F^2]\\
		& \quad  + 2\beta^2  \eta^2\frac{1}{N} \sum_{n=1}^{N}\sum_{k=1}^{K-1}\Bigg[\sum_{j=k+1}^{K-1}\Big(4\omega_{K} \mu\eta  KA_j C_j^2+ 4\omega_{0}\mu\eta  KA_j C_j^2+ 2\omega_jC_j^2 \\
		& \quad \quad + \omega_{K+2}\frac{8\mu^2\eta^2}{1-\lambda} KD_{j}  C_{j}^2\Big)  \Big(\prod_{i=k+1}^{j}(2C_{i}^2)\Big)\Bigg]\mathbb{E}[\|u_{n, t}^{(k)}- f_{n}^{(k)}(u_{n, t}^{(k-1)})  \|^2]  \\
		& \quad  + 2\beta^2  \eta^2\frac{1}{N} \sum_{n=1}^{N}\sum_{k=1}^{K-1}\Bigg[\sum_{j=k+1}^{K-1}\Big(4\omega_{K} \mu\eta  KA_j C_j^2+ 4\omega_{0}\mu\eta  KA_j C_j^2+ 2\omega_jC_j^2 + \omega_{K+2}\frac{8\mu^2\eta^2}{1-\lambda} KD_{j}  C_{j}^2\Big)  \Big(\prod_{i=k+1}^{j}(2C_{i}^2)\Big)\Bigg]\delta_{k}^2 \ . \\
	\end{aligned}
\end{equation}

Based on Lemma~\ref{lemma_x_inc_momentum}, we can get
\begin{equation}
	\small
	\begin{aligned}
		& \quad \mathcal{H}_{t+1} - \mathcal{H}_{t} \\
		& \leq -\frac{\alpha\eta}{2}\mathbb{E}[\|\nabla F(\bar{{x}}_{t})\|^2]  +4\omega_{K} \mu \beta^2\eta^3 K\sum_{k=1}^{K-1}A_k \delta_{k}^2 + \omega_{K} \mu^2\eta^2 K\sum_{k=1}^{K}B_{k}\frac{\sigma_{k}^2}{N} + 2\beta^2\eta^2\sum_{k=1}^{K-1}\omega_k\delta_{k}^2 \\
		& \quad  +\omega_{K+2}\frac{8 \beta^2  \mu^2\eta^4}{1-\lambda}   K\sum_{k=1}^{K-1}D_{k}  \delta_{k}^2  +\omega_{K+2} \frac{4\mu^2\eta^2}{1-\lambda}  K\sum_{k=1}^{K}B_k \sigma_{k}^2 +4\omega_{0}\mu \beta^2\eta^3 K\sum_{k=1}^{K-1}A_k \delta_{k}^2 + \omega_{0}\mu^2\eta^2 K\sum_{k=1}^{K}B_{k}\sigma_{k}^2 \\
		& \quad  + 2\beta^2  \eta^2\frac{1}{N} \sum_{n=1}^{N}\sum_{k=1}^{K-1}\Bigg[\sum_{j=k+1}^{K-1}\Big(4\omega_{K} \mu\eta  KA_j C_j^2+ 4\omega_{0}\mu\eta  KA_j C_j^2+ 2\omega_jC_j^2 + \omega_{K+2}\frac{8\mu^2\eta^2}{1-\lambda} KD_{j}  C_{j}^2\Big)  \Big(\prod_{i=k+1}^{j}(2C_{i}^2)\Big)\Bigg]\delta_{k}^2 \\
		& \quad + \Big(\omega_{K+2} \frac{4\mu^2\eta^2}{1-\lambda}-\mu\eta \omega_{0}\Big)\frac{1}{N}\sum_{n=1}^{N}\mathbb{E}\Big[\Big\|{m}_{n,t} -\nabla F_{n}({{x}}_{n,t})\Big\|^2\Big] \\
		& \quad + \Big(\alpha\eta -\mu\eta \omega_{K} \Big)\mathbb{E}\Big[\Big\|\frac{1}{N}\sum_{n=1}^{N}{m}_{n,t} -\frac{1}{N}\sum_{n=1}^{N}\nabla F_{n}({{x}}_{n,t})\Big\|^2\Big] \\
		& \quad  + \frac{1}{N} \sum_{n=1}^{N}\sum_{k=1}^{K-1}\Bigg[\frac{\mu^2\eta^2}{1-\lambda} K(4A_k+8 \beta^2  \eta^2D_{k})\omega_{K+2}+ 2\omega_{K} \mu\eta  KA_k+ 2\omega_{0}\mu\eta  KA_k- \omega_k\beta\eta  + 2\beta^2  \eta^2\sum_{j=k+1}^{K-1}\Big( 2\omega_jC_j^2\\
		& \quad \quad +4\omega_{K} \mu\eta  KA_j C_j^2+ 4\omega_{0}\mu\eta  KA_j C_j^2 + \omega_{K+2}\frac{8\mu^2\eta^2}{1-\lambda} KD_{j}  C_{j}^2\Big)  \Big(\prod_{i=k+1}^{j}(2C_{i}^2)\Big)\Bigg] \mathbb{E}\Big[\Big\|u_{n,t}^{(k)} -f_{n}^{(k)}(u_{n,t}^{(k-1)}) \Big\|^2\Big] \\
		& \quad + \Bigg[\alpha\eta L_F^2-\eta\frac{1-\lambda^2}{2}\omega_{K+1}+ 8\eta^2\Bigg(\frac{2 L_F^2}{\mu \eta }\omega_{K}+ \omega_{0}
		\frac{2 L_F^2}{\mu \eta } +  \omega_{K+2}\frac{4\mu^2\eta^2 }{1-\lambda} KD_{0} \\
		& \quad \quad + \sum_{k=1}^{K-1}\Big(4\omega_{K} \mu\eta  KA_k C_k^2+ 4\omega_{0}\mu\eta  KA_k C_k^2+ 2\omega_kC_k^2 + \omega_{K+2}\frac{8\mu^2\eta^2}{1-\lambda} KD_{k}  C_{k}^2\Big)\Big(\prod_{j=1}^{k-1}(2C_{j}^2)\Big)\Bigg)\Bigg]\frac{1}{N}\mathbb{E}[\|X_{t}  - \bar{X}_{t} \|_F^2 ]\\
		& \quad  + \Bigg[\frac{2\eta\alpha^2}{1-\lambda^2} \omega_{K+1} - (1-\lambda) \omega_{K+2}+ 
		4\alpha^2\eta^2 \Bigg(\frac{2 L_F^2}{\mu \eta }\omega_{K}+ \omega_{0}
		\frac{2 L_F^2}{\mu \eta } +  \omega_{K+2}\frac{4\mu^2\eta^2 }{1-\lambda} KD_{0} \\
		& \quad \quad + \sum_{k=1}^{K-1}\Big(4\omega_{K} \mu\eta  KA_k C_k^2+ 4\omega_{0}\mu\eta  KA_k C_k^2+ 2\omega_kC_k^2 + \omega_{K+2}\frac{8\mu^2\eta^2}{1-\lambda} KD_{k}  C_{k}^2\Big)\Big(\prod_{j=1}^{k-1}(2C_{j}^2)\Big)\Bigg)\Bigg]\frac{1}{N}\mathbb{E}[\| Y_{t} -  \bar{Y}_{t} \|_F^2]\\
		& \quad+ \Bigg[4\alpha^2\eta^2\Bigg(\frac{2 L_F^2}{\mu \eta }\omega_{K}+ \omega_{0}
		\frac{2 L_F^2}{\mu \eta } +  \omega_{K+2}\frac{4\mu^2\eta^2 }{1-\lambda} KD_{0} \\
		& \quad \quad + \sum_{k=1}^{K-1}\Big(4\omega_{K} \mu\eta  KA_k C_k^2+ 4\omega_{0}\mu\eta  KA_k C_k^2+ 2\omega_kC_k^2 + \omega_{K+2}\frac{8\mu^2\eta^2}{1-\lambda} KD_{k}  C_{k}^2\Big)\Big(\prod_{j=1}^{k-1}(2C_{j}^2)\Big)\Bigg) - \frac{\alpha\eta}{4}\Bigg]
		\mathbb{E}[ \|\bar{m}_{t}\|^2]  \ . \\
	\end{aligned}
\end{equation}
In the following, we enforce the coefficient of the last six terms to be non-positive. Specifically, by setting $\omega_{K} = \frac{\alpha}{\mu}$, we can get $\alpha\eta -\mu\eta \omega_{K}  \leq  0 $. Moreover, we set $\omega_{K+2} = \alpha(1-\lambda)$ and $ \omega_{K+2} \frac{4\mu^2\eta^2}{1-\lambda}-\mu\eta \omega_{0} =  0$ so that we can get $\omega_{0} = \omega_{K+2} \frac{4\mu\eta}{1-\lambda} =4\alpha \mu\eta $.

Then, we enforce  
\begin{equation}
	\begin{aligned}
		& \frac{\mu^2\eta^2}{1-\lambda} K(4A_k+8 \beta^2  \eta^2D_{k})\omega_{K+2}+ 2\omega_{K} \mu\eta  KA_k+ 2\omega_{0}\mu\eta  KA_k- \omega_k\beta\eta \\
		& \quad + 2\beta^2  \eta^2\sum_{j=k+1}^{K-1}\Big(4\omega_{K} \mu\eta  KA_j C_j^2+ 4\omega_{0}\mu\eta  KA_j C_j^2+ 2\omega_jC_j^2 + \omega_{K+2}\frac{8\mu^2\eta^2}{1-\lambda} KD_{j}  C_{j}^2\Big)  \Big(\prod_{i=k+1}^{j}(2C_{i}^2)\Big) \leq  0  \  .\\
	\end{aligned}
\end{equation}
This is equivalent to enforce 
\begin{equation}
	\begin{aligned}
		& \quad \frac{\mu^2\eta^2}{1-\lambda} K(4A_k+8 \beta^2  \eta^2D_{k})\omega_{K+2}+ 2\frac{\alpha}{\mu}\mu\eta  KA_k+ 2\omega_{K+2} \frac{4\mu\eta}{1-\lambda} \mu\eta  KA_k- \omega_k\beta\eta \\
		& \quad + 2\beta^2  \eta^2\sum_{j=k+1}^{K-1}\Big(4\frac{\alpha}{\mu}\mu\eta  KA_j C_j^2+ 4\omega_{K+2} \frac{4\mu\eta}{1-\lambda} \mu\eta  KA_j C_j^2+ 2\omega_jC_j^2 + \omega_{K+2}\frac{8\mu^2\eta^2}{1-\lambda} KD_{j}  C_{j}^2\Big)  \Big(\prod_{i=k+1}^{j}(2C_{i}^2)\Big)  \\
		& \leq \frac{\mu^2\eta^2}{1-\lambda} K(4A_k+8 \beta^2  \eta^2D_{k})\alpha(1-\lambda) + 2\frac{\alpha}{\mu}\mu\eta  KA_k+ 2\alpha(1-\lambda) \frac{4\mu\eta}{1-\lambda} \mu\eta  KA_k- \omega_k\beta\eta \\
		& \quad + 2\beta^2  \eta^2\sum_{j=k+1}^{K-1}\Big(4\frac{\alpha}{\mu}\mu\eta  KA_j C_j^2+ 4\alpha(1-\lambda)  \frac{4\mu\eta}{1-\lambda} \mu\eta  KA_j C_j^2+ 2\omega_jC_j^2 + \alpha(1-\lambda) \frac{8\mu^2\eta^2}{1-\lambda} KD_{j}  C_{j}^2\Big)  \Big(\prod_{i=k+1}^{j}(2C_{i}^2)\Big)  \\
		& \leq  K(4A_k+8 \beta^2  \eta^2D_{k})\alpha\mu^2\eta^2 + 2\alpha\eta  KA_k+ 8\alpha \mu^2\eta^2 KA_k- \omega_k\beta\eta \\
		& \quad + 2\beta^2  \eta^2\sum_{j=k+1}^{K-1}\Big(4\alpha\eta  KA_j C_j^2+ 16\alpha \mu^2\eta^2   KA_j C_j^2+ 2\omega_jC_j^2 + 8\alpha\mu^2\eta^2KD_{j}  C_{j}^2\Big)  \Big(\prod_{i=k+1}^{j}(2C_{i}^2)\Big)  \leq  0   \ . \\
	\end{aligned}
\end{equation}
It can be done by enforcing
\begin{equation}
	\begin{aligned}
		& 2\beta^2  \eta^2\sum_{j=k+1}^{K-1}\Big( 2\omega_jC_j^2 \Big)  \Big(\prod_{i=k+1}^{j}(2C_{i}^2)\Big)- \omega_k\beta\eta \leq - \frac{1}{2}\omega_k\beta\eta \ ,  \\
		&   K(4A_k+8 \beta^2  \eta^2D_{k})\alpha\mu^2\eta + 2\alpha  KA_k+ 8\alpha \mu^2\eta KA_k\\
		& \quad + 2\beta^2  \eta\sum_{j=k+1}^{K-1}\Big(4\alpha\eta  KA_j C_j^2+ 16\alpha \mu^2\eta^2   KA_j C_j^2+ 8\alpha\mu^2\eta^2KD_{j}  C_{j}^2\Big)  \Big(\prod_{i=k+1}^{j}(2C_{i}^2)\Big) \leq \frac{1}{2} \omega_k\beta  \ .\\
	\end{aligned}
\end{equation}
As for the first inequality, we can get
\begin{equation}
	\begin{aligned}
		& \eta \leq \frac{\omega_k}{4\beta\sum_{j=1}^{K-1}\Big( 2\omega_jC_j^2 \Big)  \Big(\prod_{i=k+1}^{j}(2C_{i}^2)\Big) }  \ . \\
	\end{aligned}
\end{equation}
As for the second inequality, we can get
\begin{equation}
	\begin{aligned}
		& \frac{1}{2} \omega_k\beta \geq  K(4A_k+8 \beta^2  \eta^2D_{k})\alpha\mu^2\eta + 2\alpha  KA_k+ 8\alpha \mu^2\eta KA_k\\
		& \quad + 2\beta^2  \eta\sum_{j=k+1}^{K-1}\Big(4\alpha\eta  KA_j C_j^2+ 16\alpha \mu^2\eta^2   KA_j C_j^2+ 8\alpha\mu^2\eta^2KD_{j}  C_{j}^2\Big)  \Big(\prod_{i=k+1}^{j}(2C_{i}^2)\Big)  \  .  \\
	\end{aligned}
\end{equation}
Then, due to $\beta\eta<1$, $\mu\eta<1$,  we can set
\begin{equation}
	\begin{aligned}
		& \omega_k = \frac{2\alpha K}{\beta} \Bigg((12A_k+8 D_{k})\mu  + 2  A_k+  2 \beta  \sum_{j=k+1}^{K-1}\Big( 20    A_j C_j^2+ 8 D_{j}  C_{j}^2\Big)  \Big(\prod_{i=k+1}^{j}(2C_{i}^2)\Big)  \Bigg) \ . \\
	\end{aligned}
\end{equation}
Here, we represent $\omega_k \triangleq \alpha K \tilde{\omega}_k$,   where $\tilde{\omega}_k= \frac{2}{\beta} \Bigg((12A_k+8 D_{k})\mu  + 2  A_k+  2 \beta  \sum_{j=k+1}^{K-1}\Big( 20    A_j C_j^2+ 8 D_{j}  C_{j}^2\Big)  \Big(\prod_{i=k+1}^{j}(2C_{i}^2)\Big)  \Bigg)$.  Then, we can simplify the upper bound of $\eta$ as follows:
\begin{equation}
	\begin{aligned}
		& \eta \leq \frac{ \tilde{\omega}_k}{8\beta\sum_{j=1}^{K-1}  \tilde{\omega}_jC_j^2  \Big(\prod_{i=k+1}^{j}(2C_{i}^2)\Big) }  \ . \\
	\end{aligned}
\end{equation}

Based on the value of $\omega_{k}$ where $k\in \{0, 1, \cdots, K\}$ and $\omega_{K+2}$ , due to $\eta<1$,  we can get
\begin{equation}
	\begin{aligned}
		&\quad  \frac{2 L_F^2}{\mu \eta }\omega_{K}+ \omega_{0}
		\frac{2 L_F^2}{\mu \eta } +  \omega_{K+2}\frac{4\mu^2\eta^2 }{1-\lambda} KD_{0} \\
		& \quad \quad + \sum_{k=1}^{K-1}\Big(4\omega_{K} \mu\eta  KA_k C_k^2+ 4\omega_{0}\mu\eta  KA_k C_k^2+ 2\omega_kC_k^2 + \omega_{K+2}\frac{8\mu^2\eta^2}{1-\lambda} KD_{k}  C_{k}^2\Big)\Big(\prod_{j=1}^{k-1}(2C_{j}^2)\Big) \\
		& = \frac{2 \alpha L_F^2}{\mu^2 \eta }+ 8 \alpha L_F^2
		+  4\alpha KD_{0} + \sum_{k=1}^{K-1}\Big(4\alpha\eta  KA_k C_k^2+ 16\alpha KA_k C_k^2+ 2\alpha K \tilde{\omega}_kC_k^2 + 8\alpha KD_{k}  C_{k}^2\Big)\Big(\prod_{j=1}^{k-1}(2C_{j}^2)\Big) \\
		& \leq  \frac{2 \alpha L_F^2}{\mu^2 \eta }+ 8 \alpha L_F^2
		+  4\alpha KD_{0} + \alpha K\sum_{k=1}^{K-1}\Big(20 A_k C_k^2+ 2 \tilde{\omega}_kC_k^2 + 8 D_{k}  C_{k}^2\Big)\Big(\prod_{j=1}^{k-1}(2C_{j}^2)\Big) \ . \\
	\end{aligned}
\end{equation}

Furthermore, we enforce
\begin{equation}
	\begin{aligned}
		& \quad \alpha\eta L_F^2-\eta\frac{1-\lambda^2}{2}\omega_{K+1}+ 8\eta^2\Bigg(\frac{2 L_F^2}{\mu \eta }\omega_{K}+ \omega_{0}
		\frac{2 L_F^2}{\mu \eta } +  \omega_{K+2}\frac{4\mu^2\eta^2 }{1-\lambda} KD_{0} \\
		& \quad \quad + \sum_{k=1}^{K-1}\Big(4\omega_{K} \mu\eta  KA_k C_k^2+ 4\omega_{0}\mu\eta  KA_k C_k^2+ 2\omega_kC_k^2 + \omega_{K+2}\frac{8\mu^2\eta^2}{1-\lambda} KD_{k}  C_{k}^2\Big)\Big(\prod_{j=1}^{k-1}(2C_{j}^2)\Big)\Bigg) \\
		& \leq \alpha\eta L_F^2-\eta\frac{1-\lambda^2}{2}\omega_{K+1}+ 8\eta^2\Bigg(\frac{2 \alpha L_F^2}{\mu^2 \eta }+ 8 \alpha L_F^2
		+  4\alpha KD_{0} + \alpha K\sum_{k=1}^{K-1}\Big(20 A_k C_k^2+ 2 \tilde{\omega}_kC_k^2 + 8 D_{k}  C_{k}^2\Big)\Big(\prod_{j=1}^{k-1}(2C_{j}^2)\Big) \Bigg) \\
		&\leq  0 \ . \\
	\end{aligned}
\end{equation}
Similarly, due to $\eta<1$, we can set
\begin{equation}
	\begin{aligned}
		& \omega_{K+1}= \frac{2\alpha}{(1-\lambda^2)}\Bigg[ L_F^2+ 8\Bigg(\frac{2  L_F^2}{\mu^2  }+ 8  L_F^2
		+  4 KD_{0} +  K\sum_{k=1}^{K-1}\Big(20 A_k C_k^2+ 2 \tilde{\omega}_kC_k^2 + 8 D_{k}  C_{k}^2\Big)\Big(\prod_{j=1}^{k-1}(2C_{j}^2)\Big) \Bigg) \Bigg] \ . \\
	\end{aligned}
\end{equation}
Here, we represent $ \omega_{K+1}\triangleq \frac{2\alpha}{(1-\lambda^2)} \tilde{\omega}_{K+1}$, where  $\tilde{\omega}_{K+1}=\Bigg[ L_F^2+ 8\Bigg(\frac{2  L_F^2}{\mu^2  }+ 8  L_F^2
+  4 KD_{0} +  K\sum_{k=1}^{K-1}\Big(20 A_k C_k^2+ 2 \tilde{\omega}_kC_k^2 + 8 D_{k}  C_{k}^2\Big)\Big(\prod_{j=1}^{k-1}(2C_{j}^2)\Big) \Bigg) \Bigg]$.

In addition, we enforce
\begin{equation}
	\begin{aligned}
		& \frac{2\eta\alpha^2}{1-\lambda^2} \omega_{K+1} - (1-\lambda) \omega_{K+2}+ 
		4\alpha^2\eta^2 \Bigg(\frac{2 L_F^2}{\mu \eta }\omega_{K}+ \omega_{0}
		\frac{2 L_F^2}{\mu \eta } +  \omega_{K+2}\frac{4\mu^2\eta^2 }{1-\lambda} KD_{0} \\
		& \quad \quad + \sum_{k=1}^{K-1}\Big(4\omega_{K} \mu\eta  KA_k C_k^2+ 4\omega_{0}\mu\eta  KA_k C_k^2+ 2\omega_kC_k^2 + \omega_{K+2}\frac{8\mu^2\eta^2}{1-\lambda} KD_{k}  C_{k}^2\Big)\Big(\prod_{j=1}^{k-1}(2C_{j}^2)\Big)\Bigg) \\
		& \leq \frac{2\eta\alpha^2}{1-\lambda^2}\frac{2\alpha}{(1-\lambda^2)}  \tilde{\omega}_{K+1} - \alpha(1-\lambda)^2 \\
		& \quad + 
		4\alpha^2\eta^2 \Bigg(\frac{2 \alpha L_F^2}{\mu^2 \eta }+ 8 \alpha L_F^2
		+  4\alpha KD_{0} + \alpha K\sum_{k=1}^{K-1}\Big(20 A_k C_k^2+ 2 \tilde{\omega}_kC_k^2 + 8 D_{k}  C_{k}^2\Big)\Big(\prod_{j=1}^{k-1}(2C_{j}^2)\Big) \Bigg)  \leq  0  \ . \\
	\end{aligned}
\end{equation}
Due to $\eta<1$ and $1+\lambda>1$, we can get
\begin{equation}
	\begin{aligned}
		& \alpha \leq \frac{ (1-\lambda)^2}{ \sqrt{4\tilde{\omega}_{K+1} + 
				8 L_F^2/\mu^2  + 32  L_F^2
				+  16 KD_{0} +  4K\sum_{k=1}^{K-1}\Big(20 A_k C_k^2+ 2 \tilde{\omega}_kC_k^2 + 8 D_{k}  C_{k}^2\Big)\Big(\prod_{j=1}^{k-1}(2C_{j}^2)\Big) }} \ . 
	\end{aligned}
\end{equation}

And we enforce
\begin{equation}
	\begin{aligned}
		& 4\alpha^2\eta^2\Bigg(\frac{2 L_F^2}{\mu \eta }\omega_{K}+ \omega_{0}
		\frac{2 L_F^2}{\mu \eta } +  \omega_{K+2}\frac{4\mu^2\eta^2 }{1-\lambda} KD_{0} \\
		& \quad \quad + \sum_{k=1}^{K-1}\Big(4\omega_{K} \mu\eta  KA_k C_k^2+ 4\omega_{0}\mu\eta  KA_k C_k^2+ 2\omega_kC_k^2 + \omega_{K+2}\frac{8\mu^2\eta^2}{1-\lambda} KD_{k}  C_{k}^2\Big)\Big(\prod_{j=1}^{k-1}(2C_{j}^2)\Big)\Bigg) - \frac{\alpha\eta}{4} \\
		& \leq 4\alpha^2\eta^2\Bigg(\frac{2 \alpha L_F^2}{\mu^2 \eta }+ 8 \alpha L_F^2
		+  4\alpha KD_{0} + \alpha K\sum_{k=1}^{K-1}\Big(20 A_k C_k^2+ 2 \tilde{\omega}_kC_k^2 + 8 D_{k}  C_{k}^2\Big)\Big(\prod_{j=1}^{k-1}(2C_{j}^2)\Big) \Bigg) - \frac{\alpha\eta}{4} \leq  0  \ . \\
	\end{aligned}
\end{equation}
Similarly, due to $\eta<1$, we can get
\begin{equation}
	\begin{aligned}
		& \alpha \leq  \frac{1}{4\sqrt{2  L_F^2/\mu^2  + 8  L_F^2
				+  4 KD_{0} +  K\sum_{k=1}^{K-1}\Big(20 A_k C_k^2+ 2 \tilde{\omega}_kC_k^2 + 8 D_{k}  C_{k}^2\Big)\Big(\prod_{j=1}^{k-1}(2C_{j}^2)\Big) }}  \ . 
	\end{aligned} 
\end{equation}

In summary, by setting 
\begin{equation} \label{eq_hyperparams_momentum}
	\begin{aligned}
		& \omega_{0} =4\alpha \mu\eta \ ,  \\
		& \omega_k = \alpha K \tilde{\omega}_k  \ ,   \forall k\in \{1, 2, \cdots, K-1\} \ , \\
		& \omega_{K} = \frac{\alpha}{\mu}  \ , \\
		& \omega_{K+1}= \frac{2\alpha}{1-\lambda^2} \tilde{\omega}_{K+1}  \ , \\
		& \omega_{K+2} = \alpha(1-\lambda)  \ , \\
		& \eta \leq \frac{ \tilde{\omega}_k}{8\beta\sum_{j=1}^{K-1}  \tilde{\omega}_jC_j^2  \Big(\prod_{i=k+1}^{j}(2C_{i}^2)\Big) }    \ , \\
		& \alpha \leq \frac{ (1-\lambda)^2}{ \sqrt{4\tilde{\omega}_{K+1} + 
				8 L_F^2/\mu^2  + 32  L_F^2
				+  16 KD_{0} +  4K\sum_{k=1}^{K-1}\Big(20 A_k C_k^2+ 2 \tilde{\omega}_kC_k^2 + 8 D_{k}  C_{k}^2\Big)\Big(\prod_{j=1}^{k-1}(2C_{j}^2)\Big) }}  \ , \\
		& \alpha \leq  \frac{1}{4\sqrt{2  L_F^2/\mu^2  + 8  L_F^2
				+  4 KD_{0} +  K\sum_{k=1}^{K-1}\Big(20 A_k C_k^2+ 2 \tilde{\omega}_kC_k^2 + 8 D_{k}  C_{k}^2\Big)\Big(\prod_{j=1}^{k-1}(2C_{j}^2)\Big) }}  \ ,  \\
	\end{aligned}
\end{equation}
where  $\tilde{\omega}_k= \frac{2}{\beta} \Bigg((12A_k+8 D_{k})\mu  + 2  A_k+  2 \beta  \sum_{j=k+1}^{K-1}\Big( 20    A_j C_j^2+ 8 D_{j}  C_{j}^2\Big)  \Big(\prod_{i=k+1}^{j}(2C_{i}^2)\Big)  \Bigg)$ and $\tilde{\omega}_{K+1}= L_F^2+ 8\Bigg(\frac{2  L_F^2}{\mu^2  }+ 8  L_F^2
+  4 KD_{0} +  K\sum_{k=1}^{K-1}\Big(20 A_k C_k^2+ 2 \tilde{\omega}_kC_k^2 + 8 D_{k}  C_{k}^2\Big)\Big(\prod_{j=1}^{k-1}(2C_{j}^2)\Big) \Bigg) $, we can get
\begin{equation}
	\begin{aligned}
		& \quad \mathcal{H}_{t+1} - \mathcal{H}_{t} \\
		& \leq -\frac{\alpha\eta}{2}\|\nabla F(\bar{{x}}_{t})\|^2   +8 \alpha\beta^2  \mu^2\eta^4   K\sum_{k=1}^{K-1}D_{k}  \delta_{k}^2  +4\alpha\mu^2\eta^2 K\sum_{k=1}^{K}B_k \sigma_{k}^2 +16\alpha \mu^2 \beta^2\eta^4 K\sum_{k=1}^{K-1}A_k \delta_{k}^2 + 4\alpha \mu^3\eta^3 K\sum_{k=1}^{K}B_{k}\sigma_{k}^2 \\
		& \quad  +4\alpha \beta^2\eta^3 K\sum_{k=1}^{K-1}A_k \delta_{k}^2 + \alpha\mu \eta^2 K\sum_{k=1}^{K}B_{k}\frac{\sigma_{k}^2}{N} + 2\beta^2\eta^2\alpha K \sum_{k=1}^{K-1}\tilde{\omega}_k\delta_{k}^2\\
		& \quad  + 2\alpha\beta^2  \eta^2K \sum_{k=1}^{K-1}\Bigg[\sum_{j=k+1}^{K-1}\Big(20  A_j C_j^2+ 2 \tilde{\omega}_j C_j^2 + 8 D_{j}  C_{j}^2\Big)  \Big(\prod_{i=k+1}^{j}(2C_{i}^2)\Big)\Bigg]\delta_{k}^2 \ .  \\
	\end{aligned}
\end{equation}
Then, it is easy to get
\begin{equation}
	\begin{aligned}
		& \frac{1}{T}\sum_{t=0}^{T-1} \mathbb{E}[\|\nabla F(\bar{{x}}_{t})\|^2]   \\
		& \leq \frac{2(\mathcal{H}_{0}- \mathcal{H}_{T})}{\alpha\eta T} + 16 \beta^2  \mu^2\eta^3   K\sum_{k=1}^{K-1}D_{k}  \delta_{k}^2  +8\eta\mu^2 K\sum_{k=1}^{K}B_k \sigma_{k}^2 +32 \mu^2 \beta^2\eta^3 K\sum_{k=1}^{K-1}A_k \delta_{k}^2 + 8 \mu^3\eta^2 K\sum_{k=1}^{K}B_{k}\sigma_{k}^2 \\
		& \quad  +8 \beta^2\eta^2 K\sum_{k=1}^{K-1}A_k \delta_{k}^2 + 2\mu \eta K\sum_{k=1}^{K}B_{k}\frac{\sigma_{k}^2}{N} + 4\eta\beta^2 K \sum_{k=1}^{K-1}\tilde{\omega}_k\delta_{k}^2\\
		& \quad  + 4\eta\beta^2   K \sum_{k=1}^{K-1}\Bigg[\sum_{j=k+1}^{K-1}\Big(20  A_j C_j^2+ 2 \tilde{\omega}_j C_j^2 + 8 D_{j}  C_{j}^2\Big)  \Big(\prod_{i=k+1}^{j}(2C_{i}^2)\Big)\Bigg]\delta_{k}^2  \ . \\
	\end{aligned}
\end{equation}

According to the initial value, we can get
\begin{equation}
	\begin{aligned}
		& \quad \frac{1}{N} \mathbb{E}[\|Y_{0} - \bar{Y}_{0}\|_{F}^2 ]   \\
		& = \frac{1}{N}\sum_{n=1}^{N} \mathbb{E}[\| {v}_{n,  0}^{(1)} {v}_{ n, 0}^{(2)}\cdots {v}_{ n, 0}^{(K-1)} {v}_{n,  0}^{(K)} - \frac{1}{N}\sum_{n'=1}^{N}{v}_{n',  0}^{(1)} {v}_{ n', 0}^{(2)}\cdots {v}_{ n', 0}^{(K-1)} {v}_{n',  0}^{(K)} \|^2 ]   \\
		& = \frac{1}{N}\sum_{n=1}^{N} \mathbb{E}[\| {v}_{n,  0}^{(1)} {v}_{ n, 0}^{(2)}\cdots {v}_{ n, 0}^{(K-1)} {v}_{n,  0}^{(K)} - \nabla f_{n}^{(1)} (u_{n,0}^{(0)})\nabla  f_{n}^{(2)} (u_{n,0}^{(1)}) \cdots \nabla  f_{n}^{(K-1)} (u_{n,0}^{(K-2)})\nabla  f_{n}^{(K)} (u_{n,0}^{(K-1)}) \\
		& \quad + \nabla f_{n}^{(1)} (u_{n,0}^{(0)})\nabla  f_{n}^{(2)} (u_{n,0}^{(1)}) \cdots \nabla  f_{n}^{(K-1)} (u_{n,0}^{(K-2)})\nabla  f_{n}^{(K)} (u_{n,0}^{(K-1)}) \\
		& \quad - \frac{1}{N}\sum_{n'=1}^{N}\nabla f_{n'}^{(1)} (u_{n',0}^{(0)})\nabla  f_{n'}^{(2)} (u_{n',0}^{(1)}) \cdots \nabla  f_{n'}^{(K-1)} (u_{n',0}^{(K-2)})\nabla  f_{n'}^{(K)} (u_{n',0}^{(K-1)}) \\
		& \quad + \frac{1}{N}\sum_{n'=1}^{N}\nabla f_{n'}^{(1)} (u_{n',0}^{(0)})\nabla  f_{n'}^{(2)} (u_{n',0}^{(1)}) \cdots \nabla  f_{n'}^{(K-1)} (u_{n',0}^{(K-2)})\nabla  f_{n'}^{(K)} (u_{n',0}^{(K-1)}) \\
		& \quad  - \frac{1}{N}\sum_{n'=1}^{N}{v}_{n',  0}^{(1)} {v}_{ n', 0}^{(2)}\cdots {v}_{ n', 0}^{(K-1)} {v}_{n',  0}^{(K)} \|^2 ]   \\
		& \leq  3\frac{1}{N}\sum_{n=1}^{N} \mathbb{E}[\| {v}_{n,  0}^{(1)} {v}_{ n, 0}^{(2)}\cdots {v}_{ n, 0}^{(K-1)} {v}_{n,  0}^{(K)} - \nabla f_{n}^{(1)} (u_{n,0}^{(0)})\nabla  f_{n}^{(2)} (u_{n,0}^{(1)}) \cdots \nabla  f_{n}^{(K-1)} (u_{n,0}^{(K-2)})\nabla  f_{n}^{(K)} (u_{n,0}^{(K-1)})  \|^2 ] \\
		& \quad + 3\frac{1}{N}\sum_{n=1}^{N} \mathbb{E}[\|\nabla f_{n}^{(1)} (u_{n,0}^{(0)})\nabla  f_{n}^{(2)} (u_{n,0}^{(1)}) \cdots \nabla  f_{n}^{(K-1)} (u_{n,0}^{(K-2)})\nabla  f_{n}^{(K)} (u_{n,0}^{(K-1)}) \\
		& \quad \quad - \frac{1}{N}\sum_{n'=1}^{N}\nabla f_{n'}^{(1)} (u_{n',0}^{(0)})\nabla  f_{n'}^{(2)} (u_{n',0}^{(1)}) \cdots \nabla  f_{n'}^{(K-1)} (u_{n',0}^{(K-2)})\nabla  f_{n'}^{(K)} (u_{n',0}^{(K-1)}) \|^2 ] \\
		& \quad + 3\frac{1}{N}\sum_{n=1}^{N} \mathbb{E}[\|\frac{1}{N}\sum_{n'=1}^{N}\nabla f_{n'}^{(1)} (u_{n',0}^{(0)})\nabla  f_{n'}^{(2)} (u_{n',0}^{(1)}) \cdots \nabla  f_{n'}^{(K-1)} (u_{n',0}^{(K-2)})\nabla  f_{n'}^{(K)} (u_{n',0}^{(K-1)}) \\
		& \quad \quad  - \frac{1}{N}\sum_{n'=1}^{N}{v}_{n',  0}^{(1)} {v}_{ n', 0}^{(2)}\cdots {v}_{ n', 0}^{(K-1)} {v}_{n',  0}^{(K)} \|^2 ]   \\
		& \leq  6K\sum_{k=1}^{K}B_k \sigma_{k}^2  + 12K\sum_{k=2}^{K} \frac{(\prod_{j=1}^{K}C_j^2)L_k^2}{C_k^2} \sum_{i=1}^{k-1}8\delta_{i}^2\prod_{j=i+1}^{k-1} (8C_{j}^2) \ , 
	\end{aligned}
\end{equation}
where the last step holds due to the following inequality:

\begin{equation}
	\begin{aligned}
		& \quad \frac{1}{N}\sum_{n=1}^{N} \mathbb{E}[\|\nabla f_{n}^{(1)} (u_{n,0}^{(0)})\nabla  f_{n}^{(2)} (u_{n,0}^{(1)}) \cdots \nabla  f_{n}^{(K-1)} (u_{n,0}^{(K-2)})\nabla  f_{n}^{(K)} (u_{n,0}^{(K-1)}) \\
		& \quad \quad - \frac{1}{N}\sum_{n'=1}^{N}\nabla f_{n'}^{(1)} (u_{n',0}^{(0)})\nabla  f_{n'}^{(2)} (u_{n',0}^{(1)}) \cdots \nabla  f_{n'}^{(K-1)} (u_{n',0}^{(K-2)})\nabla  f_{n'}^{(K)} (u_{n',0}^{(K-1)}) \|^2 ] \\
		& \leq K\frac{1}{N}\sum_{n=1}^{N} \mathbb{E}[\|\Big(\nabla f_{n}^{(1)} (u_{n,0}^{(0)}) - \frac{1}{N}\sum_{n'=1}^{N}\nabla f_{n'}^{(1)} (u_{n',0}^{(0)})\Big)\nabla  f_{n}^{(2)} (u_{n,0}^{(1)}) \cdots \nabla  f_{n}^{(K-1)} (u_{n,0}^{(K-2)})\nabla  f_{n}^{(K)} (u_{n,0}^{(K-1)}) \\
		& \quad  + K\frac{1}{N}\sum_{n=1}^{N} \mathbb{E}[\|\frac{1}{N}\sum_{n'=1}^{N}\nabla f_{n'}^{(1)} (u_{n',0}^{(0)})\Big(\nabla  f_{n}^{(2)} (u_{n,0}^{(1)}) - \nabla  f_{n'}^{(2)} (u_{n',0}^{(1)}) \Big)\cdots \nabla  f_{n}^{(K-1)} (u_{n,0}^{(K-2)})\nabla  f_{n}^{(K)} (u_{n,0}^{(K-1)}) \|^2 ]\\
		&\quad  + \cdots \\
		& \quad  + K\frac{1}{N}\sum_{n=1}^{N} \mathbb{E}[\|\frac{1}{N}\sum_{n'=1}^{N}\nabla f_{n'}^{(1)} (u_{n',0}^{(0)})\nabla  f_{n'}^{(2)} (u_{n',0}^{(1)}) \cdots \nabla  f_{n'}^{(K-1)} (u_{n',0}^{(K-2)})\Big(\nabla  f_{n}^{(K)} (u_{n,0}^{(K-1)}) -\nabla  f_{n'}^{(K)} (u_{n',0}^{(K-1)}) \Big)\|^2 ] \\
		& \leq 0 + 4K\sum_{k=2}^{K}\frac{1}{N}\sum_{n=1}^{N} \frac{(\prod_{j=1}^{K}C_j^2)L_k^2}{C_k^2}\mathbb{E}[\|u_{n,0}^{(k-1)}- \bar{u}_{0}^{(k-1)}\|^2] \\
		& \leq 4K\sum_{k=2}^{K} \frac{(\prod_{j=1}^{K}C_j^2)L_k^2}{C_k^2} \sum_{i=1}^{k-1}8\delta_{i}^2\prod_{j=i+1}^{k-1} (8C_{j}^2) \ , \\
	\end{aligned}
\end{equation}
where the last step holds due to the following inequality:
\begin{equation}
	\begin{aligned}
		& \quad \frac{1}{N}\sum_{n=1}^{N}\mathbb{E}[\|u_{n,0}^{(k-1)}- \bar{u}_{0}^{(k-1)}\|^2 ]\\
		& = \frac{1}{N}\sum_{n=1}^{N}\mathbb{E}[\|f_{n}^{(k-1)}({u}_{n,  0}^{(k-2)}; \xi_{ n, t}^{(k-1)})- \frac{1}{N}\sum_{n'=1}^{N}f_{n'}^{(k-1)}({u}_{n',  0}^{(k-2)}; \xi_{ n', t}^{(k-1)})\|^2 ]\\
		& = \frac{1}{N}\sum_{n=1}^{N}\mathbb{E}[\|f_{n}^{(k-1)}({u}_{n,  0}^{(k-2)}; \xi_{ n, t}^{(k-1)}) - f_{n}^{(k-1)}({u}_{n,  0}^{(k-2)}) + f_{n}^{(k-1)}({u}_{n,  0}^{(k-2)}) - f^{(k-1)}(\bar{u}_{ 0}^{(k-2)}) \\
		& \quad + f^{(k-1)}(\bar{u}_{ 0}^{(k-2)})  - \frac{1}{N}\sum_{n'=1}^{N}f_{n'}^{(k-1)}({u}_{n',  0}^{(k-2)}) + \frac{1}{N}\sum_{n'=1}^{N}f_{n'}^{(k-1)}({u}_{n',  0}^{(k-2)})- \frac{1}{N}\sum_{n'=1}^{N}f_{n'}^{(k-1)}({u}_{n',  0}^{(k-2)}; \xi_{ n', t}^{(k-1)})\|^2 ]\\
		& \leq  8C_{k-1}^2\frac{1}{N}\sum_{n=1}^{N} \mathbb{E}[\|u_{n,0}^{(k-2)}- \bar{u}_{0}^{(k-2)}\|^2 ]+ 8\delta_{k-1}^2 \\
		& \leq \sum_{i=1}^{k-1}8\delta_{i}^2\prod_{j=i+1}^{k-1} (8C_{j}^2) \ . 
	\end{aligned}
\end{equation}

Moreover, we can get
\begin{equation}
	\begin{aligned}
		&  \mathbb{E}[\|u_{n, 0}^{(k)} -f_{n}^{(k)}(u_{n, 0}^{(k-1)}) \|^2] = \mathbb{E}[\|f_{n}^{(k)}({u}_{n,  0}^{(k-1)}; \xi_{ n, 0}^{(k)}) -f_{n}^{(k)}(u_{n, 0}^{(k-1)}) \|^2] \leq \delta_{k}^2   \ , \\
	\end{aligned}
\end{equation}
and 
\begin{equation}
	\begin{aligned}
		& \quad \mathbb{E}\Big[\Big\|\frac{1}{N}\sum_{n=1}^{N}{m}_{n,0} -\frac{1}{N}\sum_{n=1}^{N}\nabla F_{n}({{x}}_{n,0})\Big\|^2\Big]   \\
		& \leq  2\mathbb{E}\Big[\Big\|\frac{1}{N}\sum_{n=1}^{N}{v}_{n,  0}^{(1)} {v}_{ n, 0}^{(2)}\cdots {v}_{ n, 0}^{(K-1)} {v}_{n,  0}^{(K)} -\frac{1}{N}\sum_{n=1}^{N}\nabla f_{n}^{(1)} (u_{n,0}^{(0)})\nabla  f_{n}^{(2)} (u_{n,0}^{(1)}) \cdots \nabla  f_{n}^{(K-1)} (u_{n,0}^{(K-2)})\nabla  f_{n}^{(K)} (u_{n,0}^{(K-1)}) \Big\|^2\Big]\\
		& \quad +  2\mathbb{E}\Big[\Big\|\frac{1}{N}\sum_{n=1}^{N}\nabla f_{n}^{(1)} (u_{n,0}^{(0)})\nabla  f_{n}^{(2)} (u_{n,0}^{(1)}) \cdots \nabla  f_{n}^{(K-1)} (u_{n,0}^{(K-2)})\nabla  f_{n}^{(K)} (u_{n,0}^{(K-1)})\\
		& \quad \quad -\frac{1}{N}\sum_{n=1}^{N}\nabla f_{n}^{(1)} (x_{n,t})\nabla  f_{n}^{(2)} (F_{n}^{(1)}(x_{n,t})) \cdots \nabla  f_{n}^{(K-1)} (F_{n}^{(K-2)}(x_{n,t}))\nabla f_{n}^{(K)}(F_{n}^{(K-1)}(x_{n,t}))\Big\|^2\Big]  \\
		& \leq  2K\sum_{k=1}^{K}B_k \frac{\sigma_{k}^2 }{N} + 2\frac{K}{N}\sum_{n=1}^{N}\sum_{k=1}^{K-1}A_k\mathbb{E}[\| u_{n,0}^{(k)} -f_{n}^{(k)}( u_{n,0}^{(k-1)}) \|^2]  \\
		& \leq  2K\sum_{k=1}^{K}B_k \frac{\sigma_{k}^2 }{N} + 2K\sum_{k=1}^{K-1}A_k\delta_{k}^2  \ ,   \\
	\end{aligned}
\end{equation}
as well as $\mathbb{E}\Big[\Big\|{m}_{n,0} -\nabla F_{n}({{x}}_{n,0})\Big\|^2\Big] \leq 2K\sum_{k=1}^{K}B_k \sigma_{k}^2 + 2K\sum_{k=1}^{K-1}A_k\delta_{k}^2 $. 
Then, we can get
\begin{equation}
	\begin{aligned}
		& 	\mathcal{H}_{0} = F({{x}}_{0})  + \omega_0\frac{1}{N}\sum_{n=1}^{N}\mathbb{E}\Big[\Big\|{m}_{n,0} -\nabla F_{n}({{x}}_{n,0})\Big\|^2\Big]  + \frac{1}{N} \sum_{n=1}^{N}\sum_{k=1}^{K-1}\omega_k \mathbb{E}[\|u_{n, 0}^{(k)} -f_{n}^{(k)}(u_{n, 0}^{(k-1)}) \|^2]   \\
		& \quad + \omega_{K}\mathbb{E}\Big[\Big\|\frac{1}{N}\sum_{n=1}^{N}{m}_{n,0} -\frac{1}{N}\sum_{n=1}^{N}\nabla F_{n}({{x}}_{n,0})\Big\|^2\Big]   + \omega_{K+1}\frac{1}{N} \mathbb{E}[\|X_{0} - \bar{X}_{0}\|_{F}^2 ]+ \omega_{K+2}\frac{1}{N} \mathbb{E}[\|Y_{0} - \bar{Y}_{0}\|_{F}^2 ]   \\
		& \leq F({{x}}_{0}) + 8\alpha \mu\eta K(\sum_{k=1}^{K}B_k \sigma_{k}^2 + \sum_{k=1}^{K-1}A_k\delta_{k}^2)+\alpha K \sum_{k=1}^{K-1} \tilde{\omega}_k  \delta_{k}^2 \\
		& \quad +  \frac{2\alpha K}{\mu} (\sum_{k=1}^{K}B_k \frac{\sigma_{k}^2 }{N} + \sum_{k=1}^{K-1}A_k\delta_{k}^2 ) + 6\alpha K\Big( \sum_{k=1}^{K}B_k \sigma_{k}^2  + 2\sum_{k=2}^{K} \frac{(\prod_{j=1}^{K}C_j^2)L_k^2}{C_k^2} \sum_{i=1}^{k-1}8\delta_{i}^2\prod_{j=i+1}^{k-1} (8C_{j}^2)\Big)  \ . \\
	\end{aligned}
\end{equation}

Finally, we can get
\begin{equation}
	\begin{aligned}
		&\quad  \frac{1}{T}\sum_{t=0}^{T-1} \mathbb{E}[\|\nabla F(\bar{{x}}_{t})\|^2]   \\
& \leq \frac{2(F({{x}}_{0}) - F({{x}}_{*}))}{\alpha\eta T} + \frac{16 \mu K}{ T}(\sum_{k=1}^{K}B_k \sigma_{k}^2 + \sum_{k=1}^{K-1}A_k\delta_{k}^2)+\frac{2K}{\eta T}  \sum_{k=1}^{K-1} \tilde{\omega}_k  \delta_{k}^2 \\
& \quad +  \frac{4 K}{\mu \eta T} (\sum_{k=1}^{K}B_k \frac{\sigma_{k}^2 }{N} + \sum_{k=1}^{K-1}A_k\delta_{k}^2 ) + \frac{12K}{\eta T} \Big( \sum_{k=1}^{K}B_k \sigma_{k}^2  + 2\sum_{k=2}^{K} \frac{(\prod_{j=1}^{K}C_j^2)L_k^2}{C_k^2} \sum_{i=1}^{k-1}8\delta_{i}^2\prod_{j=i+1}^{k-1} (8C_{j}^2)\Big)  \\
& \quad + 16 \beta^2  \mu^2\eta^3   K\sum_{k=1}^{K-1}D_{k}  \delta_{k}^2  +8\eta\mu^2 K\sum_{k=1}^{K}B_k \sigma_{k}^2 +32 \mu^2 \beta^2\eta^3 K\sum_{k=1}^{K-1}A_k \delta_{k}^2 + 8 \mu^3\eta^2 K\sum_{k=1}^{K}B_{k}\sigma_{k}^2 \\
& \quad  +8 \beta^2\eta^2 K\sum_{k=1}^{K-1}A_k \delta_{k}^2 + 2\mu \eta K\sum_{k=1}^{K}B_{k}\frac{\sigma_{k}^2}{N} + 4\eta\beta^2 K \sum_{k=1}^{K-1}\tilde{\omega}_k\delta_{k}^2\\
& \quad  + 4\eta\beta^2   K \sum_{k=1}^{K-1}\Bigg[\sum_{j=k+1}^{K-1}\Big(20  A_j C_j^2+ 2 \tilde{\omega}_j C_j^2 + 8 D_{j}  C_{j}^2\Big)  \Big(\prod_{i=k+1}^{j}(2C_{i}^2)\Big)\Bigg]\delta_{k}^2 \ .  \\
	\end{aligned}
\end{equation}

\end{proof}

\subsection{Proof of Theorem \ref{theorem2}}

\begin{lemma} \label{lemma_g_inc_var}
	Given Assumptions~\ref{assumption_smooth}-\ref{assumption_bound_variance},  we can get 
	\begin{equation}
		\begin{aligned}
			&\quad  \mathbb{E}[\|g^{\xi_{t}}_{n,t} - g^{\xi_{t}}_{n, t-1}\|^2] \leq  KD_{0} \mathbb{E}[\|{x}_{n, t}- {x}_{n, t-1}\|^2]  + 2K\sum_{k=1}^{K-1}D_{k} C_{k}^2  \mathbb{E}[\|u_{n, t-1}^{(k-1)} - u_{n,t}^{(k-1)} \|^2]  \\
			& \quad + 2 \beta^2  \eta^4 K\sum_{k=1}^{K-1}D_{k}  \mathbb{E}[\|u_{n, t-1}^{(k)}- f_{n}^{(k)}(u_{n, t-1}^{(k-1)})  \|^2] + 2 \beta^2  \eta^4K\sum_{k=1}^{K-1}D_{k}  \delta_{k}^2 \ , \\
		\end{aligned}
	\end{equation}
	where $D_k = \frac{(\prod_{j=1}^{K} C_{i}^2) L_{k+1}^2}{C_{k+1}^2}$.
\end{lemma}

\begin{proof}
	\begin{equation}
		\begin{aligned}
			&\quad  \mathbb{E}[\|g^{\xi_{t}}_{n,t} - g^{\xi_{t}}_{n, t-1}\|^2] \\
			& =  \mathbb{E}[\|\nabla  f^{(1)}_{n}({u}_{n, t}^{(0)}; \xi_{n,  t}^{(1)})  \nabla  f^{(2)}_{n}({u}_{n, t}^{(1)}; \xi_{n,  t}^{(2)}) \cdots \nabla  f^{(K-1)}_{n}({u}_{n, t}^{(K-2)}; \xi_{n,  t}^{(K-1)})  \nabla  f^{(K)}_{n}({u}_{n, t}^{(K-1)}; \xi_{n,  t}^{(K)})  \\
			& \quad - \nabla  f^{(1)}_{n}({u}_{n, t-1}^{(0)}; \xi_{n,  t}^{(1)})  \nabla  f^{(2)}_{n}({u}_{n, t-1}^{(1)}; \xi_{n,  t}^{(2)}) \cdots \nabla  f^{(K-1)}_{n}({u}_{n, t-1}^{(K-2)}; \xi_{n,  t}^{(K-1)})  \nabla  f^{(K)}_{n}({u}_{n, t-1}^{(K-1)}; \xi_{n,  t}^{(K)})\|^2] \\
			& \leq K\mathbb{E}[\|\nabla  f^{(1)}_{n}({u}_{n, t}^{(0)}; \xi_{n,  t}^{(1)})  \nabla  f^{(2)}_{n}({u}_{n, t}^{(1)}; \xi_{n,  t}^{(2)}) \cdots \nabla  f^{(K-1)}_{n}({u}_{n, t}^{(K-2)}; \xi_{n,  t}^{(K-1)})  \nabla  f^{(K)}_{n}({u}_{n, t}^{(K-1)}; \xi_{n,  t}^{(K)})  \\
			& \quad \quad - \nabla  f^{(1)}_{n}({u}_{n, t-1}^{(0)}; \xi_{n,  t}^{(1)})  \nabla  f^{(2)}_{n}({u}_{n, t}^{(1)}; \xi_{n,  t}^{(2)}) \cdots \nabla  f^{(K-1)}_{n}({u}_{n, t}^{(K-2)}; \xi_{n,  t}^{(K-1)})  \nabla  f^{(K)}_{n}({u}_{n, t}^{(K-1)}; \xi_{n,  t}^{(K)}) \|^2]  \\
			& \quad + K\mathbb{E}[\|\nabla  f^{(1)}_{n}({u}_{n, t-1}^{(0)}; \xi_{n,  t}^{(1)})  \nabla  f^{(2)}_{n}({u}_{n, t}^{(1)}; \xi_{n,  t}^{(2)}) \cdots \nabla  f^{(K-1)}_{n}({u}_{n, t}^{(K-2)}; \xi_{n,  t}^{(K-1)})  \nabla  f^{(K)}_{n}({u}_{n, t}^{(K-1)}; \xi_{n,  t}^{(K)})   \\
			& \quad \quad - \nabla  f^{(1)}_{n}({u}_{n, t-1}^{(0)}; \xi_{n,  t}^{(1)})  \nabla  f^{(2)}_{n}({u}_{n, t-1}^{(1)}; \xi_{n,  t}^{(2)}) \cdots \nabla  f^{(K-1)}_{n}({u}_{n, t}^{(K-2)}; \xi_{n,  t}^{(K-1)})  \nabla  f^{(K)}_{n}({u}_{n, t}^{(K-1)}; \xi_{n,  t}^{(K)})   \|^2]\\
			& \quad \cdots \\
			& \quad  + K\mathbb{E}[\|\nabla  f^{(1)}_{n}({u}_{n, t}^{(0)}; \xi_{n,  t}^{(1)})  \nabla  f^{(2)}_{n}({u}_{n, t}^{(1)}; \xi_{n,  t}^{(2)}) \cdots \nabla  f^{(K-1)}_{n}({u}_{n, t}^{(K-2)}; \xi_{n,  t}^{(K-1)})  \nabla  f^{(K)}_{n}({u}_{n, t-1}^{(K-1)}; \xi_{n,  t}^{(K)}) \\
			& \quad\quad  - \nabla  f^{(1)}_{n}({u}_{n, t-1}^{(0)}; \xi_{n,  t}^{(1)})  \nabla  f^{(2)}_{n}({u}_{n, t-1}^{(1)}; \xi_{n,  t}^{(2)}) \cdots \nabla  f^{(K-1)}_{n}({u}_{n, t-1}^{(K-2)}; \xi_{n,  t}^{(K-1)})  \nabla  f^{(K)}_{n}({u}_{n, t-1}^{(K-1)}; \xi_{n,  t}^{(K)})\|^2] \\
			& \leq K\frac{(\prod_{j=1}^{K} C_{i}^2) L_1^2}{C_{1}^2} \mathbb{E}[\|{u}_{n, t}^{(0)} - {u}_{n, t-1}^{(0)}\|^2] + K\frac{(\prod_{j=1}^{K} C_{i}^2) L_2^2}{C_{2}^2} \mathbb{E}[\|{u}_{n, t}^{(1)} - {u}_{n, t-1}^{(1)}\|^2] \\
			& \quad + \cdots + K\frac{(\prod_{j=1}^{K} C_{i}^2) L_K^2}{C_{K}^2} \mathbb{E}[\|{u}_{n, t}^{(K-1)} - {u}_{n, t-1}^{(K-1)}\|^2] \\
			& = KD_{0} \mathbb{E}[\|{x}_{n, t}- {x}_{n, t-1}\|^2]  + K\sum_{k=1}^{K-1}D_{k}\mathbb{E}[\|{u}_{n, t}^{(k)} - {u}_{n, t-1}^{(k)}\|^2] \\
			& \leq  KD_{0} \mathbb{E}[\|{x}_{n, t}- {x}_{n, t-1}\|^2]  + 2K\sum_{k=1}^{K-1}D_{k} C_{k}^2  \mathbb{E}[\|u_{n, t-1}^{(k-1)} - u_{n,t}^{(k-1)} \|^2]  \\
			& \quad + 2 \beta^2  \eta^4 K\sum_{k=1}^{K-1}D_{k}  \mathbb{E}[\|u_{n, t-1}^{(k)}- f_{n}^{(k)}(u_{n, t-1}^{(k-1)})  \|^2] + 2 \beta^2  \eta^4K\sum_{k=1}^{K-1}D_{k}  \delta_{k}^2 \ , \\
		\end{aligned}
	\end{equation}
	where $D_k = \frac{(\prod_{j=1}^{K} C_{j}^2) L_{k+1}^2}{C_{k+1}^2}$, the third to last step holds due to Assumption~\ref{assumption_bound_gradient} and Assumption~\ref{assumption_smooth}, the last step holds due to Lemma~\ref{lemma_u_inc_var}.

\end{proof}

\begin{lemma} \label{lemma_u_F_var}
 For $k\in\{1, \cdots, K-1\}$,	given Assumptions~\ref{assumption_smooth}-\ref{assumption_bound_variance},  we can get 
	\begin{equation}
		\begin{aligned}
			& \quad \|u_{n,t}^{(k)} - F_{n}^{(k)}(x_{n,t}) \| \leq \sum_{j=1}^{k} \Big(\prod_{i=j+1}^{k}C_{i}\Big)\|u_{n,t}^{(j)} -f_{n}^{(j)}(u_{n,t}^{(j-1)}) \| \ , 
		\end{aligned}
	\end{equation} 
\end{lemma}
This lemma is the same as Lemma~\ref{lemma_u_F_momentum}.

\begin{lemma} \label{lemma_u_inc_var}
For $k\in\{2, \cdots, K\}$,		given Assumptions~\ref{assumption_smooth}-\ref{assumption_bound_variance},  we can get 
	\begin{equation}
		\begin{aligned}
			&  \mathbb{E}[\|u_{n,t}^{(k-1)}  - u_{n, t-1}^{(k-1)} \|^2]  \leq  2C_{k-1}^2 \mathbb{E}[\|u_{n, t-1}^{(k-2)} - u_{n,t}^{(k-2)} \|^2 ] + 2 \beta^2  \eta^4 \mathbb{E}[\|u_{n, t-1}^{(k-1)}- f_{n}^{(k-1)}(u_{n, t-1}^{(k-2)})  \|^2] + 2 \beta^2  \eta^4 \delta_{k-1}^2 \ , 
		\end{aligned}
	\end{equation}
	and
	\begin{equation}
		\begin{aligned}
			& \quad \mathbb{E}[\|u_{n,t}^{(k-1)}  - u_{n, t-1}^{(k-1)} \|^2]  \leq \Big(\prod_{j=1}^{k-1}(2C_{j}^2)\Big)\mathbb{E}[\|u_{n, t-1}^{(0)} - u_{n,t}^{(0)} \|^2]   + 2\beta^2  \eta^4\sum_{j=1}^{k-1} \Big(\prod_{i=j+1}^{k-1}(2C_{i}^2)\Big)\mathbb{E}[\|u_{n, t-1}^{(j)}- f_{n}^{(j)}(u_{n, t-1}^{(j-1)})  \|^2] \\
			& \quad \quad\quad\quad\quad \quad \quad\quad\quad\quad + 2\beta^2  \eta^4\sum_{j=1}^{k-1} \Big(\prod_{i=j+1}^{k-1}(2C_{i}^2)\Big)\delta_{j}^2 \ . 
		\end{aligned}
	\end{equation}
\end{lemma}
This lemma can be proved by following Lemma~\ref{lemma_u_inc_momentum} through replacing $\eta$ with $\eta^2$.

\begin{lemma}\label{lemma_f_u_f_x_var}
		Given Assumptions~\ref{assumption_smooth}-\ref{assumption_bound_variance},  we can get 
	\begin{equation}
		\begin{aligned}
			&\quad   \frac{1}{N}\sum_{n=1}^{N}\|\nabla f_{n}^{(1)} ( u_{n,t}^{(0)})\nabla  f_{n}^{(2)} ( u_{n,t}^{(1)}) \cdots \nabla  f_{n}^{(K-1)} ( u_{n,t}^{(K-2)})\nabla f_{n}^{(K)}( u_{n,t}^{(K-1)})  \\
			& \quad  - \nabla f_{n}^{(1)} (x_{n,t})\nabla  f_{n}^{(2)} (F_{n}^{(1)}(x_{n,t})) \cdots \nabla  f_{n}^{(K-1)} (F_{n}^{(K-2)}(x_{n,t}))\nabla f_{n}^{(K)}(F_{n}^{(K-1)}(x_{n,t})) \|^2 \\
			& \leq  \frac{K}{N}\sum_{n=1}^{N}\sum_{k=1}^{K-1}A_k\| u_{n,t}^{(k)} -f_{n}^{(k)}( u_{n,t}^{(k-1)}) \|^2  \ ,  \\
		\end{aligned}
	\end{equation}
	where $A_k= \Bigg(\sum_{j=k}^{K-1}\Big(\frac{L_{j+1}\prod_{i=1}^{K}C_i}{C_{j+1}} \prod_{i=k+1}^{j}C_{i}\Big)\Bigg)^2 $ . 
\end{lemma}
This lemma is the same as Lemma~\ref{lemma_f_u_f_x_momentum}

\begin{lemma} \label{lemma_f_u_inc_var}
	Given Assumptions~\ref{assumption_smooth}-\ref{assumption_bound_variance}, we can get 
	\begin{equation}
		\begin{aligned}
			& \quad \mathbb{E}\Big[\Big\|  \nabla f_{n}^{(1)} ( u_{n,t-1}^{(0)})\nabla  f_{n}^{(2)} ( u_{n,t-1}^{(1)}) \cdots \nabla  f_{n}^{(K-1)} ( u_{n,t-1}^{(K-2)})\nabla f_{n}^{(K)}( u_{n,t-1}^{(K-1)})\\
			& \quad  \quad   - \nabla f_{n}^{(1)} ( u_{n,t}^{(0)})\nabla  f_{n}^{(2)} ( u_{n,t}^{(1)}) \cdots \nabla  f_{n}^{(K-1)} ( u_{n,t}^{(K-2)})\nabla f_{n}^{(K)}( u_{n,t}^{(K-1)}) \Big\|^2\Big]  \\
			& \leq KD_{0}\|x_{n,t}-x_{n,t-1}\|^2  +  2K\sum_{k=1}^{K-1}D_{k}  C_{k}^2 \|u_{n, t-1}^{(k-1)} - u_{n,t}^{(k-1)} \|^2 \\
			& \quad  + 2 \beta^2  \eta^4 K\sum_{k=1}^{K-1}D_{k}  \|u_{n, t-1}^{(k)}- f_{n}^{(k)}(u_{n, t-1}^{(k-1)})  \|^2 + 2 \beta^2  \eta^4 K\sum_{k=1}^{K-1}D_{k}  \delta_{k}^2 \ , \\
		\end{aligned}
	\end{equation}
	where $D_{k}=\frac{(\prod_{j=1}^{K} C_{j}^2)L_{k+1}^2}{C_{k+1}^2}$. 
	
\end{lemma}
This lemma can be proved by following Lemma \ref{lemma_f_u_inc_momentum}  through replacing $\eta$ with $\eta^2$.

\begin{lemma}\label{lemma_g_var_var}
	Given Assumptions~\ref{assumption_smooth}-\ref{assumption_bound_variance}, we can get 
	\begin{equation}
		\begin{aligned}
			& \quad  \mathbb{E}\Big[\Big\|\frac{1}{N} \sum_{n=1}^{N} ( g^{\xi_t}_{n, t}-  \nabla f_{n}^{(1)} ( u_{n,t}^{(0)})\nabla  f_{n}^{(2)} ( u_{n,t}^{(1)}) \cdots \nabla  f_{n}^{(K-1)} ( u_{n,t}^{(K-2)})\nabla f_{n}^{(K)}( u_{n,t}^{(K-1)})) \Big\|^2\Big] \leq K\sum_{k=1}^{K}B_k \frac{\sigma_{k}^2 }{N} \ , \\
		\end{aligned}
	\end{equation}
	where $B_{k} = \frac{\prod_{j=1}^{K}C_j^2 }{C_k^2} $.
\end{lemma}
This lemma is the same as Lemma~\ref{lemma_g_var_momentum}.

\begin{lemma}\label{lemma_g_var_var2}
	Given Assumptions~\ref{assumption_smooth}-\ref{assumption_bound_variance}, we can get 
	\begin{equation}
		\begin{aligned}
			& \quad  \mathbb{E}\Big[\Big\| g^{\xi_t}_{n, t}-  \nabla f_{n}^{(1)} ( u_{n,t}^{(0)})\nabla  f_{n}^{(2)} ( u_{n,t}^{(1)}) \cdots \nabla  f_{n}^{(K-1)} ( u_{n,t}^{(K-2)})\nabla f_{n}^{(K)}( u_{n,t}^{(K-1)}) \Big\|^2\Big] \leq  K\sum_{k=1}^{K}B_k \sigma_{k}^2 \ ,  \\
		\end{aligned}
	\end{equation}
	where $B_{k} = \frac{\prod_{j=1}^{K}C_j^2 }{C_k^2} $.
\end{lemma}
This lemma is the same as Lemma~\ref{lemma_g_var_momentum2}.

\begin{lemma} \label{lemma_u_var_var}
	For $k\in\{1, \cdots, K-1\}$, given Assumptions~\ref{assumption_smooth}-\ref{assumption_bound_variance}, we can get 
	\begin{equation}
		\begin{aligned}
			&  \mathbb{E}[ \|u_{n, t}^{(k)} - f_{n}^{(k)}(u_{n, t}^{(k-1)}) \|^2] \leq (1-\beta\eta^2) \mathbb{E}[\|u_{n, t-1}^{(k)} -f_{n}^{(k)}(u_{n, t-1}^{(k-1)}) \|^2]  +2 C_k^2 \mathbb{E}[\|u_{n, t-1}^{(k-1)}- u_{n,t}^{(k-1)} \|^2 ]+ 2\beta^2\eta^4\delta_{k}^2 \ .  \\
		\end{aligned}
	\end{equation}
\end{lemma}
This lemma can be proved by following Lemma~\ref{lemma_u_var_momentum} through replacing $\eta$ with $\eta^2$. 

\begin{lemma} \label{lemma_x_consensus_var}
		Given Assumptions~\ref{assumption_graph}-\ref{assumption_bound_variance}, we can get 
	\begin{equation}
		\begin{aligned}
			& \quad \mathbb{E}[\|X_{t+1} - \bar{X}_{t+1}\|_{F}^2] \leq  (1-\eta\frac{1-\lambda^2}{2})\mathbb{E}[\|X_{t}  - \bar{X}_{t} \|_F^2]  + \frac{2\eta\alpha^2}{1-\lambda^2} \mathbb{E}[\| Y_{t} -  \bar{Y}_{t} \|_F^2] \ . \\
		\end{aligned}
	\end{equation}
\end{lemma}
This lemma is the same as Lemma~\ref{lemma_x_consensus_momentum}.

\begin{lemma}\label{lemma_m_var_var}
	Given Assumptions~\ref{assumption_smooth}-\ref{assumption_bound_variance} and $\mu\eta^2 \in (0, 1)$, we can get 
	\begin{equation}
		\begin{aligned}
			& \quad \mathbb{E}\Big[\Big\|\bar{m}_t - \frac{1}{N} \sum_{n=1}^{N}\nabla f_{n}^{(1)} ( u_{n,t}^{(0)})\nabla  f_{n}^{(2)} ( u_{n,t}^{(1)}) \cdots \nabla  f_{n}^{(K-1)} ( u_{n,t}^{(K-2)})\nabla f_{n}^{(K)}( u_{n,t}^{(K-1)})\Big\|^2\Big]   \\
			& \leq (1-\mu\eta^2 ) \mathbb{E}\Big[\Big\|\bar{m}_{t-1} - \frac{1}{N} \sum_{n=1}^{N}\nabla f_{n}^{(1)} ( u_{n,t-1}^{(0)})\nabla  f_{n}^{(2)} ( u_{n,t-1}^{(1)}) \cdots \nabla  f_{n}^{(K-1)} ( u_{n,t-1}^{(K-2)})\nabla f_{n}^{(K)}( u_{n,t-1}^{(K-1)}) \Big\|^2 \Big] \\
			& \quad + 2KD_{0}\frac{1}{N^2} \sum_{n=1}^{N}  \mathbb{E}[\|{x}_{n, t}- {x}_{n, t-1}\|^2]  + 4K\frac{1}{N^2} \sum_{n=1}^{N}\sum_{k=1}^{K-1}D_{k} C_{k}^2  \mathbb{E}[\|u_{n, t-1}^{(k-1)} - u_{n,t}^{(k-1)} \|^2]  \\
			& \quad + 4 \beta^2  \eta^4 K\frac{1}{N^2} \sum_{n=1}^{N}\sum_{k=1}^{K-1}D_{k}  \mathbb{E}[\|u_{n, t-1}^{(k)}- f_{n}^{(k)}(u_{n, t-1}^{(k-1)})  \|^2] + 4 \beta^2  \eta^4K  \sum_{k=1}^{K-1}D_{k} \frac{\delta_{k}^2}{N}  + 2\mu^2\eta^4 K\sum_{k=1}^{K}B_k \frac{\sigma_{k}^2 }{N} \ . \\
		\end{aligned}
	\end{equation}
	
\end{lemma}

\begin{proof}

	\begin{equation}
		\begin{aligned}
			& \quad \mathbb{E}\Big[\Big\|\bar{m}_t - \frac{1}{N} \sum_{n=1}^{N}\nabla f_{n}^{(1)} ( u_{n,t}^{(0)})\nabla  f_{n}^{(2)} ( u_{n,t}^{(1)}) \cdots \nabla  f_{n}^{(K-1)} ( u_{n,t}^{(K-2)})\nabla f_{n}^{(K)}( u_{n,t}^{(K-1)})\Big\|^2\Big]   \\
			& = \mathbb{E}\Big[\Big\|\frac{1}{N} \sum_{n=1}^{N} \Big((1-\mu\eta^2 )({m}_{n,t-1} -\nabla f_{n}^{(1)} ( u_{n,t-1}^{(0)})\nabla  f_{n}^{(2)} ( u_{n,t-1}^{(1)}) \cdots \nabla  f_{n}^{(K-1)} ( u_{n,t-1}^{(K-2)})\nabla f_{n}^{(K)}( u_{n,t-1}^{(K-1)})) \\
			& \quad + (\nabla f_{n}^{(1)} ( u_{n,t-1}^{(0)})\nabla  f_{n}^{(2)} ( u_{n,t-1}^{(1)}) \cdots \nabla  f_{n}^{(K-1)} ( u_{n,t-1}^{(K-2)})\nabla f_{n}^{(K)}( u_{n,t-1}^{(K-1)})\\
			& \quad \quad   - \nabla f_{n}^{(1)} ( u_{n,t}^{(0)})\nabla  f_{n}^{(2)} ( u_{n,t}^{(1)}) \cdots \nabla  f_{n}^{(K-1)} ( u_{n,t}^{(K-2)})\nabla f_{n}^{(K)}( u_{n,t}^{(K-1)}) - g_{n, t-1} + g_{n, t})  \\
			& \quad + \mu\eta^2 ( g_{n, t-1}-  \nabla f_{n}^{(1)} ( u_{n,t-1}^{(0)})\nabla  f_{n}^{(2)} ( u_{n,t-1}^{(1)}) \cdots \nabla  f_{n}^{(K-1)} ( u_{n,t-1}^{(K-2)})\nabla f_{n}^{(K)}( u_{n,t-1}^{(K-1)}) ) \Big)\Big\|^2 \Big] \\
			& =(1-\mu\eta^2 )^2 \mathbb{E}\Big[\Big\|\frac{1}{N} \sum_{n=1}^{N} {m}_{n,t-1} - \frac{1}{N} \sum_{n=1}^{N}\nabla f_{n}^{(1)} ( u_{n,t-1}^{(0)})\nabla  f_{n}^{(2)} ( u_{n,t-1}^{(1)}) \cdots \nabla  f_{n}^{(K-1)} ( u_{n,t-1}^{(K-2)})\nabla f_{n}^{(K)}( u_{n,t-1}^{(K-1)}) \Big\|^2 \Big] \\
			& \quad + 2\mathbb{E}\Big[\Big\|\frac{1}{N} \sum_{n=1}^{N} \Big(\nabla f_{n}^{(1)} ( u_{n,t-1}^{(0)})\nabla  f_{n}^{(2)} ( u_{n,t-1}^{(1)}) \cdots \nabla  f_{n}^{(K-1)} ( u_{n,t-1}^{(K-2)})\nabla f_{n}^{(K)}( u_{n,t-1}^{(K-1)})\\
			& \quad \quad   - \nabla f_{n}^{(1)} ( u_{n,t}^{(0)})\nabla  f_{n}^{(2)} ( u_{n,t}^{(1)}) \cdots \nabla  f_{n}^{(K-1)} ( u_{n,t}^{(K-2)})\nabla f_{n}^{(K)}( u_{n,t}^{(K-1)}) - g_{n, t-1} + g_{n, t}  \Big)\Big\|^2 \Big] \\
			& \quad + 2\mu^2\eta^4 \mathbb{E}\Big[\Big\|\frac{1}{N} \sum_{n=1}^{N} \Big(g_{n, t-1}-  \nabla f_{n}^{(1)} ( u_{n,t-1}^{(0)})\nabla  f_{n}^{(2)} ( u_{n,t-1}^{(1)}) \cdots \nabla  f_{n}^{(K-1)} ( u_{n,t-1}^{(K-2)})\nabla f_{n}^{(K)}( u_{n,t-1}^{(K-1)})  \Big)\Big\|^2 \Big] \\
			& \leq (1-\mu\eta^2 )^2 \mathbb{E}\Big[\Big\|\frac{1}{N} \sum_{n=1}^{N} {m}_{n,t-1} - \frac{1}{N} \sum_{n=1}^{N}\nabla f_{n}^{(1)} ( u_{n,t-1}^{(0)})\nabla  f_{n}^{(2)} ( u_{n,t-1}^{(1)}) \cdots \nabla  f_{n}^{(K-1)} ( u_{n,t-1}^{(K-2)})\nabla f_{n}^{(K)}( u_{n,t-1}^{(K-1)}) \Big\|^2 \Big] \\
			& \quad + 2\frac{1}{N^2} \sum_{n=1}^{N} \mathbb{E}\Big[\Big\|g_{n, t}  - g_{n, t-1}  \Big\|^2 \Big] \\
			& \quad + 2\mu^2\eta^4 \mathbb{E}\Big[\Big\|\frac{1}{N} \sum_{n=1}^{N} \Big(g_{n, t-1}-  \nabla f_{n}^{(1)} ( u_{n,t-1}^{(0)})\nabla  f_{n}^{(2)} ( u_{n,t-1}^{(1)}) \cdots \nabla  f_{n}^{(K-1)} ( u_{n,t-1}^{(K-2)})\nabla f_{n}^{(K)}( u_{n,t-1}^{(K-1)})  \Big)\Big\|^2 \Big] \\
			& \leq (1-\mu\eta^2 ) \mathbb{E}\Big[\Big\|\bar{m}_{t-1} - \frac{1}{N} \sum_{n=1}^{N}\nabla f_{n}^{(1)} ( u_{n,t-1}^{(0)})\nabla  f_{n}^{(2)} ( u_{n,t-1}^{(1)}) \cdots \nabla  f_{n}^{(K-1)} ( u_{n,t-1}^{(K-2)})\nabla f_{n}^{(K)}( u_{n,t-1}^{(K-1)}) \Big\|^2 \Big] \\
			& \quad + 2KD_{0}\frac{1}{N^2} \sum_{n=1}^{N}  \mathbb{E}[\|{x}_{n, t}- {x}_{n, t-1}\|^2]  + 4K\frac{1}{N^2} \sum_{n=1}^{N}\sum_{k=1}^{K-1}D_{k} C_{k}^2  \mathbb{E}[\|u_{n, t-1}^{(k-1)} - u_{n,t}^{(k-1)} \|^2]  \\
			& \quad + 4 \beta^2  \eta^4 K\frac{1}{N^2} \sum_{n=1}^{N}\sum_{k=1}^{K-1}D_{k}  \mathbb{E}[\|u_{n, t-1}^{(k)}- f_{n}^{(k)}(u_{n, t-1}^{(k-1)})  \|^2] + 4 \beta^2  \eta^4K  \sum_{k=1}^{K-1}D_{k} \frac{\delta_{k}^2}{N}  + 2\mu^2\eta^4 K\sum_{k=1}^{K}B_k \frac{\sigma_{k}^2 }{N}  \ . \\
		\end{aligned}
	\end{equation}

\end{proof}

\begin{lemma}\label{lemma_m_var_var2}
	
	Given Assumptions~\ref{assumption_smooth}-\ref{assumption_bound_variance} and $\mu\eta^2 \in (0, 1)$,  we can get 
	\begin{equation}
		\begin{aligned}
			& \quad \mathbb{E}\Big[\Big\|{m}_{n,t} - \nabla f_{n}^{(1)} ( u_{n,t}^{(0)})\nabla  f_{n}^{(2)} ( u_{n,t}^{(1)}) \cdots \nabla  f_{n}^{(K-1)} ( u_{n,t}^{(K-2)})\nabla f_{n}^{(K)}( u_{n,t}^{(K-1)})\Big\|^2\Big]   \\
			& \leq (1-\mu\eta^2 ) \mathbb{E}\Big[\Big\|{m}_{n, t-1} - \nabla f_{n}^{(1)} ( u_{n,t-1}^{(0)})\nabla  f_{n}^{(2)} ( u_{n,t-1}^{(1)}) \cdots \nabla  f_{n}^{(K-1)} ( u_{n,t-1}^{(K-2)})\nabla f_{n}^{(K)}( u_{n,t-1}^{(K-1)}) \Big\|^2 \Big] \\
			& \quad + 2KD_{0} \mathbb{E}[\|{x}_{n, t}- {x}_{n, t-1}\|^2]  + 4K\sum_{k=1}^{K-1}D_{k} C_{k}^2  \mathbb{E}[\|u_{n, t-1}^{(k-1)} - u_{n,t}^{(k-1)} \|^2]  \\
			& \quad + 4 \beta^2  \eta^4 K\sum_{k=1}^{K-1}D_{k}  \mathbb{E}[\|u_{n, t-1}^{(k)}- f_{n}^{(k)}(u_{n, t-1}^{(k-1)})  \|^2] + 4 \beta^2  \eta^4K  \sum_{k=1}^{K-1}D_{k} \delta_{k}^2 + 2\mu^2\eta^4 K\sum_{k=1}^{K}B_k \sigma_{k}^2  \ . \\
		\end{aligned}
	\end{equation}
	
\end{lemma}
This lemma is easy to prove by following Lemma~\ref{lemma_m_var_var}.

\begin{lemma} \label{lemma_y_consensus_var}
		Given Assumptions~\ref{assumption_graph}-\ref{assumption_bound_variance}, we can get 
	\begin{equation}
		\begin{aligned}
			&\quad  \mathbb{E}[\|Y_{t+1}-   \bar{Y}_{t+1}\|_F^2] \\
			& \leq \lambda\mathbb{E}[\|Y_{t}-   \bar{Y}_{t}\|_F^2] + \frac{2\mu^2\eta^4}{1-\lambda} \sum_{n=1}^{N} \mathbb{E}[\|m_{n, t}  - \nabla f_{n}^{(1)} ( u_{n,t}^{(0)})\nabla  f_{n}^{(2)} ( u_{n,t}^{(1)}) \cdots \nabla  f_{n}^{(K-1)} ( u_{n,t}^{(K-2)})\nabla f_{n}^{(K)}( u_{n,t}^{(K-1)}) \|^2 ]\\
			& \quad + \frac{2KD_{0}}{1-\lambda} \sum_{n=1}^{N}  \mathbb{E}[\|{x}_{n, t+1}- {x}_{n, t}\|^2]  + \frac{4K}{1-\lambda}\sum_{n=1}^{N}\sum_{k=1}^{K-1}D_{k} C_{k}^2  \mathbb{E}[\|u_{n, t}^{(k-1)} - u_{n,t+1}^{(k-1)} \|^2]  \\
			& \quad + \frac{4 \beta^2  \eta^4 K}{1-\lambda}\sum_{n=1}^{N}\sum_{k=1}^{K-1}D_{k}  \mathbb{E}[\|u_{n, t}^{(k)}- f_{n}^{(k)}(u_{n, t}^{(k-1)})  \|^2] + \frac{4 \beta^2  \eta^4KN}{1-\lambda}\sum_{k=1}^{K-1}D_{k}  \delta_{k}^2 + \frac{\mu^2\eta^4KN}{1-\lambda} \sum_{k=1}^{K}B_k \sigma_{k}^2  \ . \\
		\end{aligned}
	\end{equation}
	
\end{lemma}

\begin{proof}
	\begin{equation}
		\begin{aligned}
			& \quad \mathbb{E}[\|M_{t+1} - M_{t}\|_F^2]= \sum_{n=1}^{N} \mathbb{E}[\|m_{n, t+1} - m_{n, t}\|^2]  = \sum_{n=1}^{N} \mathbb{E}[\|(1-\mu\eta^2)(m_{n, t} -  g^{\xi_{t+1}}_{n, t}) + g^{\xi_{t+1}}_{n, t+1}- m_{n, t}\|^2] \\
			& = \sum_{n=1}^{N} \mathbb{E}[\|g^{\xi_{t+1}}_{n, t+1} -  g^{\xi_{t+1}}_{n, t} -\mu\eta^2(m_{n, t}  - \nabla f_{n}^{(1)} ( u_{n,t}^{(0)})\nabla  f_{n}^{(2)} ( u_{n,t}^{(1)}) \cdots \nabla  f_{n}^{(K-1)} ( u_{n,t}^{(K-2)})\nabla f_{n}^{(K)}( u_{n,t}^{(K-1)}) )\\
			& \quad  -\mu\eta^2(\nabla f_{n}^{(1)} ( u_{n,t}^{(0)})\nabla  f_{n}^{(2)} ( u_{n,t}^{(1)}) \cdots \nabla  f_{n}^{(K-1)} ( u_{n,t}^{(K-2)})\nabla f_{n}^{(K)}( u_{n,t}^{(K-1)}) -  g^{\xi_{t+1}}_{n, t})\|^2 ]\\
			& = \sum_{n=1}^{N} \mathbb{E}[\|g^{\xi_{t+1}}_{n, t+1} -  g^{\xi_{t+1}}_{n, t} -\mu\eta^2(m_{n, t}  - \nabla f_{n}^{(1)} ( u_{n,t}^{(0)})\nabla  f_{n}^{(2)} ( u_{n,t}^{(1)}) \cdots \nabla  f_{n}^{(K-1)} ( u_{n,t}^{(K-2)})\nabla f_{n}^{(K)}( u_{n,t}^{(K-1)}) )\|^2] \\
			& \quad  + \mu^2\eta^4\sum_{n=1}^{N} \mathbb{E}[\|\nabla f_{n}^{(1)} ( u_{n,t}^{(0)})\nabla  f_{n}^{(2)} ( u_{n,t}^{(1)}) \cdots \nabla  f_{n}^{(K-1)} ( u_{n,t}^{(K-2)})\nabla f_{n}^{(K)}( u_{n,t}^{(K-1)}) -  g^{\xi_{t+1}}_{n, t}\|^2] \\
			& \leq 2\sum_{n=1}^{N} \mathbb{E}[\|g^{\xi_{t+1}}_{n, t+1} -  g^{\xi_{t+1}}_{n, t}\|^2 ] + 2\mu^2\eta^4\sum_{n=1}^{N} \mathbb{E}[\|m_{n, t}  - \nabla f_{n}^{(1)} ( u_{n,t}^{(0)})\nabla  f_{n}^{(2)} ( u_{n,t}^{(1)}) \cdots \nabla  f_{n}^{(K-1)} ( u_{n,t}^{(K-2)})\nabla f_{n}^{(K)}( u_{n,t}^{(K-1)}) \|^2 ]\\
			& \quad  + \mu^2\eta^4\sum_{n=1}^{N} \mathbb{E}[\|\nabla f_{n}^{(1)} ( u_{n,t}^{(0)})\nabla  f_{n}^{(2)} ( u_{n,t}^{(1)}) \cdots \nabla  f_{n}^{(K-1)} ( u_{n,t}^{(K-2)})\nabla f_{n}^{(K)}( u_{n,t}^{(K-1)}) -  g^{\xi_{t+1}}_{n, t}\|^2 ]\\
			& \leq  2\mu^2\eta^4\sum_{n=1}^{N} \mathbb{E}[\|m_{n, t}  - \nabla f_{n}^{(1)} ( u_{n,t}^{(0)})\nabla  f_{n}^{(2)} ( u_{n,t}^{(1)}) \cdots \nabla  f_{n}^{(K-1)} ( u_{n,t}^{(K-2)})\nabla f_{n}^{(K)}( u_{n,t}^{(K-1)}) \|^2 ]\\
			& \quad + 2KD_{0} \sum_{n=1}^{N}  \mathbb{E}[\|{x}_{n, t+1}- {x}_{n, t}\|^2]  + 4K\sum_{n=1}^{N}\sum_{k=1}^{K-1}D_{k} C_{k}^2  \mathbb{E}[\|u_{n, t}^{(k-1)} - u_{n,t+1}^{(k-1)} \|^2]  \\
			& \quad + 4 \beta^2  \eta^4 K\sum_{n=1}^{N}\sum_{k=1}^{K-1}D_{k}  \mathbb{E}[\|u_{n, t}^{(k)}- f_{n}^{(k)}(u_{n, t}^{(k-1)})  \|^2] + 4 \beta^2  \eta^4KN\sum_{k=1}^{K-1}D_{k}  \delta_{k}^2 + \mu^2\eta^4KN \sum_{k=1}^{K}B_k \sigma_{k}^2  \ . \\
		\end{aligned}
	\end{equation}
	Then, we can get
	\begin{equation}
		\begin{aligned}
			&\quad  \mathbb{E}[\|Y_{t+1}-   \bar{Y}_{t+1}\|_F^2]\leq \lambda\mathbb{E}[\|Y_{t}-   \bar{Y}_{t}\|_F^2] + \frac{1}{1-\lambda} \mathbb{E}[\|M_{t+1} - M_{t}\|_F^2]\\
						& \leq \lambda\mathbb{E}[\|Y_{t}-   \bar{Y}_{t}\|_F^2] + \frac{2\mu^2\eta^4}{1-\lambda} \sum_{n=1}^{N} \mathbb{E}[\|m_{n, t}  - \nabla f_{n}^{(1)} ( u_{n,t}^{(0)})\nabla  f_{n}^{(2)} ( u_{n,t}^{(1)}) \cdots \nabla  f_{n}^{(K-1)} ( u_{n,t}^{(K-2)})\nabla f_{n}^{(K)}( u_{n,t}^{(K-1)}) \|^2 ]\\
						& \quad + \frac{2KD_{0}}{1-\lambda} \sum_{n=1}^{N}  \mathbb{E}[\|{x}_{n, t+1}- {x}_{n, t}\|^2]  + \frac{4K}{1-\lambda}\sum_{n=1}^{N}\sum_{k=1}^{K-1}D_{k} C_{k}^2  \mathbb{E}[\|u_{n, t}^{(k-1)} - u_{n,t+1}^{(k-1)} \|^2]  \\
						& \quad + \frac{4 \beta^2  \eta^4 K}{1-\lambda}\sum_{n=1}^{N}\sum_{k=1}^{K-1}D_{k}  \mathbb{E}[\|u_{n, t}^{(k)}- f_{n}^{(k)}(u_{n, t}^{(k-1)})  \|^2] + \frac{4 \beta^2  \eta^4KN}{1-\lambda}\sum_{k=1}^{K-1}D_{k}  \delta_{k}^2 + \frac{\mu^2\eta^4KN}{1-\lambda} \sum_{k=1}^{K}B_k \sigma_{k}^2  \ . \\
		\end{aligned}
	\end{equation}

\end{proof}

Based on these lemmas, we begin to prove Theorem~\ref{theorem2}. 

\begin{proof}
	At first, we can get
	\begin{equation}
		\begin{aligned}
			& F\left(\bar{x}_{t+1}\right) \leq F(\bar{x}_{t})+\langle \nabla F(\bar{x}_{t}),  \bar{x}_{t+1}-\bar{x}_{t} \rangle+\frac{L_{F}}{2}\|\bar{x}_{t+1}-\bar{x}_{t}\|^{2} \\
			& =  F(\bar{{x}}_{t})-\frac{\alpha\eta}{2}\|\nabla F(\bar{{x}}_{t})\|^2 -(\frac{\alpha\eta}{2}-\frac{\alpha^2\eta^2L_{F}}{2})\|\bar{m}_t\|^{2}  + \frac{\alpha\eta}{2} \|\bar{m}_t - \nabla F(\bar{{x}}_{t})\|^2\\
			& \leq F(\bar{{x}}_{t})-\frac{\alpha\eta}{2}\|\nabla F(\bar{{x}}_{t})\|^2 -\frac{\alpha\eta}{4}\|\bar{m}_t\|^{2}  + \frac{\alpha\eta}{2} \|\bar{m}_t - \nabla F(\bar{{x}}_{t})\|^2\\
			& \leq F(\bar{{x}}_{t})-\frac{\alpha\eta}{2}\|\nabla F(\bar{{x}}_{t})\|^2 -\frac{\alpha\eta}{4}\|\bar{m}_t\|^{2}  +\alpha\eta \|\bar{m}_t - \frac{1}{N} \sum_{n=1}^{N}\nabla F_{n}({{x}}_{n, t})\|^2 + \alpha\eta \|\frac{1}{N} \sum_{n=1}^{N}\nabla F_{n}({{x}}_{n, t})- \nabla F(\bar{{x}}_{t})\|^2\\
			& \leq F(\bar{{x}}_{t})-\frac{\alpha\eta}{2}\|\nabla F(\bar{{x}}_{t})\|^2 -\frac{\alpha\eta}{4}\|\bar{m}_t\|^{2}  +\alpha\eta \|\bar{m}_t - \frac{1}{N} \sum_{n=1}^{N}\nabla F_{n}({{x}}_{n, t})\|^2 + \alpha\eta L_F^2\frac{1}{N} \sum_{n=1}^{N}\|{x}_{n, t}-\bar{{x}}_{t}\|^2\\
			& \leq F(\bar{{x}}_{t})-\frac{\alpha\eta}{2}\|\nabla F(\bar{{x}}_{t})\|^2 -\frac{\alpha\eta}{4}\|\bar{m}_t\|^{2} + \alpha\eta L_F^2\frac{1}{N} \sum_{n=1}^{N}\|{x}_{n, t}-\bar{{x}}_{t}\|^2\\
			& \quad  +2\alpha\eta \|\bar{m}_t - \frac{1}{N} \sum_{n=1}^{N}\nabla f_{n}^{(1)} ( u_{n,t}^{(0)})\nabla  f_{n}^{(2)} ( u_{n,t}^{(1)}) \cdots \nabla  f_{n}^{(K-1)} ( u_{n,t}^{(K-2)})\nabla f_{n}^{(K)}( u_{n,t}^{(K-1)})\|^2 \\
			& \quad +2\alpha\eta \frac{1}{N} \sum_{n=1}^{N}\| \nabla f_{n}^{(1)} ( u_{n,t}^{(0)})\nabla  f_{n}^{(2)} ( u_{n,t}^{(1)}) \cdots \nabla  f_{n}^{(K-1)} ( u_{n,t}^{(K-2)})\nabla f_{n}^{(K)}( u_{n,t}^{(K-1)}) - \nabla F_{n}({{x}}_{n, t})\|^2 \\
			& \leq F(\bar{{x}}_{t})-\frac{\alpha\eta}{2}\|\nabla F(\bar{{x}}_{t})\|^2 -\frac{\alpha\eta}{4}\|\bar{m}_t\|^{2} + \alpha\eta L_F^2\frac{1}{N} \sum_{n=1}^{N}\|{x}_{n, t}-\bar{{x}}_{t}\|^2\\
			& \quad  +2\alpha\eta \|\bar{m}_t - \frac{1}{N} \sum_{n=1}^{N}\nabla f_{n}^{(1)} ( u_{n,t}^{(0)})\nabla  f_{n}^{(2)} ( u_{n,t}^{(1)}) \cdots \nabla  f_{n}^{(K-1)} ( u_{n,t}^{(K-2)})\nabla f_{n}^{(K)}( u_{n,t}^{(K-1)})\|^2 \\
			& \quad +2\alpha\eta \frac{K}{N}\sum_{n=1}^{N}\sum_{k=1}^{K-1}A_k\| u_{n,t}^{(k)} -f_{n}^{(k)}( u_{n,t}^{(k-1)}) \|^2   \ , \\
		\end{aligned}
	\end{equation}
	where the fourth step follows from $\eta\leq \frac{1}{2\alpha L_F}$,  the last step follows from Lemma~\ref{lemma_f_u_f_x_var}. 
	Similarly, we define a novel potential function below:
	\begin{equation}
		\begin{aligned}
			& 	\mathcal{H}_{t+1} = \mathbb{E}[F({\bar{x}}_{t+1}) ] + \frac{1}{N} \sum_{n=1}^{N}\sum_{k=1}^{K-1}\omega_k \mathbb{E}[\|u_{n, t+1}^{(k)} -f_{n}^{(k)}(u_{n, t+1}^{(k-1)}) \|^2 ]  \\
			& + \omega_{K} \mathbb{E}\Big[\Big\|\bar{m}_{t+1}- \frac{1}{N} \sum_{n=1}^{N}\nabla f_{n}^{(1)} ( u_{n,t+1}^{(0)})\nabla  f_{n}^{(2)} ( u_{n,t+1}^{(1)}) \cdots \nabla  f_{n}^{(K-1)} ( u_{n,t+1}^{(K-2)})\nabla f_{n}^{(K)}( u_{n,t+1}^{(K-1)})\Big\|^2\Big]   \\
			& + \omega_{K+1}\frac{1}{N} \sum_{n=1}^{N}\mathbb{E}[\|{m}_{n,t+1} - \nabla f_{n}^{(1)} ( u_{n,t+1}^{(0)})\nabla  f_{n}^{(2)} ( u_{n,t+1}^{(1)}) \cdots \nabla  f_{n}^{(K-1)} ( u_{n,t+1}^{(K-2)})\nabla f_{n}^{(K)}( u_{n,t+1}^{(K-1)})\|^2 ] \\
			& + \omega_{K+2}\frac{1}{N}\mathbb{E}[\|X_{t+1} - \bar{X}_{t+1}\|_{F}^2] + \omega_{K+3}\frac{1}{N}\mathbb{E}[\|Y_{t+1} - \bar{Y}_{t+1}\|_{F}^2 ] \ . 
		\end{aligned}
	\end{equation}
Then, we can get
\begin{equation}
	\small
	\begin{aligned}
		& \quad \mathcal{H}_{t+1} - \mathcal{H}_{t} \\
		& \leq -\frac{\alpha\eta}{2}\mathbb{E}[\|\nabla F(\bar{{x}}_{t})\|^2] -\frac{\alpha\eta}{4}\mathbb{E}[\|\bar{m}_t\|^{2}] +\omega_{K+3} \frac{4 \beta^2  \eta^4K}{1-\lambda}\sum_{k=1}^{K-1}D_{k}  \delta_{k}^2 + \omega_{K+3}\frac{\mu^2\eta^4K}{1-\lambda} \sum_{k=1}^{K}B_k \sigma_{k}^2 + 2\beta^2\eta^4 \frac{1}{N} \sum_{n=1}^{N}\sum_{k=1}^{K-1}\omega_k\delta_{k}^2 \\
		& \quad + 4  \omega_{K}\beta^2  \eta^4K  \sum_{k=1}^{K-1}D_{k} \frac{\delta_{k}^2}{N}  + 2 \omega_{K}\mu^2\eta^4 K\sum_{k=1}^{K}B_k \frac{\sigma_{k}^2 }{N} + 4 \omega_{K+1} \beta^2  \eta^4K  \sum_{k=1}^{K-1}D_{k} \delta_{k}^2 + 2\omega_{K+1} \mu^2\eta^4 K\sum_{k=1}^{K}B_k \sigma_{k}^2 \\
		& \quad + \Big(2\alpha\eta-\mu\eta^2 \omega_{K} \Big)\mathbb{E}\Big[\Big\|\bar{m}_{t} - \frac{1}{N} \sum_{n=1}^{N}\nabla f_{n}^{(1)} ( u_{n,t}^{(0)})\nabla  f_{n}^{(2)} ( u_{n,t}^{(1)}) \cdots \nabla  f_{n}^{(K-1)} ( u_{n,t}^{(K-2)})\nabla f_{n}^{(K)}( u_{n,t}^{(K-1)}) \Big\|^2 \Big] \\
		&  \quad + \frac{1}{N} \sum_{n=1}^{N}\sum_{k=1}^{K-1}\Big(4  \omega_{K}\beta^2  \eta^4 KD_{k}  \frac{1}{N}+ 4 \omega_{K+1} \beta^2  \eta^4 K D_{k} + 2\alpha\eta KA_{k}+\omega_{K+3}\frac{4 \beta^2  \eta^4 K}{1-\lambda}D_{k}  -\beta\eta^2 \omega_k\Big)\mathbb{E}[\|u_{n, t}^{(k)}- f_{n}^{(k)}(u_{n, t}^{(k-1)})  \|^2] \\
		& \quad + \Big( +\omega_{K+3} \frac{2\mu^2\eta^4}{1-\lambda}-\mu\eta^2 \omega_{K+1} \Big)\frac{1}{N} \sum_{n=1}^{N} \mathbb{E}\Big[\Big\|{m}_{n, t} - \nabla f_{n}^{(1)} ( u_{n,t}^{(0)})\nabla  f_{n}^{(2)} ( u_{n,t}^{(1)}) \cdots \nabla  f_{n}^{(K-1)} ( u_{n,t}^{(K-2)})\nabla f_{n}^{(K)}( u_{n,t}^{(K-1)}) \Big\|^2 \Big] \\
		& \quad  + \Big(\alpha\eta L_F^2-\eta\frac{1-\lambda^2}{2}\omega_{K+2}\Big)\frac{1}{N} \mathbb{E}[\|X_{t}  - \bar{X}_{t} \|_F^2]  + \Big( \omega_{K+2} \frac{2\eta\alpha^2}{1-\lambda^2} -(1- \lambda) \omega_{K+3}\Big)\frac{1}{N} \mathbb{E}[\|Y_{t}-   \bar{Y}_{t}\|_F^2] \\
		& \quad + \Big(2 \omega_{K}KD_{0}\frac{1}{N} + 2\omega_{K+1} KD_{0} + \omega_{K+3}\frac{2KD_{0}}{1-\lambda} \Big) \frac{1}{N} \mathbb{E}[\|{X}_{t+1}- {X}_{ t}\|_F^2]  \\
		& \quad +  \frac{1}{N}\sum_{n=1}^{N}\sum_{k=1}^{K-1}\Bigg(\Big(4 \omega_{K}K\frac{1}{N}+ 4\omega_{K+1} K+ \omega_{K+3}\frac{4K}{1-\lambda}\Big)D_{k} C_{k}^2 + 2\omega_k C_k^2\Bigg) \mathbb{E}[\|u_{n, t}^{(k-1)} - u_{n,t+1}^{(k-1)} \|^2]   \\
		& \leq -\frac{\alpha\eta}{2}\mathbb{E}[\|\nabla F(\bar{{x}}_{t})\|^2] -\frac{\alpha\eta}{4}\mathbb{E}[\|\bar{m}_t\|^{2}] +\omega_{K+3} \frac{4 \beta^2  \eta^4K}{1-\lambda}\sum_{k=1}^{K-1}D_{k}  \delta_{k}^2 + \omega_{K+3}\frac{\mu^2\eta^4K}{1-\lambda} \sum_{k=1}^{K}B_k \sigma_{k}^2 + 2\beta^2\eta^4 \frac{1}{N} \sum_{n=1}^{N}\sum_{k=1}^{K-1}\omega_k\delta_{k}^2 \\
		& \quad + 4  \omega_{K}\beta^2  \eta^4K  \sum_{k=1}^{K-1}D_{k} \frac{\delta_{k}^2}{N}  + 2 \omega_{K}\mu^2\eta^4 K\sum_{k=1}^{K}B_k \frac{\sigma_{k}^2 }{N} + 4 \omega_{K+1} \beta^2  \eta^4K  \sum_{k=1}^{K-1}D_{k} \delta_{k}^2 + 2\omega_{K+1} \mu^2\eta^4 K\sum_{k=1}^{K}B_k \sigma_{k}^2 \\
		& \quad + \Big(2\alpha\eta-\mu\eta^2 \omega_{K} \Big)\mathbb{E}\Big[\Big\|\bar{m}_{t} - \frac{1}{N} \sum_{n=1}^{N}\nabla f_{n}^{(1)} ( u_{n,t}^{(0)})\nabla  f_{n}^{(2)} ( u_{n,t}^{(1)}) \cdots \nabla  f_{n}^{(K-1)} ( u_{n,t}^{(K-2)})\nabla f_{n}^{(K)}( u_{n,t}^{(K-1)}) \Big\|^2 \Big] \\
		& \quad + \frac{1}{N} \sum_{n=1}^{N}\sum_{k=1}^{K-1}\Big(4  \omega_{K}\beta^2  \eta^4 KD_{k}  \frac{1}{N}+ 4 \omega_{K+1} \beta^2  \eta^4 K D_{k} + 2\alpha\eta KA_{k}+\omega_{K+3}\frac{4 \beta^2  \eta^4 K}{1-\lambda}D_{k}  -\beta\eta^2 \omega_k\Big)\mathbb{E}[\|u_{n, t}^{(k)}- f_{n}^{(k)}(u_{n, t}^{(k-1)})  \|^2] \\
		& \quad + \Big( +\omega_{K+3} \frac{2\mu^2\eta^4}{1-\lambda}-\mu\eta^2 \omega_{K+1} \Big)\frac{1}{N} \sum_{n=1}^{N} \mathbb{E}\Big[\Big\|{m}_{n, t} - \nabla f_{n}^{(1)} ( u_{n,t}^{(0)})\nabla  f_{n}^{(2)} ( u_{n,t}^{(1)}) \cdots \nabla  f_{n}^{(K-1)} ( u_{n,t}^{(K-2)})\nabla f_{n}^{(K)}( u_{n,t}^{(K-1)}) \Big\|^2 \Big] \\
		& \quad  + \Big(\alpha\eta L_F^2-\eta\frac{1-\lambda^2}{2}\omega_{K+2}\Big)\frac{1}{N} \mathbb{E}[\|X_{t}  - \bar{X}_{t} \|_F^2]  + \Big( \omega_{K+2} \frac{2\eta\alpha^2}{1-\lambda^2} -(1- \lambda) \omega_{K+3}\Big)\frac{1}{N} \mathbb{E}[\|Y_{t}-   \bar{Y}_{t}\|_F^2] \\
		& \quad + \Big(2 \omega_{K}KD_{0}\frac{1}{N} + 2\omega_{K+1} KD_{0} + \omega_{K+3}\frac{2KD_{0}}{1-\lambda} \Big) \frac{1}{N} \mathbb{E}[\|{X}_{t+1}- {X}_{ t}\|_F^2]  \\
		& \quad +  \sum_{k=1}^{K-1}\Bigg(\Big(4 \omega_{K}K\frac{1}{N}+ 4\omega_{K+1} K+ \omega_{K+3}\frac{4K}{1-\lambda}\Big)D_{k} C_{k}^2 + 2\omega_k C_k^2\Bigg)\Big(\prod_{j=1}^{k-1}(2C_{j}^2)\Big)\frac{1}{N}\mathbb{E}[\|{X}_{t+1}- {X}_{ t}\|_F^2]  \\
		& \quad +  \frac{1}{N}\sum_{n=1}^{N}\sum_{k=1}^{K-1}\Bigg(\Big(4 \omega_{K}K\frac{1}{N}+ 4\omega_{K+1} K+ \omega_{K+3}\frac{4K}{1-\lambda}\Big)D_{k} C_{k}^2 + 2\omega_k C_k^2\Bigg) 2\beta^2  \eta^4\sum_{j=1}^{k-1} \Big(\prod_{i=j+1}^{k-1}(2C_{i}^2)\Big)\mathbb{E}[\|u_{n, t}^{(j)}- f_{n}^{(j)}(u_{n, t}^{(j-1)})  \|^2] \\
		& \quad +  \frac{1}{N}\sum_{n=1}^{N}\sum_{k=1}^{K-1}\Bigg(\Big(4 \omega_{K}K\frac{1}{N}+ 4\omega_{K+1} K+ \omega_{K+3}\frac{4K}{1-\lambda}\Big)D_{k} C_{k}^2 + 2\omega_k C_k^2\Bigg)2\beta^2  \eta^4\sum_{j=1}^{k-1} \Big(\prod_{i=j+1}^{k-1}(2C_{i}^2)\Big)\delta_{j}^2 \ ,  \\
	\end{aligned}
\end{equation}
where the last step holds due to Lemma~\ref{lemma_u_inc_var}.  It ca be reformulated as below:
\begin{equation}
	\small
	\begin{aligned}
		& \quad \mathcal{H}_{t+1} - \mathcal{H}_{t} \\
		& \leq -\frac{\alpha\eta}{2}\mathbb{E}[\|\nabla F(\bar{{x}}_{t})\|^2] -\frac{\alpha\eta}{4}\mathbb{E}[\|\bar{m}_t\|^{2}]  +\omega_{K+3} \frac{4 \beta^2  \eta^4K}{1-\lambda}\sum_{k=1}^{K-1}D_{k}  \delta_{k}^2 + \omega_{K+3}\frac{\mu^2\eta^4K}{1-\lambda} \sum_{k=1}^{K}B_k \sigma_{k}^2 + 2\beta^2\eta^4 \frac{1}{N} \sum_{n=1}^{N}\sum_{k=1}^{K-1}\omega_k\delta_{k}^2 \\
		& \quad + 4  \omega_{K}\beta^2  \eta^4K  \sum_{k=1}^{K-1}D_{k} \frac{\delta_{k}^2}{N}  + 2 \omega_{K}\mu^2\eta^4 K\sum_{k=1}^{K}B_k \frac{\sigma_{k}^2 }{N} + 4 \omega_{K+1} \beta^2  \eta^4K  \sum_{k=1}^{K-1}D_{k} \delta_{k}^2 + 2\omega_{K+1} \mu^2\eta^4 K\sum_{k=1}^{K}B_k \sigma_{k}^2 \\
		& \quad + \Big(2\alpha\eta-\mu\eta^2 \omega_{K} \Big)\mathbb{E}\Big[\Big\|\bar{m}_{t} - \frac{1}{N} \sum_{n=1}^{N}\nabla f_{n}^{(1)} ( u_{n,t}^{(0)})\nabla  f_{n}^{(2)} ( u_{n,t}^{(1)}) \cdots \nabla  f_{n}^{(K-1)} ( u_{n,t}^{(K-2)})\nabla f_{n}^{(K)}( u_{n,t}^{(K-1)}) \Big\|^2 \Big] \\
		& \quad + \frac{1}{N} \sum_{n=1}^{N}\sum_{k=1}^{K-1}\Big(4  \omega_{K}\beta^2  \eta^4 KD_{k}  \frac{1}{N}+ 4 \omega_{K+1} \beta^2  \eta^4 K D_{k} + 2\alpha\eta KA_{k}+\omega_{K+3}\frac{4 \beta^2  \eta^4 K}{1-\lambda}D_{k}  -\beta\eta^2 \omega_k\Big)\mathbb{E}[\|u_{n, t}^{(k)}- f_{n}^{(k)}(u_{n, t}^{(k-1)})  \|^2] \\
		& \quad + \Big( +\omega_{K+3} \frac{2\mu^2\eta^4}{1-\lambda}-\mu\eta^2 \omega_{K+1} \Big)\frac{1}{N} \sum_{n=1}^{N} \mathbb{E}\Big[\Big\|{m}_{n, t} - \nabla f_{n}^{(1)} ( u_{n,t}^{(0)})\nabla  f_{n}^{(2)} ( u_{n,t}^{(1)}) \cdots \nabla  f_{n}^{(K-1)} ( u_{n,t}^{(K-2)})\nabla f_{n}^{(K)}( u_{n,t}^{(K-1)}) \Big\|^2 \Big] \\
		& \quad  + \Big(\alpha\eta L_F^2-\eta\frac{1-\lambda^2}{2}\omega_{K+2}\Big)\frac{1}{N} \mathbb{E}[\|X_{t}  - \bar{X}_{t} \|_F^2] + \Big( \omega_{K+2} \frac{2\eta\alpha^2}{1-\lambda^2} -(1- \lambda) \omega_{K+3}\Big)\frac{1}{N} \mathbb{E}[\|Y_{t}-   \bar{Y}_{t}\|_F^2] \\
		& \quad + \Bigg[ \sum_{k=1}^{K-1}\Bigg(\Big(4 \omega_{K}K\frac{1}{N}+ 4\omega_{K+1} K+ \omega_{K+3}\frac{4K}{1-\lambda}\Big)D_{k} C_{k}^2 + 2\omega_k C_k^2\Bigg)\Big(\prod_{j=1}^{k-1}(2C_{j}^2)\Big)\\
		& \quad \quad + 2 \omega_{K}KD_{0}\frac{1}{N} + 2\omega_{K+1} KD_{0} + \omega_{K+3}\frac{2KD_{0}}{1-\lambda}  \Bigg] \frac{1}{N} \mathbb{E}[\|{X}_{t+1}- {X}_{ t}\|_F^2]  \\
		& \quad + 2\beta^2  \eta^4 \frac{1}{N}\sum_{n=1}^{N}\sum_{k=1}^{K-1} \Bigg[\sum_{j=k+1}^{K-1}\Bigg(\Big(4 \omega_{K}K\frac{1}{N}+ 4\omega_{K+1} K+ \omega_{K+3}\frac{4K}{1-\lambda}\Big)D_{j} C_{j}^2 \\
		& \quad \quad \quad \quad \quad \quad \quad \quad \quad \quad \quad \quad \quad \quad + 2\omega_j C_j^2\Bigg)  \Big(\prod_{i=k+1}^{j}(2C_{i}^2)\Big)\Bigg]\mathbb{E}[\|u_{n, t}^{(k)}- f_{n}^{(k)}(u_{n, t}^{(k-1)})  \|^2] \\
		& \quad + 2\beta^2  \eta^4 \frac{1}{N}\sum_{n=1}^{N}\sum_{k=1}^{K-1}\Bigg[\sum_{j=k+1}^{K-1}\Bigg(\Big(4 \omega_{K}K\frac{1}{N}+ 4\omega_{K+1} K+ \omega_{K+3}\frac{4K}{1-\lambda}\Big)D_{j} C_{j}^2 + 2\omega_j C_j^2\Bigg)  \Big(\prod_{i=k+1}^{j}(2C_{i}^2)\Big)\Bigg]\delta_{k}^2 \ .  \\
	\end{aligned}
\end{equation}
Then, according to Lemma~\ref{lemma_x_inc_momentum}, we can get
\begin{equation}
	\small
	\begin{aligned}
& \quad \mathcal{H}_{t+1} - \mathcal{H}_{t} \\
		& \leq -\frac{\alpha\eta}{2}\mathbb{E}[\|\nabla F(\bar{{x}}_{t})\|^2] +\omega_{K+3} \frac{4 \beta^2  \eta^4K}{1-\lambda}\sum_{k=1}^{K-1}D_{k}  \delta_{k}^2 + \omega_{K+3}\frac{\mu^2\eta^4K}{1-\lambda} \sum_{k=1}^{K}B_k \sigma_{k}^2 + 2\beta^2\eta^4 \frac{1}{N} \sum_{n=1}^{N}\sum_{k=1}^{K-1}\omega_k\delta_{k}^2 \\
		& \quad + 4  \omega_{K}\beta^2  \eta^4K  \sum_{k=1}^{K-1}D_{k} \frac{\delta_{k}^2}{N}  + 2 \omega_{K}\mu^2\eta^4 K\sum_{k=1}^{K}B_k \frac{\sigma_{k}^2 }{N} + 4 \omega_{K+1} \beta^2  \eta^4K  \sum_{k=1}^{K-1}D_{k} \delta_{k}^2 + 2\omega_{K+1} \mu^2\eta^4 K\sum_{k=1}^{K}B_k \sigma_{k}^2 \\
		& \quad + \Big(2\alpha\eta-\mu\eta^2 \omega_{K} \Big)\mathbb{E}\Big[\Big\|\bar{m}_{t} - \frac{1}{N} \sum_{n=1}^{N}\nabla f_{n}^{(1)} ( u_{n,t}^{(0)})\nabla  f_{n}^{(2)} ( u_{n,t}^{(1)}) \cdots \nabla  f_{n}^{(K-1)} ( u_{n,t}^{(K-2)})\nabla f_{n}^{(K)}( u_{n,t}^{(K-1)}) \Big\|^2 \Big] \\
				& \quad + \Big( \omega_{K+3} \frac{2\mu^2\eta^4}{1-\lambda}-\mu\eta^2 \omega_{K+1} \Big)\frac{1}{N} \sum_{n=1}^{N} \mathbb{E}\Big[\Big\|{m}_{n, t} - \nabla f_{n}^{(1)} ( u_{n,t}^{(0)})\nabla  f_{n}^{(2)} ( u_{n,t}^{(1)}) \cdots \nabla  f_{n}^{(K-1)} ( u_{n,t}^{(K-2)})\nabla f_{n}^{(K)}( u_{n,t}^{(K-1)}) \Big\|^2 \Big] \\
		& \quad + \frac{1}{N} \sum_{n=1}^{N}\sum_{k=1}^{K-1}\Big(4  \omega_{K}\beta^2  \eta^4 KD_{k}  \frac{1}{N}+ 4 \omega_{K+1} \beta^2  \eta^4 K D_{k} + 2\alpha\eta KA_{k}+\omega_{K+3}\frac{4 \beta^2  \eta^4 K}{1-\lambda}D_{k}  -\beta\eta^2 \omega_k \\
		& \quad \quad  + 2\beta^2  \eta^4 \Bigg[\sum_{j=k+1}^{K-1}\Bigg(\Big(4 \omega_{K}K\frac{1}{N}+ 4\omega_{K+1} K+ \omega_{K+3}\frac{4K}{1-\lambda}\Big)D_{j} C_{j}^2 + 2\omega_j C_j^2\Bigg)  \Big(\prod_{i=k+1}^{j}(2C_{i}^2)\Big)\Bigg]\Big)\mathbb{E}[\|u_{n, t}^{(k)}- f_{n}^{(k)}(u_{n, t}^{(k-1)})  \|^2] \\
		& \quad  + \Big( \Bigg[ \sum_{k=1}^{K-1}\Bigg(\Big(4 \omega_{K}K\frac{1}{N}+ 4\omega_{K+1} K+ \omega_{K+3}\frac{4K}{1-\lambda}\Big)D_{k} C_{k}^2 + 2\omega_k C_k^2\Bigg)\Big(\prod_{j=1}^{k-1}(2C_{j}^2)\Big)\\
		& \quad \quad + 2 \omega_{K}KD_{0}\frac{1}{N} + 2\omega_{K+1} KD_{0} + \omega_{K+3}\frac{2KD_{0}}{1-\lambda}  \Bigg] 8\eta^2+ \alpha\eta L_F^2-\eta\frac{1-\lambda^2}{2}\omega_{K+2}\Big)\frac{1}{N} \mathbb{E}[\|X_{t}  - \bar{X}_{t} \|_F^2] \\
		& \quad + \Big( \omega_{K+2} \frac{2\eta\alpha^2}{1-\lambda^2} -(1- \lambda) \omega_{K+3}+  \Bigg[ \sum_{k=1}^{K-1}\Bigg(\Big(4 \omega_{K}K\frac{1}{N}+ 4\omega_{K+1} K+ \omega_{K+3}\frac{4K}{1-\lambda}\Big)D_{k} C_{k}^2 + 2\omega_k C_k^2\Bigg)\Big(\prod_{j=1}^{k-1}(2C_{j}^2)\Big)\\
		& \quad \quad + 2 \omega_{K}KD_{0}\frac{1}{N} + 2\omega_{K+1} KD_{0} + \omega_{K+3}\frac{2KD_{0}}{1-\lambda}  \Bigg] 4\alpha^2\eta^2 \Big)\frac{1}{N} \mathbb{E}[\|Y_{t}-   \bar{Y}_{t}\|_F^2] \\
		& \quad + \Bigg(\Bigg[ \sum_{k=1}^{K-1}\Bigg(\Big(4 \omega_{K}K\frac{1}{N}+ 4\omega_{K+1} K+ \omega_{K+3}\frac{4K}{1-\lambda}\Big)D_{k} C_{k}^2 + 2\omega_k C_k^2\Bigg)\Big(\prod_{j=1}^{k-1}(2C_{j}^2)\Big)\\
		& \quad \quad + 2 \omega_{K}KD_{0}\frac{1}{N} + 2\omega_{K+1} KD_{0} + \omega_{K+3}\frac{2KD_{0}}{1-\lambda}  \Bigg] 4\alpha^2\eta^2 - \frac{\alpha\eta}{4}\Bigg) \mathbb{E}[\|\bar{m}_{t}\|^2] \\
		& \quad + 2\beta^2  \eta^4 \frac{1}{N}\sum_{n=1}^{N}\sum_{k=1}^{K-1}\Bigg[\sum_{j=k+1}^{K-1}\Bigg(\Big(4 \omega_{K}K\frac{1}{N}+ 4\omega_{K+1} K+ \omega_{K+3}\frac{4K}{1-\lambda}\Big)D_{j} C_{j}^2 + 2\omega_j C_j^2\Bigg)  \Big(\prod_{i=k+1}^{j}(2C_{i}^2)\Big)\Bigg]\delta_{k}^2 \ .  \\
	\end{aligned}
\end{equation}
At first, we set $\omega_{K}   = \frac{2\alpha}{\mu\eta } $ such that $2\alpha\eta-\mu\eta^2 \omega_{K}  =  0$.  Then, we set $\omega_{K+3} = \alpha (1-\lambda) $. By enforcing $ \omega_{K+3} \frac{2\mu^2\eta^4}{1-\lambda}-\mu\eta^2 \omega_{K+1} =  0$, we can get $\omega_{K+1}  =  \omega_{K+3} \frac{2\mu\eta^2}{1-\lambda} =2\alpha\mu\eta^2$.

Then, we enforce
\begin{equation}
	\begin{aligned}
		& 4  \omega_{K}\beta^2  \eta^4 KD_{k}  \frac{1}{N}+ 4 \omega_{K+1} \beta^2  \eta^4 K D_{k} + 2\alpha\eta KA_{k}+\omega_{K+3}\frac{4 \beta^2  \eta^4 K}{1-\lambda}D_{k}  -\beta\eta^2 \omega_k \\
		& \quad \quad  + 2\beta^2  \eta^4 \Bigg[\sum_{j=k+1}^{K-1}\Bigg(\Big(4 \omega_{K}K\frac{1}{N}+ 4\omega_{K+1} K+ \omega_{K+3}\frac{4K}{1-\lambda}\Big)D_{j} C_{j}^2 + 2\omega_j C_j^2\Bigg)  \Big(\prod_{i=k+1}^{j}(2C_{i}^2)\Big)\Bigg] \leq 0  \ , \\
	\end{aligned}
\end{equation}
It is easy to get
\begin{equation}
	\begin{aligned}
		&  \frac{8\alpha\beta^2  \eta^3}{\mu }   \frac{KD_{k} }{N}+ 4\omega_{K+3} K D_{k}\frac{2\mu\beta^2  \eta^6}{1-\lambda}   + 2\alpha\eta KA_{k}+\omega_{K+3}KD_{k} \frac{4 \beta^2  \eta^4 }{1-\lambda} -\beta\eta^2 \omega_k \\
		& \quad \quad  + 2\beta^2  \eta^4 \Bigg[\sum_{j=1}^{K-1}\Bigg(\Big(4 \omega_{K}K\frac{1}{N}+ 4\omega_{K+3} \frac{2\mu\eta^2}{1-\lambda}  K+ \omega_{K+3}\frac{4K}{1-\lambda}\Big)D_{j} C_{j}^2 + 2\omega_j C_j^2\Bigg)  \Big(\prod_{i=k+1}^{j}(2C_{i}^2)\Big)\Bigg]  \leq  0 \ .\\
	\end{aligned}
\end{equation}
Then, we enforce 
\begin{equation}
	\begin{aligned}
		& 2\beta^2  \eta^4 \Bigg[\sum_{j=1}^{K-1}2\omega_j C_j^2\Big(\prod_{i=k+1}^{j}(2C_{i}^2)\Big)\Bigg]-\beta\eta^2 \omega_k \leq -\frac{1}{2}\beta\eta^2 \omega_k  \ , \\
	\end{aligned}
\end{equation}
and 
\begin{equation}
	\begin{aligned}
		& 	\frac{8\alpha\beta^2  \eta^3}{\mu }   \frac{KD_{k} }{N}+ 4\omega_{K+3} K D_{k}\frac{2\mu\beta^2  \eta^6}{1-\lambda}   + 2\alpha\eta KA_{k}+\omega_{K+3}KD_{k} \frac{4 \beta^2  \eta^4 }{1-\lambda} \\
		& \quad \quad  + 2\beta^2  \eta^4 \Bigg[\sum_{j=1}^{K-1}\Bigg(\Big(\frac{8\alpha}{\mu\eta } \frac{K}{N}+ 4\omega_{K+3} \frac{2\mu\eta^2}{1-\lambda}  K+ \omega_{K+3}\frac{4K}{1-\lambda}\Big)D_{j} C_{j}^2 \Bigg)  \Big(\prod_{i=k+1}^{j}(2C_{i}^2)\Big)\Bigg]  -  \frac{1}{2} \beta\eta^2 \omega_k  \leq  0  \ . \\
	\end{aligned}
\end{equation}
From the first inequality, we can get
\begin{equation}
	\begin{aligned}
				& \eta \leq \frac{1}{2}\sqrt{\frac{\omega_k}{\beta\sum_{j=1}^{K-1}2\omega_j C_j^2\Big(\prod_{i=k+1}^{j}(2C_{i}^2)\Big)}} \ .\\
	\end{aligned}
\end{equation}
As for the second inequality, due to $\mu\eta^2<1$, we have
\begin{equation}
	\begin{aligned}
		& \quad  \frac{8\alpha\beta^2  \eta^3}{\mu }   \frac{KD_{k} }{N}+ 4\omega_{K+3} K D_{k}\frac{2\mu\beta^2  \eta^6}{1-\lambda}   + 2\alpha\eta KA_{k}+\omega_{K+3}KD_{k} \frac{4 \beta^2  \eta^4 }{1-\lambda} \\
		& \quad \quad  + 2\beta^2  \eta^4 \Bigg[\sum_{j=1}^{K-1}\Bigg(\Big(\frac{8\alpha}{\mu\eta } \frac{K}{N}+ 4\omega_{K+3} \frac{2\mu\eta^2}{1-\lambda}  K+ \omega_{K+3}\frac{4K}{1-\lambda}\Big)D_{j} C_{j}^2 \Bigg)  \Big(\prod_{i=k+1}^{j}(2C_{i}^2)\Big)\Bigg]  -  \frac{1}{2} \beta\eta^2 \omega_k  \\
		& \leq \frac{8\alpha\beta^2  \eta^3}{\mu }   \frac{KD_{k} }{N}+ \omega_{K+3} K D_{k}\frac{8\beta^2  \eta^4}{1-\lambda}   + 2\alpha\eta KA_{k}+\omega_{K+3}KD_{k} \frac{4 \beta^2  \eta^4 }{1-\lambda} \\
		& \quad \quad  + 2\beta^2  \eta^4 \Bigg[\sum_{j=1}^{K-1}\Bigg(\Big(\frac{8\alpha}{\mu\eta } \frac{K}{N}+ \omega_{K+3} \frac{8}{1-\lambda}  K+ \omega_{K+3}\frac{4K}{1-\lambda}\Big)D_{j} C_{j}^2 \Bigg)  \Big(\prod_{i=k+1}^{j}(2C_{i}^2)\Big)\Bigg]  -  \frac{1}{2} \beta\eta^2 \omega_k  \\
		& \leq \frac{8\alpha\beta^2  \eta^3}{\mu }   \frac{KD_{k} }{N}+ \omega_{K+3} K D_{k}\frac{12\beta^2  \eta^4}{1-\lambda}   + 2\alpha\eta KA_{k} \\
		& \quad \quad  + 2\beta^2  \eta^4 \Bigg[\sum_{j=1}^{K-1}\Bigg(\Big(\frac{8\alpha}{\mu\eta } \frac{K}{N}+ \omega_{K+3}\frac{12K}{1-\lambda}\Big)D_{j} C_{j}^2 \Bigg)  \Big(\prod_{i=k+1}^{j}(2C_{i}^2)\Big)\Bigg]  -  \frac{1}{2} \beta\eta^2 \omega_k  \\
		& \leq \frac{8\alpha\beta^2  \eta^3}{\mu }   \frac{KD_{k} }{N}+ 12\beta^2  \eta^4\alpha  K D_{k}   + 2\alpha\eta KA_{k} \\
		& \quad \quad  + 2\alpha \beta^2  \eta^4 \Bigg[\sum_{j=1}^{K-1}\Bigg(\Big(\frac{8}{\mu\eta } \frac{K}{N}+12K  \Big)D_{j} C_{j}^2 \Bigg)  \Big(\prod_{i=k+1}^{j}(2C_{i}^2)\Big)\Bigg]  -  \frac{1}{2} \beta\eta^2 \omega_k  \\
		& \leq \beta\eta^2 \Bigg[\frac{8\alpha\beta  \eta}{\mu }   \frac{KD_{k} }{N}+ 12\beta  \eta^2\alpha  K D_{k}   + 2\frac{1}{\beta\eta}\alpha KA_{k} \\
		& \quad \quad  + 2\alpha \beta \eta^2 \sum_{j=1}^{K-1}\Bigg(\Big(\frac{8}{\mu\eta } \frac{K}{N}+12K  \Big)D_{j} C_{j}^2 \Bigg)  \Big(\prod_{i=k+1}^{j}(2C_{i}^2)\Big) -  \frac{1}{2} \omega_k \Bigg]  \ .  \\
	\end{aligned}
\end{equation}
We enforce this upper bound to be non-positive, i.e.,
\begin{equation}
	\begin{aligned}
		& \frac{8\alpha\beta  \eta}{\mu }   \frac{KD_{k} }{N}+ 12\beta  \eta^2\alpha  K D_{k}   + 2\frac{1}{\beta\eta}\alpha KA_{k}  + 2\alpha \beta \eta^2 \sum_{j=1}^{K-1}\Bigg(\Big(\frac{8}{\mu\eta } \frac{K}{N}+12K  \Big)D_{j} C_{j}^2 \Bigg)  \Big(\prod_{i=k+1}^{j}(2C_{i}^2)\Big) -  \frac{1}{2} \omega_k \leq  0  \ , \\
		& \omega_k\geq \frac{16\alpha\beta  \eta}{\mu }   \frac{KD_{k} }{N}+ 24\beta  \eta^2\alpha  K D_{k}   + 4\frac{1}{\beta\eta}\alpha KA_{k}  + 4\alpha \beta \eta^2 \sum_{j=1}^{K-1}\Bigg(\Big(\frac{8}{\mu\eta } \frac{K}{N}+12K  \Big)D_{j} C_{j}^2 \Bigg)  \Big(\prod_{i=k+1}^{j}(2C_{i}^2)\Big)  \\
		& = \frac{\alpha}{\eta}\Bigg[\frac{16\beta  \eta^2}{\mu }   \frac{KD_{k} }{N}+ 24\beta  \eta^3  K D_{k}   + 4\frac{1}{\beta} KA_{k}  + 4 \beta \eta^3 \sum_{j=1}^{K-1}\Bigg(\Big(\frac{8}{\mu\eta } \frac{K}{N}+12K  \Big)D_{j} C_{j}^2 \Bigg)  \Big(\prod_{i=k+1}^{j}(2C_{i}^2)\Big)  \Bigg] \ . \\
	\end{aligned}
\end{equation}
Due to $ \beta\eta^2\leq 1 , \eta\leq 1$, we can set 
\begin{equation}
	\begin{aligned}
		&\omega_k  \triangleq \frac{\alpha K}{\eta}\tilde{\omega}_{k} = \frac{\alpha}{\eta}\Bigg[ \frac{16 KD_{k} }{\mu N}+ 24  K D_{k}   + \frac{4KA_{k} }{\beta}  + 16K  \sum_{j=1}^{K-1}\Bigg(\Big(\frac{2}{\mu N} +3 \Big)D_{j} C_{j}^2 \Bigg)  \Big(\prod_{i=k+1}^{j}(2C_{i}^2)\Big)  \Bigg] \ , \\
	\end{aligned}
\end{equation}
where $ \tilde{\omega}_{k} = \Bigg[ \frac{16 D_{k} }{\mu N}+ 24   D_{k}   + \frac{4A_{k} }{\beta}  + 16  \sum_{j=1}^{K-1}\Bigg(\Big(\frac{2}{\mu N} +3 \Big)D_{j} C_{j}^2 \Bigg)  \Big(\prod_{i=k+1}^{j}(2C_{i}^2)\Big)  \Bigg]$.  Then, we can simplified the upper bound of $\eta$ as below:
\begin{equation}
	\begin{aligned}
		& \eta \leq  \frac{1}{2}\sqrt{\frac{ \tilde{\omega}_{k}}{2\beta\sum_{j=1}^{K-1}\tilde{\omega}_{j} C_j^2\Big(\prod_{i=k+1}^{j}(2C_{i}^2)\Big)}}  \ . \\
	\end{aligned}
\end{equation}

Based on these values, we have
\begin{equation}
	\begin{aligned}
		& \Bigg[ \sum_{k=1}^{K-1}\Bigg(\Big(4 \omega_{K}K\frac{1}{N}+ 4\omega_{K+1} K+ \omega_{K+3}\frac{4K}{1-\lambda}\Big)D_{k} C_{k}^2 + 2\omega_k C_k^2\Bigg)\Big(\prod_{j=1}^{k-1}(2C_{j}^2)\Big)\\
		& \quad \quad + 2 \omega_{K}KD_{0}\frac{1}{N} + 2\omega_{K+1} KD_{0} + \omega_{K+3}\frac{2KD_{0}}{1-\lambda}  \Bigg]  \\
		& \leq \Bigg[ \sum_{k=1}^{K-1}\Bigg(\Big(4\frac{2\alpha}{\mu\eta }  K\frac{1}{N}+ 4\omega_{K+3} \frac{2\mu\eta^2}{1-\lambda}  K+ \omega_{K+3}\frac{4K}{1-\lambda}\Big)D_{k} C_{k}^2 + 2\omega_k C_k^2\Bigg)\Big(\prod_{j=1}^{k-1}(2C_{j}^2)\Big)\\
		& \quad \quad + 2 \frac{2\alpha}{\mu\eta } KD_{0}\frac{1}{N} + 2\omega_{K+3} \frac{2\mu\eta^2}{1-\lambda}  KD_{0} + \omega_{K+3}\frac{2KD_{0}}{1-\lambda}  \Bigg] \\
		& \leq \Bigg[ \sum_{k=1}^{K-1}\Bigg(\Big(\frac{8\alpha}{\mu\eta }  \frac{K}{N}+  \omega_{K+3}\frac{12K}{1-\lambda}\Big)D_{k} C_{k}^2 + 2\omega_k C_k^2\Bigg)\Big(\prod_{j=1}^{k-1}(2C_{j}^2)\Big)+ \frac{4\alpha}{\mu\eta }\frac{ KD_{0}}{N} + \omega_{K+3}\frac{6KD_{0}}{1-\lambda}  \Bigg] \\
		& \leq \Bigg[ \sum_{k=1}^{K-1}\Bigg(\Big(\frac{8\alpha}{\mu\eta }  \frac{K}{N}+  12\alpha K \Big)D_{k} C_{k}^2 +  \frac{2\alpha K}{\eta}\tilde{\omega}_{k}  C_k^2\Bigg)\Big(\prod_{j=1}^{k-1}(2C_{j}^2)\Big)+ \frac{4\alpha}{\mu\eta }\frac{ KD_{0}}{N} + 6\alpha KD_{0} \Bigg]  \ . \\
	\end{aligned}
\end{equation}

Then, we enforce 
\begin{equation}
	\begin{aligned}
		& \quad \Bigg[ \sum_{k=1}^{K-1}\Bigg(\Big(4 \omega_{K}K\frac{1}{N}+ 4\omega_{K+1} K+ \omega_{K+3}\frac{4K}{1-\lambda}\Big)D_{k} C_{k}^2 + 2\omega_k C_k^2\Bigg)\Big(\prod_{j=1}^{k-1}(2C_{j}^2)\Big)\\
		& \quad \quad + 2 \omega_{K}KD_{0}\frac{1}{N} + 2\omega_{K+1} KD_{0} + \omega_{K+3}\frac{2KD_{0}}{1-\lambda}  \Bigg] 8\eta^2+ \alpha\eta L_F^2-\eta\frac{1-\lambda^2}{2}\omega_{K+2}  \\
		& \leq \Bigg[ \sum_{k=1}^{K-1}\Bigg(\Big(\frac{8\alpha}{\mu\eta }  \frac{K}{N}+  12\alpha K \Big)D_{k} C_{k}^2 +  \frac{2\alpha K}{\eta}\tilde{\omega}_{k}  C_k^2\Bigg)\Big(\prod_{j=1}^{k-1}(2C_{j}^2)\Big)+ \frac{4\alpha}{\mu\eta }\frac{ KD_{0}}{N} + 6\alpha KD_{0} \Bigg] 8\eta^2+ \alpha\eta L_F^2-\eta\frac{1-\lambda^2}{2}\omega_{K+2}\\
		& \leq  0  \ . \\
	\end{aligned}
\end{equation}
It is easy to know
\begin{equation}
	\begin{aligned}
		& \omega_{K+2} \geq \frac{2}{1-\lambda^2}\Bigg[\Bigg( \sum_{k=1}^{K-1}\Bigg(\Big(\frac{8\alpha}{\mu }  \frac{K}{N}+  12\alpha \eta K \Big)D_{k} C_{k}^2 +  2\alpha\tilde{\omega}_{k}K  C_k^2\Bigg)\Big(\prod_{j=1}^{k-1}(2C_{j}^2)\Big)+ \frac{4\alpha}{\mu }\frac{ KD_{0}}{N} + 6\alpha\eta  KD_{0} \Bigg)8+ \alpha L_F^2\Bigg]  \ .  \\
	\end{aligned}
\end{equation}
Since $\eta<1$, we can set 
\begin{equation}
	\begin{aligned}
		& \omega_{K+2} \triangleq \frac{\alpha }{1-\lambda^2}\tilde{\omega}_{K+2} = \frac{2\alpha }{1-\lambda^2}\Bigg[\Bigg( \sum_{k=1}^{K-1}\Bigg(\Big(\frac{8}{\mu N }  +  12   \Big)D_{k} C_{k}^2 +  2\tilde{\omega}_{k}  C_k^2\Bigg)\Big(\prod_{j=1}^{k-1}(2C_{j}^2)\Big)+ \frac{ 4D_{0}}{\mu N} + 6   D_{0} \Bigg)8K+  L_F^2\Bigg]  \ ,  \\
	\end{aligned}
\end{equation}
where $\tilde{\omega}_{K+2}=2\Bigg[\Bigg( \sum_{k=1}^{K-1}\Bigg(\Big(\frac{8}{\mu N }  +  12   \Big)D_{k} C_{k}^2 +  2\tilde{\omega}_{k}  C_k^2\Bigg)\Big(\prod_{j=1}^{k-1}(2C_{j}^2)\Big)+ \frac{ 4D_{0}}{\mu N} + 6   D_{0} \Bigg)8K+  L_F^2\Bigg] $.

Moreover, we enforce 
\begin{equation}
	\begin{aligned}
		& \omega_{K+2} \frac{2\eta\alpha^2}{1-\lambda^2} -(1- \lambda) \omega_{K+3}+  \Bigg[ \sum_{k=1}^{K-1}\Bigg(\Big(4 \omega_{K}K\frac{1}{N}+ 4\omega_{K+1} K+ \omega_{K+3}\frac{4K}{1-\lambda}\Big)D_{k} C_{k}^2 + 2\omega_k C_k^2\Bigg)\Big(\prod_{j=1}^{k-1}(2C_{j}^2)\Big)\\
		& \quad \quad + 2 \omega_{K}KD_{0}\frac{1}{N} + 2\omega_{K+1} KD_{0} + \omega_{K+3}\frac{2KD_{0}}{1-\lambda}  \Bigg] 4\alpha^2\eta^2 \\
		& \leq \tilde{\omega}_{K+2}\frac{2\eta\alpha^3}{(1-\lambda^2)^2} -\alpha (1-\lambda)^2\\
		& \quad +  \Bigg[ \sum_{k=1}^{K-1}\Bigg(\Big(\frac{8\alpha}{\mu\eta }  \frac{K}{N}+  12\alpha K \Big)D_{k} C_{k}^2 +  \frac{2\alpha K}{\eta}\tilde{\omega}_{k}  C_k^2\Bigg)\Big(\prod_{j=1}^{k-1}(2C_{j}^2)\Big)+ \frac{4\alpha}{\mu\eta }\frac{ KD_{0}}{N} + 6\alpha KD_{0} \Bigg] 4\alpha^2\eta^2  \leq  0  \ . \\
	\end{aligned}
\end{equation}
Then, due to $\eta<1$ and $1+\lambda>1$, we can get
\begin{equation}
	\begin{aligned}
		& \alpha\leq \frac{(1-\lambda)^2}{\sqrt{2\tilde{\omega}_{K+2} +4 K \Bigg[ \sum_{k=1}^{K-1}\Bigg(\Big(\frac{8}{\mu N}  +  12   \Big)D_{k} C_{k}^2 +  2\tilde{\omega}_{k}  C_k^2\Bigg)\Big(\prod_{j=1}^{k-1}(2C_{j}^2)\Big)+ \frac{ 4D_{0}}{\mu N} + 6  D_{0} \Bigg]  }} .
	\end{aligned}
\end{equation}

Moreover, with $\eta<1$,  we enforce 
\begin{equation}
	\begin{aligned}
		& \quad \Bigg[ \sum_{k=1}^{K-1}\Bigg(\Big(4 \omega_{K}K\frac{1}{N}+ 4\omega_{K+1} K+ \omega_{K+3}\frac{4K}{1-\lambda}\Big)D_{k} C_{k}^2 + 2\omega_k C_k^2\Bigg)\Big(\prod_{j=1}^{k-1}(2C_{j}^2)\Big)\\
		& \quad \quad + 2 \omega_{K}KD_{0}\frac{1}{N} + 2\omega_{K+1} KD_{0} + \omega_{K+3}\frac{2KD_{0}}{1-\lambda}  \Bigg] 4\alpha^2\eta^2 - \frac{\alpha\eta}{4} \\
		& \leq \Bigg[ \sum_{k=1}^{K-1}\Bigg(\Big(\frac{8\alpha}{\mu\eta }  \frac{K}{N}+  12\alpha K \Big)D_{k} C_{k}^2 +  \frac{2\alpha K}{\eta}\tilde{\omega}_{k}  C_k^2\Bigg)\Big(\prod_{j=1}^{k-1}(2C_{j}^2)\Big)+ \frac{4\alpha}{\mu\eta }\frac{ KD_{0}}{N} + 6\alpha KD_{0} \Bigg]  4\alpha^2\eta^2 - \frac{\alpha\eta}{4} \\
		& \leq \Bigg[ \sum_{k=1}^{K-1}\Bigg(\Big(\frac{8\alpha}{\mu }  \frac{K}{N}+  12\eta \alpha K \Big)D_{k} C_{k}^2 +  2\alpha K\tilde{\omega}_{k}  C_k^2\Bigg)\Big(\prod_{j=1}^{k-1}(2C_{j}^2)\Big)+ \frac{4\alpha}{\mu }\frac{ KD_{0}}{N} + 6\alpha \eta KD_{0} \Bigg]  4\alpha^2\eta - \frac{\alpha\eta}{4} \\
		& \leq \Bigg[ \sum_{k=1}^{K-1}\Bigg(\Big(\frac{8}{\mu }  \frac{K}{N}+  12  K \Big)D_{k} C_{k}^2 +  2K\tilde{\omega}_{k}  C_k^2\Bigg)\Big(\prod_{j=1}^{k-1}(2C_{j}^2)\Big)+ \frac{4}{\mu }\frac{ KD_{0}}{N} + 6  KD_{0} \Bigg]  4\alpha^3\eta - \frac{\alpha\eta}{4}  \leq  0  \ ,  \\
	\end{aligned}
\end{equation}
so that we can get
\begin{equation}
	\begin{aligned}
		& \alpha \leq \frac{1}{4}\Bigg/\sqrt{K\sum_{k=1}^{K-1}\Bigg(\Big(\frac{8}{\mu N }  +  12   \Big)D_{k} C_{k}^2 +  2\tilde{\omega}_{k}  C_k^2\Bigg)\Big(\prod_{j=1}^{k-1}(2C_{j}^2)\Big)+ \frac{ 4KD_{0}}{\mu N} + 6  KD_{0} }   \ .  \\ 
	\end{aligned}
\end{equation}
In summary, by setting 
\begin{equation}\label{eq_hyperparams_var}
	\begin{aligned}
		& \omega_{k} = \frac{\alpha K}{\eta}\tilde{\omega}_{k}, \ k\in \{1, 2, \cdots, K-1\}  \ , \\
		& \omega_{K}   = \frac{2\alpha}{\mu\eta } \ ,   \omega_{K+1}  =  2\alpha\mu\eta^2   \ ,  \omega_{K+2} = \frac{\alpha }{1-\lambda^2} \tilde{\omega}_{K+2}  \ , \omega_{K+3} = \alpha (1-\lambda)  \ , \\
		& \alpha\leq(1-\lambda)^2\Bigg/\sqrt{2\tilde{\omega}_{K+2} +4 K \Bigg[ \sum_{k=1}^{K-1}\Bigg(\Big(\frac{8}{\mu N}  +  12   \Big)D_{k} C_{k}^2 +  2\tilde{\omega}_{k}  C_k^2\Bigg)\Big(\prod_{j=1}^{k-1}(2C_{j}^2)\Big)+ \frac{ 4D_{0}}{\mu N} + 6  D_{0} \Bigg]  }  \  , \\
		& \alpha \leq \frac{1}{4}\Bigg/\sqrt{K\sum_{k=1}^{K-1}\Bigg(\Big(\frac{8}{\mu N }  +  12   \Big)D_{k} C_{k}^2 +  2\tilde{\omega}_{k}  C_k^2\Bigg)\Big(\prod_{j=1}^{k-1}(2C_{j}^2)\Big)+ \frac{ 4KD_{0}}{\mu N} + 6  KD_{0} }   \ ,  \\ 
		& \eta \leq  \frac{1}{2}\sqrt{\frac{ \tilde{\omega}_{k}}{2\beta\sum_{j=1}^{K-1}\tilde{\omega}_{j} C_j^2\Big(\prod_{i=k+1}^{j}(2C_{i}^2)\Big)}}  \ ,  \\
	\end{aligned}
\end{equation}
where $\tilde{\omega}_{k}= \Bigg[\frac{16 D_{k} }{\mu N}+ 24   D_{k}   + \frac{4A_{k} }{\beta}  + 16 \sum_{j=1}^{K-1}\Bigg(\Big(\frac{2}{\mu N} +3 \Big)D_{j} C_{j}^2 \Bigg)  \Big(\prod_{i=k+1}^{j}(2C_{i}^2)\Big)\Bigg]  $, $\tilde{\omega}_{K+2} = 2\Bigg[\Bigg( \sum_{k=1}^{K-1}\Bigg(\Big(\frac{8K}{\mu N }  +  12  K \Big)D_{k} C_{k}^2 +  2\tilde{\omega}_{k}  C_k^2\Bigg)\Big(\prod_{j=1}^{k-1}(2C_{j}^2)\Big)+ \frac{ 4KD_{0}}{\mu N} + 6   KD_{0} \Bigg)8+  L_F^2\Bigg] $, 
we can get
\begin{equation}
	\begin{aligned}
& \quad \mathcal{H}_{t+1} - \mathcal{H}_{t} \\
		& \leq -\frac{\alpha\eta}{2}\mathbb{E}[\|\nabla F(\bar{{x}}_{t})\|^2]  +  4 \alpha\beta^2  \eta^4K\sum_{k=1}^{K-1}D_{k}  \delta_{k}^2 + \alpha \mu^2\eta^4K \sum_{k=1}^{K}B_k \sigma_{k}^2 + 2\alpha\beta^2\eta^3 K\sum_{k=1}^{K-1}\tilde{\omega}_{k} \delta_{k}^2 \\
		& \quad + \frac{8\alpha}{\mu }\beta^2  \eta^3K  \sum_{k=1}^{K-1}D_{k} \frac{\delta_{k}^2}{N}  +  \frac{4\alpha \mu^2\eta^3}{\mu } K\sum_{k=1}^{K}B_k \frac{\sigma_{k}^2 }{N} + 8 \alpha\mu \beta^2  \eta^6 K  \sum_{k=1}^{K-1}D_{k} \delta_{k}^2 + 4\alpha\mu^3\eta^6 K\sum_{k=1}^{K}B_k \sigma_{k}^2 \\
		& \quad + 2\alpha \beta^2  \eta^3 K\sum_{k=1}^{K-1}\Bigg[\sum_{j=k+1}^{K-1}\Bigg(\Big( \frac{8}{\mu N}+ 8\mu\eta^3 + 4\eta \Big)D_{j} C_{j}^2 +2\tilde{\omega}_{j} C_j^2\Bigg)  \Big(\prod_{i=k+1}^{j}(2C_{i}^2)\Big)\Bigg]\delta_{k}^2 \ .  \\
	\end{aligned}
\end{equation}

By summing over $t$ from $0$ to $T-1$, we can get
\begin{equation}
	\begin{aligned}
		& \quad \frac{1}{T}\sum_{t=0}^{T-1}\mathbb{E}[\|\nabla F(\bar{{x}}_{t})\|^2 ] \\
		& \leq \frac{2(\mathcal{H}_{0} - \mathcal{H}_{T})}{\alpha\eta T} + 8 \beta^2  \eta^3K\sum_{k=1}^{K-1}D_{k}  \delta_{k}^2 + 2\mu^2\eta^3K \sum_{k=1}^{K}B_k \sigma_{k}^2 + 4\beta^2\eta^2 K\sum_{k=1}^{K-1}\tilde{\omega}_{k} \delta_{k}^2\\
		& \quad + \frac{16\beta^2  \eta^2}{\mu }K  \sum_{k=1}^{K-1}D_{k} \frac{\delta_{k}^2}{N}  + 8 \mu\eta^2 K\sum_{k=1}^{K}B_k \frac{\sigma_{k}^2 }{N} +16 \mu \beta^2  \eta^5 K  \sum_{k=1}^{K-1}D_{k} \delta_{k}^2 + 8\mu^3\eta^5 K\sum_{k=1}^{K}B_k \sigma_{k}^2 \\
		& \quad + 4\beta^2  \eta^2 K\sum_{k=1}^{K-1}\Bigg[\sum_{j=k+1}^{K-1}\Bigg(\Big( \frac{8}{\mu N}+ 8\mu\eta^3 + 4\eta \Big)D_{j} C_{j}^2 +2\tilde{\omega}_{j} C_j^2\Bigg)  \Big(\prod_{i=k+1}^{j}(2C_{i}^2)\Big)\Bigg]\delta_{k}^2 \ .  \\
	\end{aligned}
\end{equation}
In the following, we bound $\mathcal{H}_{0}$. Specifically,  we have 
\begin{equation}
	\begin{aligned}
		&  \mathbb{E}[\|u_{n, 0}^{(k)} -f_{n}^{(k)}(u_{n, 0}^{(k-1)}) \|^2] = \mathbb{E}[\|f_{n}^{(k)}({u}_{n,  0}^{(k-1)}; \xi_{ n, 0}^{(k)}) -f_{n}^{(k)}(u_{n, 0}^{(k-1)}) \|^2] \leq \frac{\delta_{k}^2 }{S}  \ , \\
	\end{aligned}
\end{equation}
\begin{equation}
	\begin{aligned}
		& \quad \mathbb{E}\Big[\Big\|\frac{1}{N}\sum_{n=1}^{N}{m}_{n,0} -\frac{1}{N}\sum_{n=1}^{N}\nabla f_{n}^{(1)} (u_{n,0}^{(0)})\nabla  f_{n}^{(2)} (u_{n,0}^{(1)}) \cdots \nabla  f_{n}^{(K-1)} (u_{n,0}^{(K-2)})\nabla  f_{n}^{(K)} (u_{n,0}^{(K-1)}) \Big\|^2\Big]  \\
		& =  \mathbb{E}\Big[\Big\|\frac{1}{N}\sum_{n=1}^{N}{v}_{n,  0}^{(1)} {v}_{ n, 0}^{(2)}\cdots {v}_{ n, 0}^{(K-1)} {v}_{n,  0}^{(K)} -\frac{1}{N}\sum_{n=1}^{N}\nabla f_{n}^{(1)} (u_{n,0}^{(0)})\nabla  f_{n}^{(2)} (u_{n,0}^{(1)}) \cdots \nabla  f_{n}^{(K-1)} (u_{n,0}^{(K-2)})\nabla  f_{n}^{(K)} (u_{n,0}^{(K-1)}) \Big\|^2\Big]\\
		& \leq  K\sum_{k=1}^{K}B_k \frac{\sigma_{k}^2 }{SN} \ ,   \\
	\end{aligned}
\end{equation}
as well as $\mathbb{E}\Big[\Big\|{m}_{n,0} -\nabla f_{n}^{(1)} (u_{n,0}^{(0)})\nabla  f_{n}^{(2)} (u_{n,0}^{(1)}) \cdots \nabla  f_{n}^{(K-1)} (u_{n,0}^{(K-2)})\nabla  f_{n}^{(K)} (u_{n,0}^{(K-1)}) \Big\|^2\Big] \leq K\sum_{k=1}^{K}B_k \frac{\sigma_{k}^2 }{S}$. 
Similar to Theorem~\ref{theorem1}, we can get
\begin{equation}
	\begin{aligned}
		& \quad \frac{1}{N} \mathbb{E}[\|Y_{0} - \bar{Y}_{0}\|_{F}^2 ]  \leq  6K\sum_{k=1}^{K}B_k \sigma_{k}^2  + 12K\sum_{k=2}^{K} \frac{(\prod_{j=1}^{K}C_j^2)L_k^2}{C_k^2} \sum_{i=1}^{k-1}8\delta_{i}^2\prod_{j=i+1}^{k-1} (8C_{j}^2) \ . 
	\end{aligned}
\end{equation}
As a result, we can get 
\begin{equation}
	\begin{aligned}
		& 	H_{0} = F({{x}}_{0}) + \frac{1}{N} \sum_{n=1}^{N}\sum_{k=1}^{K-1}\omega_k\mathbb{E}[\|u_{n, 0}^{(k)} -f_{n}^{(k)}(u_{n, 0}^{(k-1)}) \|^2]   \\
		& + \omega_{K} \mathbb{E}\Big[\Big\|\bar{m}_{0}- \frac{1}{N} \sum_{n=1}^{N}\nabla f_{n}^{(1)} ( u_{n,0}^{(0)})\nabla  f_{n}^{(2)} ( u_{n,0}^{(1)}) \cdots \nabla  f_{n}^{(K-1)} ( u_{n,0}^{(K-2)})\nabla f_{n}^{(K)}( u_{n,0}^{(K-1)})\Big\|^2\Big]   \\
		& + \omega_{K+1}\frac{1}{N} \sum_{n=1}^{N}\mathbb{E}[\|{m}_{n,0} - \nabla f_{n}^{(1)} ( u_{n,0}^{(0)})\nabla  f_{n}^{(2)} ( u_{n,0}^{(1)}) \cdots \nabla  f_{n}^{(K-1)} ( u_{n,0}^{(K-2)})\nabla f_{n}^{(K)}( u_{n,0}^{(K-1)})\|^2]  \\
		& + \omega_{K+2}\frac{1}{N}\mathbb{E}[\|X_{0} - \bar{X}_{0}\|_{F}^2 ]+ \omega_{K+3}\frac{1}{N}\mathbb{E}[\|Y_{0} - \bar{Y}_{0}\|_{F}^2]  \\
		& \leq F({{x}}_{0})  + \frac{\alpha K}{\eta}\frac{1}{N} \sum_{n=1}^{N}\sum_{k=1}^{K-1}\tilde{\omega}_{k}\frac{\delta_{k}^2}{S}  +  \frac{2\alpha}{\mu\eta }K\sum_{k=1}^{K}B_k \frac{\sigma_{k}^2 }{SN}  +  2\alpha\mu\eta^2  K\sum_{k=1}^{K}B_k\frac{ \sigma_{k}^2 }{S} \\
		&\quad  + \alpha\Big( 6K\sum_{k=1}^{K}B_k \sigma_{k}^2  + 12K\sum_{k=2}^{K} \frac{(\prod_{j=1}^{K}C_j^2)L_k^2}{C_k^2} \sum_{i=1}^{k-1}8\delta_{i}^2\prod_{j=i+1}^{k-1} (8C_{j}^2)\Big)  \ . \\
	\end{aligned}
\end{equation}

Finally, we can get
\begin{equation}
	\begin{aligned}
		& \quad \frac{1}{T}\sum_{t=0}^{T-1}\mathbb{E}[\|\nabla F(\bar{{x}}_{t})\|^2 ] \\
		& \leq \frac{2({F}(x_0) -F(x_*))}{\alpha\eta T} + \frac{2K}{\eta^2 T}\frac{1}{N} \sum_{n=1}^{N}\sum_{k=1}^{K-1}\tilde{\omega}_{k}\frac{\delta_{k}^2}{S}  +  \frac{4K}{\mu\eta^2 T} \sum_{k=1}^{K}B_k \frac{\sigma_{k}^2 }{SN}  +  \frac{4\mu\eta  K}{ T} \sum_{k=1}^{K}B_k\frac{ \sigma_{k}^2 }{S} \\
		&\quad  +  \frac{12K}{\eta T}\Big( \sum_{k=1}^{K}B_k \sigma_{k}^2  + 2\sum_{k=2}^{K} \frac{(\prod_{j=1}^{K}C_j^2)L_k^2}{C_k^2} \sum_{i=1}^{k-1}8\delta_{i}^2\prod_{j=i+1}^{k-1} (8C_{j}^2)\Big) \\
		& \quad +  8 \beta^2  \eta^3K\sum_{k=1}^{K-1}D_{k}  \delta_{k}^2 + 2\mu^2\eta^3K \sum_{k=1}^{K}B_k \sigma_{k}^2 + 4\beta^2\eta^2 K\sum_{k=1}^{K-1}\tilde{\omega}_{k} \delta_{k}^2\\
		& \quad + \frac{16\beta^2  \eta^2}{\mu }K  \sum_{k=1}^{K-1}D_{k} \frac{\delta_{k}^2}{N}  + 8 \mu\eta^2 K\sum_{k=1}^{K}B_k \frac{\sigma_{k}^2 }{N} +16 \mu \beta^2  \eta^5 K  \sum_{k=1}^{K-1}D_{k} \delta_{k}^2 + 8\mu^3\eta^5 K\sum_{k=1}^{K}B_k \sigma_{k}^2 \\
		& \quad + 4\beta^2  \eta^2 K\sum_{k=1}^{K-1}\Bigg[\sum_{j=k+1}^{K-1}\Bigg(\Big( \frac{8}{\mu N}+ 8\mu\eta^3 + 4\eta \Big)D_{j} C_{j}^2 +2\tilde{\omega}_{j} C_j^2\Bigg)  \Big(\prod_{i=k+1}^{j}(2C_{i}^2)\Big)\Bigg]\delta_{k}^2 \ .  \\
	\end{aligned}
\end{equation}

\end{proof}

\end{document}